\documentclass[accepted]{uai2022}

\usepackage[american]{babel}

\usepackage{natbib} 
\usepackage[tbtags]{mathtools} 
\usepackage{booktabs} 
\usepackage{tikz}
\usepackage{amsthm}
\usepackage{amssymb}
\usepackage{multirow}

\usepackage[utf8]{inputenc}   
\usepackage[T1]{fontenc}      

\usepackage{comment}

\usepackage{natbib}

\usepackage[tbtags]{amsmath}
\usepackage{amsthm}
\allowdisplaybreaks
\usepackage{amssymb,mathrsfs}
\usepackage{amsfonts}
\usepackage{upgreek}
\usepackage{xspace}

\usepackage{graphicx}
\usepackage{subfig}
\usepackage{color}
\usepackage{algorithm, algorithmic}

\usepackage{stmaryrd}
\usepackage[inline]{enumitem}
\usepackage{url}

\usepackage{tikz}
\usetikzlibrary{calc,arrows.meta}

\usepackage{pgfplots}
\usepackage{xcolor}
\usepackage{bbm}
\usepackage{ifthen}
\usepackage{xargs}

\usepackage{aliascnt}
\usepackage{cleveref}
\usepackage{autonum}
\makeatletter

\crefname{theorem}{theorem}{Theorems}
\Crefname{Theorem}{Theorem}{Theorems}

\newtheorem*{lemma_nonumber*}{Lemma}

\newaliascnt{lemma}{theorem}

\aliascntresetthe{lemma}
\crefname{lemma}{lemma}{lemmas}
\Crefname{Lemma}{Lemma}{Lemmas}

\newaliascnt{corollary}{theorem}

\aliascntresetthe{corollary}
\crefname{corollary}{corollary}{corollaries}
\Crefname{Corollary}{Corollary}{Corollaries}

\newaliascnt{proposition}{theorem}
\newtheorem{proposition}[proposition]{Proposition}
\aliascntresetthe{proposition}
\crefname{proposition}{proposition}{propositions}
\Crefname{Proposition}{Proposition}{Propositions}

\newaliascnt{definition}{theorem}

\aliascntresetthe{definition}
\crefname{definition}{definition}{definitions}
\Crefname{Definition}{Definition}{Definitions}

\newaliascnt{remark}{theorem}

\aliascntresetthe{remark}
\crefname{remark}{remark}{remarks}
\Crefname{Remark}{Remark}{Remarks}

\crefname{example}{example}{examples}
\Crefname{Example}{Example}{Examples}

\crefname{technique}{technique}{techniques}
\Crefname{Technique}{Technique}{Techniques}

\crefname{figure}{figure}{figures}
\Crefname{Figure}{Figure}{Figures}

\crefformat{assumption}{{\textbf{A}}#2#1#3}

\newtheorem{assumptionF}{\textbf{F}\hspace{-3pt}}
\crefformat{assumptionF}{{\textbf{F}}#2#1#3}

\Crefname{assumptionB}{\textbf{B}\hspace{-3pt}}{\textbf{B}\hspace{-3pt}}
\crefname{assumptionB}{\textbf{B}}{\textbf{B}}

\Crefname{assumptionC}{\textbf{C}\hspace{-3pt}}{\textbf{C}\hspace{-3pt}}
\crefname{assumptionC}{\textbf{C}}{\textbf{C}}

\Crefname{assumptionH}{\textbf{H}\hspace{-3pt}}{\textbf{H}\hspace{-3pt}}
\crefname{assumptionH}{\textbf{H}}{\textbf{H}}

\Crefname{assumptionT}{\textbf{T}\hspace{-3pt}}{\textbf{T}\hspace{-3pt}}
\crefname{assumptionT}{\textbf{T}}{\textbf{T}}

\Crefname{assumptionT}{\textbf{T}\hspace{-3pt}}{\textbf{T}\hspace{-3pt}}
\crefname{assumptionT}{\textbf{T}}{\textbf{T}}

\Crefname{assumptionL}{\textbf{L}\hspace{-3pt}}{\textbf{L}\hspace{-3pt}}
\crefname{assumptionL}{\textbf{L}}{\textbf{L}}

\Crefname{assumptionQ}{\textbf{Q}\hspace{-3pt}}{\textbf{Q}\hspace{-3pt}}
\crefname{assumptionQ}{\textbf{Q}}{\textbf{Q}}

\Crefname{assumptionAR}{\textbf{AR}\hspace{-3pt}}{\textbf{AR}\hspace{-3pt}}
\crefname{assumptionAR}{\textbf{AR}}{\textbf{AR}}

\usepackage{bm}
\usepackage{wrapfig}

\def\hlf{\hat{\ell}^f}
\def\hlb{\hat{\ell}^b}

\def\contspace{\mathcal{C}}

\def\bpobs{\bar{p}_{\textup{obs}}}
\def\pobs{p_{\textup{obs}}}
\def\pjoin{p_{\textup{join}}}
\def\pjref{p_{\textup{jref}}}

\def\pdata{p_{\textup{data}}}

\def\pref{p_{\textup{ref}}}

\def\yobs{y^{\textup{obs}}}

\def\Pens{\mathscr{P}}

\newcommand{\schro}{Schr\"{o}dinger\xspace}

\newcommand{\tta}{\mathtt{A}}

\newcommand{\Capprox}{\tta}

\newcommandx\ctun[1][1=T]{\Capprox_{#1,1}}

\newcommandx{\expec}[2]{{\mathbb E}\left[#1 \middle \vert #2  \right]} 

\newcommandx{\norm}[2][1=]{\ifthenelse{\equal{#1}{}}{\left\Vert #2 \right\Vert}{\left\Vert #2 \right\Vert^{#1}}}
\newcommandx{\normLigne}[2][1=]{\ifthenelse{\equal{#1}{}}{\Vert #2 \Vert}{\Vert #2\Vert^{#1}}}

\def\bfc{\mathbf{c}}
\def\bfY{\mathbf{Y}}

\def\bfX{\mathbf{X}}
\def\tbfX{\tilde{\mathbf{X}}}
\def\hbfX{\hat{\mathbf{X}}}
\def\tbfY{\tilde{\mathbf{Y}}}
\def\hbfY{\hat{\mathbf{Y}}}

\def\bfZ{\mathbf{Z}}

\def\bfZ{\mathbf{Z}}

\def\bfB{\mathbf{B}}

\def\msa{\mathsf{A}}

\newcommand{\mcb}[1]{\mathcal{B}(#1)}

\def\mcy{\mathcal{Y}}
\def\mcx{\mathcal{X}}

\def\Mbb{\mathbb{M}}
\def\Pbb{\mathbb{P}}

\def\rset{\mathbb{R}}

\def\nset{\mathbb{N}}

\def\rmd{\mathrm{d}}

\def\rmc{\mathrm{C}}

\newcommandx{\functionspace}[2][1=+]{\mathbb{F}_{#1}(#2)}

\newcommand{\argmin}{\operatorname*{arg\,min}}

\newcommandx{\VarDeux}[3][3=]{\operatorname{Var}^{#3}_{#1}\left\{#2 \right\}}

\newcommand{\LeftEqNo}{\let\veqno\@@leqno}

\newcommand{\vois}{\mathcal{N}}

\newcommand{\N}{\ensuremath{\mathbb{N}}}

\newcommand{\PE}{\mathbb{E}}

\newcommand{\absLigne}[1]{\vert #1 \vert}

\newcommandx{\Vnorm}[2][1=V]{\| #2 \|_{#1}}
\newcommandx{\VnormEq}[2][1=V]{\left\| #2 \right\|_{#1}}

\newcommandx\probaMarkovTilde[2][2=]
{\ifthenelse{\equal{#2}{}}{{\widetilde{\mathbb{P}}_{#1}}}{\widetilde{\mathbb{P}}_{#1}\left[ #2\right]}}

\newcommand{\expeLigne}[1]{\PE [ #1 ]}

\newcommand{\expeMarkovLigne}[2]{\PE_{#1} [ #2 ]}

\def\ie{\textit{i.e.}}

\def\eqsp{\;}

\newcommand{\ccint}[1]{\left[#1\right]}

\newcommandx{\weight}[2][2=n]{\omega_{#1,#2}^N}

\newcommandx\sequence[3][2=,3=]
{\ifthenelse{\equal{#3}{}}{\ensuremath{\{ #1_{#2}\}}}{\ensuremath{\{ #1_{#2}, \eqsp #2 \in #3 \}}}}

\newcommandx\sequenceD[3][2=,3=]
{\ifthenelse{\equal{#3}{}}{\ensuremath{\{ #1_{#2}\}}}{\ensuremath{( #1)_{ #2 \in #3} }}}

\newcommandx{\sequencen}[2][2=n\in\N]{\ensuremath{\{ #1_n, \eqsp #2 \}}}
\newcommandx\sequenceDouble[4][3=,4=]
{\ifthenelse{\equal{#3}{}}{\ensuremath{\{ (#1_{#3},#2_{#3}) \}}}{\ensuremath{\{  (#1_{#3},#2_{#3}), \eqsp #3 \in #4 \}}}}
\newcommandx{\sequencenDouble}[3][3=n\in\N]{\ensuremath{\{ (#1_{n},#2_{n}), \eqsp #3 \}}}

\def\eg{\textit{e.g.}}

\newcommand{\opnorm}[1]{{\left\vert\kern-0.25ex\left\vert\kern-0.25ex\left\vert #1
    \right\vert\kern-0.25ex\right\vert\kern-0.25ex\right\vert}}

\def\generator{\mathcal{A}}

\def\Id{\operatorname{Id}}

\newcommandx{\CPE}[3][1=]{{\mathbb E}_{#1}\left[#2 \middle \vert #3  \right]} 
\newcommandx{\CPELigne}[3][1=]{{\mathbb E}_{#1}[#2  \vert #3  ]} 
\newcommandx{\CPEsq}[3][1=]{{\mathbb{E}^{1/2}}_{#1}\left[#2 \middle \vert #3  \right]} 
\newcommandx{\CPVar}[3][1=]{\mathrm{Var}^{#3}_{#1}\left\{ #2 \right\}}
\newcommand{\CPP}[3][]
{\ifthenelse{\equal{#1}{}}{{\mathbb P}\left(\left. #2 \, \right| #3 \right)}{{\mathbb P}_{#1}\left(\left. #2 \, \right | #3 \right)}}

\newcommandx{\osc}[2][1=]{\mathrm{osc}_{#1}(#2)}

\def\Id{\operatorname{Id}}

\newcommand{\ensemble}[2]{\left\{#1\,:\eqsp #2\right\}}
\newcommand{\ensembleLigne}[2]{\{#1\,:\eqsp #2\}}

\newcommand\coupling[2]{\Gamma(\mu,\nu)}

\newcommandx{\KL}[2]{\operatorname{KL}\left( #1 | #2 \right)}
\newcommandx{\KLsqrt}[2]{\operatorname{KL}^{1/2}\left( #1 | #2 \right)}
\newcommandx{\Jef}[2]{\operatorname{J}\left( #1 , #2 \right)}
\newcommandx{\JefLigne}[2]{\operatorname{J}( #1 , #2 )}
\newcommandx{\KLLigne}[2]{\operatorname{KL}( #1 | #2 )}

\def\gaStep
\def\QKer{Q}

\def\distance{\mathbf{d}}
\newcommandx{\wasserstein}[3][1=\distance,3=]{\mathbf{W}_{#1}^{#3}\left(#2\right)}
\newcommandx{\wassersteinLigne}[3][1=\distance,3=]{\mathbf{W}_{#1}^{#3}(#2)}
\newcommandx{\wassersteinD}[1][1=\distance]{\mathbf{W}_{#1}}
\newcommandx{\wassersteinDLigne}[1][1=\distance]{\mathbf{W}_{#1}}

\def\sigmaD{\sigma^2}

\newcommandx{\phibfs}[1][1=]{\pmb{\varphi}_{\sigmaD_{#1}}}

\newcommandx\sequenceg[3][2=,3=]
{\ifthenelse{\equal{#3}{}}{\ensuremath{( #1_{#2})}}{\ensuremath{( #1_{#2})_{ #2 \geq #3}}}}

\def\rmL{\mathrm{L}}

\newcommandx{\distV}[1][1=\bfc]{\mathbf{W}_{#1}}
\newcommandx{\distVdeux}[1][1=W_2]{\mathbf{d}_{#1}}

\makeatletter

\DeclareMathOperator\erf{erf}

\makeatother

\begin{document}
\title{Conditional Simulation Using Diffusion Schr\"{o}dinger Bridges}
\author[1]{Yuyang~Shi}
\author[2]{Valentin~De~Bortoli}
\author[1]{George~Deligiannidis}
\author[1]{Arnaud~Doucet}

\affil[1]{%
    Department of Statistics\\
    University of Oxford, UK
}

\affil[2]{%
   ENS, PSL University, Paris, France 
}

\maketitle

\begin{abstract}
  Denoising diffusion models have recently emerged as a powerful class of generative
  models. They provide state-of-the-art results, not
  only for unconditional simulation, but also when used to solve conditional
  simulation problems arising in a wide range of inverse problems. A limitation of these models is that they are
  computationally intensive at generation time as they require simulating a
  diffusion process over a long time horizon. When performing unconditional
  simulation, a Schr\"odinger bridge formulation of generative modeling leads to
  a theoretically grounded algorithm shortening generation time which is
  complementary to other proposed acceleration techniques.  We extend the
  Schr\"odinger bridge framework to conditional simulation. We demonstrate this
  novel methodology on various applications including image super-resolution, 
  optimal filtering for state-space models and the refinement of pre-trained networks. Our code can be found at \href{https://github.com/vdeborto/cdsb}{\texttt{https://github.com/vdeborto/cdsb}}. 
\end{abstract}

\section{Introduction}\label{sec:intro}
\emph{Score-Based Generative Models} (SGMs), also known as denoising diffusion models, are a class of generative models that have become recently very popular as they provide state-of-the-art performance; see \eg~
\cite{chen2020wavegrad,ho2020denoising,song2020score,saharia2021image,nichol2021beatgans}.
Existing SGMs proceed as follows. First, noise is gradually added to the data using a time-discretized diffusion so
as to provide a sequence of perturbed data distributions eventually
approximating an easy-to-sample reference distribution, typically a multivariate Gaussian. Second, one approximates the corresponding time-reversed
denoising diffusion using neural network approximations of the logarithmic
derivatives of the perturbed data distributions known as scores; these
approximations are obtained using denoising score matching techniques
\citep{vincent2011connection, hyvarinen2005estimation}. Finally, the generative
model is obtained by initializing this reverse-time process using samples from
the reference distribution
\citep{ho2020denoising,song2020score}.

In many applications, one is not interested in unconditional simulation but the generative model is used as an implicit prior
$\pdata (x)$ on some parameter $X$ (e.g. image) in a Bayesian
inference problem with a likelihood function $g(\yobs|x)$ for observation
$Y=\yobs$. SGMs have been extended to address such tasks, see
\eg~\cite{song2020score,saharia2021image,batzolis2021conditional,tashiro2021csdi}. In
this conditional simulation case, one only requires being able to simulate from
the joint distribution of data and synthetic observations
$(X,Y)\sim \pdata (x)g(y|x)$. As in the unconditional case, the
time-reversal of the noising diffusion is approximated using neural network
estimates of its scores, the key difference being that this network admits not
only $x$ but also $y$ as an input. Sampling from the posterior
$p(x|\yobs) \propto \pdata (x) g(\yobs|x)$ is achieved by
simulating the time-reversal using the scores evaluated
at $Y=\yobs$.

However, performing unconditional or conditional simulation using SGMs is
computationally expensive as, to obtain a good approximation of the
time-reversed diffusion, one needs to run the forward noising diffusion long enough
to converge to the reference distribution. Many techniques have been proposed to
accelerate simulation including \eg~ knowledge distillation
\citep{luhman2021knowledge,salimans2022progressive}, non-Markovian forward process and subsampling
\citep{song2020denoising}, optimized noising diffusions and improved numerical solvers
\citep{jolicoeur2021gotta,dockhorn2021score,kingma2021variational,watson2022learning}. In
the unconditional scenario, reformulating generative modeling as a \schro bridge
(SB) problem provides a principled theoretical framework to accelerate
simulation time complementary to most other acceleration techniques
\citep{debortoli2021neurips}.
The SB solution is
the finite time process which is the closest in terms of Kullback--Leibler (KL)
discrepancy to the forward noising process used by SGMs but admits as marginals
the data distribution at time $t=0$ and the reference distribution at time
$t=T$. The time-reversal of the SB thus enables unconditional generation from
the data distribution.
However, the use of the SB formulation has not yet been developed in the context of conditional simulation. 

The contributions of this paper are as follows.
\begin{itemize}
\item We develop conditional SB (CSB), an original SB formulation for conditional simulation.    
\item By adapting the Diffusion SB algorithm of \cite{debortoli2021neurips} to
  our setting, we propose an iterative algorithm, Conditional Diffusion SB
  (CDSB), to approximate the solution to the CSB problem.

\item CDSB performance is demonstrated on various examples. In particular,
  we propose the first application of score-based techniques to optimal
  filtering in state-space models.
\end{itemize}

\section{Score-Based Generative Modeling}\label{sec:SGM}

    \subsection{Unconditional Simulation}
    \label{sec:discr-sett-mark}
    Assume we are given samples from some data distribution with positive
    density\footnote{We assume here that all
      distributions admit a positive density w.r.t. Lebesgue measure.} $\pdata$ on $\mathbb{R}^d$. Our aim is to
    provide a generative model to sample new data from $\pdata$.  SGMs 
    achieve this as follows. We gradually add noise to data samples, i.e. we
    consider a Markov chain $x_{0:N}=\{x_k\}_{k=0}^N \in \mcx = (\rset^d)^{N+1}$
    of joint density
\begin{equation}\label{eq:mu_forward}
           \textstyle{p(x_{0:N}) = p_0(x_0) \prod_{k=0}^{N-1}p_{k+1|k}(x_{k+1}|x_{k}),}
\end{equation}
where $p_0=\pdata$ and $p_{k+1|k}$ are Markov transition densities
inducing the following marginal densities
$p_{k+1}(x_{k+1})=\int
p_{k+1|k}(x_{k+1}|x_{k})p_{k}(x_{k})\textrm{d}x_{k}$. These transition densities
are selected such that $p_N(x_N) \approx \pref(x_N)$ for large $N$, where
$\pref$ is an easy-to-sample \emph{reference} density. In practice we set
$\pref(x_N)=\mathcal{N}(x_N;0,\Id)$, while 
$p_{k+1|k}(x_{k+1}|x_{k})=\mathcal{N}(x_{k+1};x_k-\gamma_{k+1}x_k;2 \gamma_{k+1} \Id)$
for $\gamma_k>0$, $\gamma_k \ll 1$ so $x_{0:N}$ is a time-discretized Ornstein--Uhlenbeck diffusion (see supplementary for details).

The main idea behind SGMs is to obtain samples from $p_0$ by exploiting the backward decomposition of \eqref{eq:mu_forward}  
\begin{equation}\label{eq:timereversal}
  \textstyle{
    p(x_{0:N}) = p_N(x_N)\prod_{k=0}^{N-1}p_{k|k+1}(x_{k}|x_{k+1}),
    }
\end{equation}
i.e.\ by sampling $X_N\sim p_N(x_N)$ then sampling $X_k\sim p_{k|k+1}(x_k|X_{k+1})$
for $k \in \{N-1, \dots, 0\}$, we obtain $X_0 \sim p_0(x_0)$.  In practice, we know neither $p_N$ nor the backward transition densities $p_{k|k+1}$ for
$k\in\{0,...,N-1\}$ and therefore this ancestral sampling procedure cannot be implemented
exactly. We thus approximate $p_N$ by $\pref$ and $p_{k|k+1}$ using a Taylor expansion approximation
\begin{equation}
\textstyle{p_{k|k+1}(x_k|x_{k+1})\approx  \mathcal{N}(x_k;B_{k+1}(x_{k+1}), 2\gamma_{k+1}  \Id),}
\end{equation}
where $B_{k+1}(x)=x+\gamma_{k+1} \{x + 2 \nabla \log p_{k+1}(x)\}$. Finally, we
approximate the score terms $\nabla\log p_{k}$ using denoising score matching
methods \citep{hyvarinen2005estimation,vincent2011connection,song2020score}. Since $p_{k}(x_{k})=\int p_{0}(x_{0})p_{k|0}(x_{k}|x_{0})\textrm{d}x_{0}$, it
follows that $\nabla \log p_{k}(x_{k})=\mathbb{E}[\nabla_{x_{k}} \log
        p_{k|0}(x_{k}|X_0)]$, 
where the expectation is w.r.t. to the distribution of $X_0$ given $x_{k}$.  We learn a neural network approximation
$\mathbf{s}_{\theta^\star}(k,x_k) \approx \nabla \log p_{k}(x_k)$ by minimizing
w.r.t. $\theta$ the loss
\begin{equation}\label{eq:scorematching}
\textstyle{\mathbb{E}[\sum_{k=1}^{N} \lambda_k ||\mathbf{s}_\theta(k, X_{k})-\nabla_{x_k} \log p_{k|0}(X_{k}|X_{0})||^2] },
\end{equation}
where $\lambda_k>0$ is a weighting coefficient \citep{ho2020denoising,song2020score} and the expectation is w.r.t. $p(x_{0:N})$. Once we have estimated
$\theta^\star$ from noisy data, we start by first sampling
$X_N \sim \pref(x_N)$ and then sampling
$X_k \sim \hat{p}_{k|k+1}(x_k|X_{k+1})$ for $\hat{p}_{k|k+1}$ as in
$p_{k|k+1}$ but with $\nabla \log p_{k+1}(X_{k+1})$ replaced by
$\mathbf{s}_{\theta^\star}(k+1,X_{k+1})$.  Under regularity assumptions, the resulting $X_0$ can be
shown to be approximately distributed according to $p_0=\pdata$ if $p_N\approx \pref$ \cite[Theorem
1]{debortoli2021neurips}.
        
\subsection{Conditional Simulation}\label{sec:condSGMs}
We now consider the scenario where we have samples from $p_0=\pdata$ and are
interested in generating samples from the posterior
$p(x|\yobs) \propto p_0(x) g(\yobs|x)$ for some observation $Y=\yobs \in
\mcy$. Here it is assumed that it is possible to sample synthetic observations from
$Y|(X=x)\sim g(y|x)$ but the expression of $g(y|x)$ might not be available.

In this case, conditional SGMs (CSGMs) proceed as follows; see
e.g. \cite{saharia2021image,batzolis2021conditional,li2022srdiff,tashiro2021csdi}. For any
realization $Y=y$, we consider a Markov chain of the form \eqref{eq:mu_forward}
but initialized using $X_0 \sim p(x|y)$ instead of $p_0(x)$. Obviously it is not
possible to simulate this chain but this will not prove necessary. This chain
induces for $k\geq 0$ the marginals denoted $p_{k+1}(x_{{k+1}}|y)$ which satisfy
$p_{k+1}(x_{{k+1}}|y)=\int
p_{{k+1}|k}(x_{k+1}|x_{k})p_{k}(x_{k}|y)\textrm{d}x_{k}$ for $p_0(x_0|y)=p(x_0|y)$. Similarly to the
unconditional case, to perform approximate ancestral sampling from this Markov
chain, we need to sample from
$p_{k|k+1}(x_{k}|x_{k+1},y)\approx \mathcal{N}(x_{k};B_{k+1}(x_{k+1},y), 2\gamma_{k+1} \Id)$ where 
$B_{k+1}(x,y)=x + \gamma_{k+1} \{x+2 \nabla \log p_{k+1}(x|y)\}$. We can again
estimate these score terms using
\begin{equation}
        \nabla \log p_{k}(x_{k}|y)=\mathbb{E}[\nabla_{x_{k}} \log p_{k|0}(x_{k}|X_{0})],
\end{equation}
where the expectation is w.r.t. to the distribution of $X_0$ given
$(X_{k},Y)=(x_{k},y)$.  In this case, we learn again a neural network approximation
$\mathbf{s}_{\theta^\star}(k,x_k,y) \approx \nabla \log p_{k}(x_k|y)$ by minimizing
w.r.t. $\theta$ the loss
\begin{equation}\label{eq:scorematchingcond}
        \textstyle{\mathbb{E}[\sum_{k=1}^{N} \lambda_k ||\mathbf{s}_\theta(k, X_{k},Y)-\nabla_{x_k} \log p_{k|0}(X_{k}|X_{0})||^2] },
\end{equation}
where the expectation is w.r.t. $p(x_{0:N})g(y|x_0)$ which we can sample
from.  Once the neural network is trained, we simulate from the
posterior $p(x|\yobs) \propto p_0(x) g(\yobs|x)$ for any observation $Y=\yobs$
 as follows: sample first $X_N \sim \pref(x_N)$ and then         $X_k \sim \hat{p}_{k|k+1}(x_k|X_{k+1},\yobs)$ where this density is
similar to $p_{k|k+1}(x_k|X_{k+1},\yobs)$ but with
$\nabla \log p_{k+1}(X_{k+1}|\yobs)$ replaced by
$\mathbf{s}_{\theta^\star}(k+1,X_{k+1},\yobs)$. The resulting sample $X_0$ will
be approximately distributed according to $p(x|\yobs)$. This scheme can
be seen as an amortized variational inference
procedure.

\begin{figure*}
\centering
\subfloat[\label{fig:sbdiagram}]{
\resizebox{.46\textwidth}{!}{
\begin{tikzpicture}[%
  block/.style = {rounded corners, fill=gray!10, align=center, text width=2cm, minimum height=3cm, inner sep=0}]

\node[text width=2cm,align=center] (note1) {$\pdata(x)$}; 
\node[block,below of=note1,yshift=-0.8cm] (tray1) {\includegraphics[width=.75\textwidth]{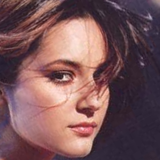}};
\node[right of=note1,xshift=6cm,text width=2cm,align=center] (note2) {$\pref(x)$};
\node[block,below of=note2,yshift=-0.8cm] (tray2) {\includegraphics[width=.75\textwidth]{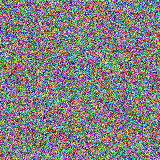}};

\draw[draw=red!30,line width=2cm,{Triangle[width=3cm,length=10pt]}-{Triangle[width=3cm,length=10pt]}] (tray1)  -- (tray2); 
\draw[draw=blue,latex-, thick] ([xshift=-0.2cm] tray1.east)  -- ([xshift=0.2cm] tray2.west) 
        node[midway,text=black]{\includegraphics[width=30pt]{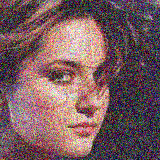} \, \includegraphics[width=30pt]{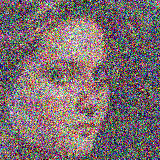} \, \includegraphics[width=30pt]{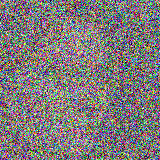}};

\node[] () [below = 41pt] at (tray1.east) {};
\end{tikzpicture} 
}
} \quad
\subfloat[\label{fig:csbdiagram}]{
\resizebox{.46\textwidth}{!}{
\begin{tikzpicture}[
  block/.style = {rounded corners, fill=gray!10, align=center, text width=2cm, minimum height=3cm, inner sep=0}]

\node[text width=2cm,align=center] (note1) {$\pjoin(x,y)$}; 
\node[block,below of=note1,yshift=-0.8cm, fill=gray!60] (tray1) {};
\node[block,below of=note1,yshift=-0.3cm,minimum height=1.9cm, text width=1.8cm] (tray1x) {\small{$p(x|y)$}\\\includegraphics[width=.75\textwidth]{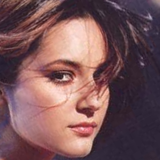}};
\node[block,below of=note1,yshift=-1.8cm, fill=gray!20,minimum height=0.9cm, text width=1.8cm] (tray1y) {\small{$\pobs(y)$}\\\includegraphics[width=.25\textwidth]{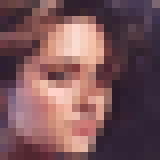}};

\node[right of=note1,xshift=6cm,text width=2cm,align=center] (note2) {$\pjref(x,y)$};
\node[block,below of=note2,yshift=-0.8cm, fill=gray!60] (tray2) {};
\node[block,below of=note2,yshift=-0.3cm,minimum height=1.9cm, text width=1.8cm] (tray2x) {\small{$\pref(x|y)$}\\\includegraphics[width=.75\textwidth]{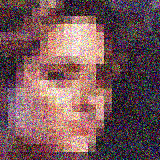}};
\node[block,below of=note2,yshift=-1.8cm, fill=gray!20,minimum height=0.9cm, text width=1.8cm] (tray2y) {\small{$\pobs(y)$}\\\includegraphics[width=.25\textwidth]{Plots/Diagram/2c/y.png}};

\draw[draw=blue,-latex, thick] ([xshift=-0.15cm] tray2y.east)  to[bend right]  ([shift=({-0.15cm,-0.1cm})] tray2x.east); 

\draw[draw=red!30,line width=2.5cm,{Triangle[width=3.2cm,length=10pt]}-{Triangle[width=3.2cm,length=10pt]}] ([yshift=-0.1cm] tray1.east)  -- ([yshift=-0.1cm] tray2.west) 
        node[pos=0.225,rounded corners, text=black,fill=black!10,minimum height=2.25cm,text width=1cm,align=center]{
        \vspace{1.45cm}
        \\
        \includegraphics[width=12pt]{Plots/Diagram/2c/y.png}}
        node[pos=0.5,rounded corners, text=black,fill=black!10,minimum height=2.25cm,text width=1cm,align=center]{
        \vspace{1.45cm}
        \\
        \includegraphics[width=12pt]{Plots/Diagram/2c/y.png}}
        node[pos=0.775,rounded corners, text=black,fill=black!10,minimum height=2.25cm,text width=1cm,align=center]{
        \vspace{1.45cm}
        \\
        \includegraphics[width=12pt]{Plots/Diagram/2c/y.png}};

\draw[draw=blue,latex-, thick] ([shift=({-0.2cm,-0.2cm})] tray1x.east)  -- ([shift=({0.2cm,-0.2cm})] tray2x.west) 
        node[pos=0.255,text=black]{\includegraphics[width=30pt]{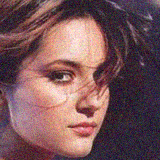}}
        node[pos=0.5,text=black]{\includegraphics[width=30pt]{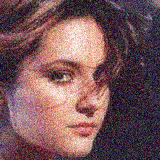}}
        node[pos=0.745,text=black]{\includegraphics[width=30pt]{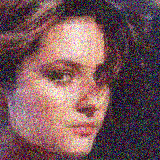}};

\draw[draw=black!80,latex-latex] ([shift=({0.1cm,-0.15cm})] tray1y.east)  -- ([shift=({-0.1cm,-0.15cm})] tray2y.west)
        node[midway, below, yshift=-0.125cm] {\footnotesize{fixed $y$}};

\node[above,rounded corners,draw=black!80,text width=9.25cm,minimum height=1.25cm,align=center] at (current bounding box.south) {};

\end{tikzpicture}
}
}

\caption{(a) An unconditional \schro bridge (SB) between $\pdata(x)$ and $\pref(x)$; (b) our proposed conditional \schro bridge (CSB) on the extended space between $\pjoin(x,y)$ and $\pjref(x,y)$. The blue arrows denote the direction of the generative procedure at simulation time. }

\end{figure*}

\section{\schro Bridges and Generative Modeling}
\label{sec:schro-bridges}

For SGMs to work well, we must diffuse the process long enough so that 
$p_N \approx \pref$. The SB methodology introduced in
\citep{debortoli2021neurips} allows us to mitigate this problem.  We refer to
\citet{chen2020optimal} for recent reviews on
the SB problem. We first
recall how the SB problem can be applied to perform unconditional
simulation.  

Consider the \emph{forward} density $p(x_{0:N})$ given by \eqref{eq:mu_forward}, describing the process adding noise to the data.  We want to find the joint density $\pi^\star(x_{0:N})$ such
that
\begin{equation}
  \label{eq:discrete_schro}
 \textstyle{\pi^\star = \argmin_{\pi} \ensemble{\KLLigne{\pi}{p}}{ \pi_0=\pdata,~ \pi_N= \pref}},
\end{equation}
where $\pi_0$,
resp. $\pi_N$, is the marginal of $X_0$, resp. $X_N$, under $\pi$.  
A visualization of the SB problem \eqref{eq:discrete_schro} is provided in Figure \ref{fig:sbdiagram}. 
Were $\pi^\star$ available, we would obtain a generative model by ancestral sampling: sample $X_N \sim \pref(x_N)$, then $X_k \sim \pi^\star_{k|k+1}(x_k | X_{k+1})$ for
$k \in \{N-1, \dots, 0\}$. 

The SB problem does not admit a closed-form solution
but it can be solved numerically using Iterative Proportional Fitting (IPF)
\citep{kullback1968probability}. This algorithm defines the following recursion initialized at $\pi^0=p$ given in
\eqref{eq:mu_forward}:
\begin{align}\label{eq:IFPrecursion}
  &\textstyle{\pi^{2n+1} = \argmin_{\pi} \ensemble{\KLLigne{\pi}{\pi^{2n}}}{\pi_N = \pref},} \\
  &\textstyle{\pi^{2n+2} =  \argmin_{\pi}  \ensemble{\KLLigne{\pi}{\pi^{2n+1}}}{\pi_0 = \pdata}. }
\end{align}
\cite{debortoli2021neurips,vargas2021solving} showed that the IPF iterates admit a representation suited to numerical approximation. Indeed, if we denote $p^n=\pi^{2n}$ and $q^n=\pi^{2n+1}$, then $p^0(x_{0:N})=p(x_{0:N})$ and
\begin{align}
     & \textstyle{q^n(x_{0:N}) = \pref(x_N) \prod_{k=0}^{N-1} q^n_{k|k+1}(x_k|x_{k+1}),}\\
     & \textstyle{p^{n+1}(x_{0:N}) = \pdata(x_0) \prod_{k=0}^{N-1} p^{n+1}_{k+1|k}(x_{k+1}|x_{k}),}
\end{align} 
where 
$q^n_{k|k+1} = p^n_{k|k+1}$  and
$p^{n+1}_{k+1|k} = q^n_{k+1|k}$. 
To summarize, at step $n=0$, 
$q^0$ is the backward
process obtained by reversing the dynamics of $p^0$ initialized at time $N$
from $\pref$. The forward process $p^1$ is then obtained from the reversed dynamics
of $q^0$ initialized at time $0$ from $\pdata$, and so on. Note that $q^0$ corresponds to the unconditional SGM described in \Cref{sec:discr-sett-mark}. 

\subsection{Diffusion \schro Bridge}

Similarly to SGMs, one can approximate the time-reversals appearing in the IPF iterates using score matching ideas. If
$p^{n}_{k+1|k}(x'|x)= \mathcal{N}(x';x+\gamma_{k+1} f^{n}_k(x),2 \gamma_{k+1}
\Id)$, with $f^0_k(x)=-x$, we approximate the reverse-time
transitions by
$q^n_{k|k+1}(x|x') \approx \mathcal{N}(x;x'+ \gamma_{k+1}
b^{n}_{k+1}(x'), 2\gamma_{k+1} \Id)$, where
$b^{n}_{k+1}(x') = -f^{n}_{k}(x')+2 \nabla \log p^{n}_{k+1}(x')$; and next
$p^{n+1}_{k+1|k}(x'|x) \approx \mathcal{N}(x';x+ \gamma_{k+1}
f^{n+1}_{k}(x), 2\gamma_{k+1} \Id)$, where
$f^{n+1}_{k}(x) = -b^{n}_{k+1}(x)+2 \nabla \log q^{n}_{k}(x)$. The drifts
$b^n_{k+1},f^{n+1}_k$ could be estimated by approximating
$\{\nabla \log p^{i}_{k+1}(x)\}_{i=0}^{n}$,
$\{\nabla \log q^{i}_{k}(x)\}_{i=0}^{n}$ using score matching. However this is too expensive both in terms of compute and memory. \cite{debortoli2021neurips} instead
directly approximate the mean of the Gaussians using neural networks, $\mathbf{B}_{\theta}$ and $\mathbf{F}_{\phi}$, by
generalizing the score matching approach, i.e.
$ q_{k|k+1}^n(x|x') = \mathcal{N}(x;\mathbf{B}_{\theta^n}(k+1,x'), 2 \gamma_{k+1}
\Id)$ and
$p_{k+1|k}^n(x'|x) = \mathcal{N}(x';\mathbf{F}_{\phi^n}(k,x), 2\gamma_{k+1} \Id)$,
where $\theta^n$ is obtained by minimizing
\begin{equation} \textstyle{\ell^{b}_n(\theta)=\mathbb{E}_{p^{n}}[\sum_{k}\normLigne{\mathbf{B}_\theta(k+1,X_{k+1})-G_{n,k}(X_k,X_{k+1})}^2]},
\end{equation}
for $G_{n,k}(x,x')=x'+\mathbf{F}_{\phi^n}(k,x)-\mathbf{F}_{\phi^n}(k,x')$, 
and $\phi^{n+1}$ by minimizing 
\begin{equation} \textstyle{\ell^{f}_{n+1}(\phi)=\mathbb{E}_{q^{n}}[\sum_{k}\normLigne{\mathbf{F}_\phi(k,X_k)-H_{n,k}(X_k,X_{k+1})}^2]},
\end{equation}
for $H_{n,k}(x,x')=x + \mathbf{B}_{\theta^n}(k+1,x')-\mathbf{B}_{\theta^n}(k+1,x)$.
This implementation of IPF, referred to as Diffusion SB (DSB), is presented in
the supplementary; see \cite{vargas2021solving,chen2021likelihood} for alternative
numerical schemes. After we have learned $\theta^L$ using $L$ DSB iterations, we
sample $X_N \sim \pref(x_N)$ and then set
$X_{k} = \mathbf{B}_{\theta^L}(k+1, X_{k+1})+ \sqrt{2 \gamma_{k+1}} Z_{k+1}$ with
$Z_k \overset{\textup{i.i.d.}}\sim \mathcal{N}(0,\Id)$ to obtain
$X_0$ approximately distributed from $\pdata$.

\subsection{Link With Optimal Transport}
\label{subsec:linkwithot}
It can be shown that the solution $\pi^\star$ of the SB problem \eqref{eq:discrete_schro}, $\pi^\star(x_{0:N})=\pi^{s,\star}(x_0,x_N)p_{|0,N}(x_{1:N-1}|x_0,x_N)$ where $\pi^{s,\star}(x_0,x_N)$ is the marginal of $\pi^\star(x_{0:N})$ at times $0$ and $N$. In this case, \eqref{eq:discrete_schro} reduces to the static SB problem
\begin{equation}
 \label{eq:static_schro}
 \textstyle{\pi^{s,\star} = \argmin_{\pi^s} \ensemble{\KLLigne{\pi^s}{p_{0,N}}}{ \pi^s_0=\pdata,~ \pi^s_N= \pref}}.
\end{equation}
The static SB problem can be interpreted as an entropy-regularized optimal transport problem between $\pdata$ and $\pref$, with regularized transportation cost $\mathbb{E}_{\pi^s} [-\log p_{N|0}(X_N|X_0)] - H(\pi^s)$. 
When $p_{N|0}(x_N|x_0) = \mathcal{N} (x_N ; x_0, \sigma^2)$ as in \cite{song2019generative}, the transportation cost $-\log p_{N|0}(x_N|x_0)$ reduces to the quadratic cost $\frac{1}{2\sigma^2}\|x_0-x_n\|^2$ up to a constant. In other words, the static SB solution $\pi^{s,\star}$
not only transports samples $X_N\sim\pref$ into samples from the data distribution $\pdata$, but also seeks to minimize an entropy-regularized Wasserstein distance of order $2$.
The regularization strength is controlled by the variance $\sigma^2$. 
Similar properties hold for the time-discretized Ornstein--Uhlenbeck diffusion defined by \eqref{eq:mu_forward} in Section \ref{sec:discr-sett-mark}.

\section{Conditional Diffusion \schro Bridge}
\label{sec:cond-simul-sb}

We now want to use SBs for conditional simulation, i.e.\ to be able sample from
a posterior distribution $p(x|\yobs) \propto \pdata(x) g(\yobs|x)$ assuming only
that it is possible to sample $(X,Y)\sim \pdata (x) g(y|x)$.  In
this case, an obvious approach would be to consider the SB problem where we
replace $\pdata(x)$ by the posterior $p(x|\yobs)$, i.e.
\begin{equation}\label{eq:discrete_schro-cond1}
 \textstyle{\pi^{\star}= \argmin_{\pi} \ensemble{\KLLigne{\pi}{p_{\yobs}}\hspace{-.15cm}}{ \hspace{-.15cm} \pi_0=p(\cdot|\yobs),~ \pi_N= \pref}, }
\end{equation}
where $p_{\yobs}(x_{0:n}):=p(x_0|\yobs) \prod_{k=0}^{N-1}p_{k+1|k}(x_{k+1}|x_k)$ is the forward noising process. However, DSB is not applicable here as it requires sampling from $p(x_{0}|\yobs)$ at step $0$.

We propose instead to solve an amortized problem. Let us introduce
$\pjoin(x,y)=\pdata(x)g(y|x)=p(x|y)\pobs(y)$ and $\pjref(x,y)=\pref(x)\pobs(y)$
where $\pobs(y)=\int \pdata(x)g(y|x) \rmd x$. We are interested in finding the
transition kernel $\pi^{c,\star} = (\pi^{c,\star}_y)_{y \in \mcy}$,
where
$\pi^{c,\star}_{y}$ defines a distribution on $\mcx = (\rset^d)^{N+1}$ for each
$y \in \mcy$, satisfying
\begin{align}
&\pi^{c,\star}= \text{argmin}_{\pi^c} \{\mathbb{E}_{Y \sim \pobs}[\textup{KL}(\pi_{Y}^c||p_{Y})]: \\
&\qquad \qquad  \pi_{0}^c\otimes \pobs= \pjoin,~ \pi_{N}^c \otimes \pobs =\pjref\}.\label{eq:SBuncondreformulated}
\end{align}
This corresponds to an averaged version of
\eqref{eq:discrete_schro-cond1} over the distribution $\pobs(y)$ of
$Y$. The first constraint
$\pi^{c,\star}_{y,0}(x_0)\pobs(y)=\pjoin(x_0,y)=p(x_0|y)\pobs(y)$ ensures that
$\pi^{c,\star}_{y,0}(x_0)=p(x_0|y)$,  $\pobs$-almost surely. Similarly
$\pi^{c,\star}_{y,N}(x_N) = \pref(x_N)$, $\pobs$-almost surely.  Hence,
to obtain a sample from $p(x|\yobs)$ for a given $Y=\yobs$, we can sample
$X_N \sim \pref(x_N)$ then
$X_k|X_{k+1} \sim \pi^{c,\star}_{\yobs,k|k+1}(x_k|X_{k+1})$ for $k=N-1,...,0$
and $X_0$ is a sample from $p(x|\yobs)$.

We show here that \eqref{eq:SBuncondreformulated} can be reformulated as a SB on
an extended space, which we will refer to as Conditional SB (CSB), so the theoretical results for existence and uniqueness of the solution to the SB problem apply. 
\begin{proposition}\label{prop:SBreformulation} Consider the following SB problem
\begin{align}
 \mkern-18mu \bar{\pi}^\star = \textup{argmin}_{\bar{\pi}} \{&\textup{KL}(\bar{\pi}|\bar{p}):  \text{s.t.} \ \label{eq:conditionalSBextended}\bar{\pi}_0= \pjoin,~ \bar{\pi}_N =\pjref\},
\end{align}
where we define $\bar{p}(x_{0:N},y_{0:N}):=p_{y_0}(x_{0:N})\bpobs(y_{0:N})$ with
$\bpobs(y_{0:N}):=\pobs(y_0) \prod_{k=0}^{N-1}\delta_{y_k}(y_{k+1})$ and $p_{y_0}$ is the forward process defined below \eqref{eq:discrete_schro-cond1}. If $\KLLigne{\bar{\pi}^\star}{\bar{p}}<+\infty$ then $\bar{\pi}^\star=\pi^{c,\star} \otimes \bpobs$ where $\pi^{c,\star}$ solves
\eqref{eq:SBuncondreformulated}. 
\end{proposition}
 We provide an illustration of the CSB problem \eqref{eq:conditionalSBextended} in Figure \ref{fig:csbdiagram}. 
 Under $\bar{p}$, the $Y$-component is sampled at time $0$ according to $\pobs$ and then is kept constant until time $N$ while the $X$-component is initialized at $p(x|y_0)$ and then diffuses according to $p_{k+1|k}(x_{k+1}|x_k)$.
 
Contrary to \eqref{eq:discrete_schro-cond1}, we can adapt DSB to solve numerically the CSB problem  \eqref{eq:conditionalSBextended} as both the distributions $\pjoin$ and
 $\pjref$ can be sampled. The
resulting algorithm is called Conditional DSB (CDSB). It approximates the following IPF recursion 
\begin{align}\label{eq:IFPrecursion}
  &\textstyle{\bar{\pi}^{2n+1} = \argmin_{\bar{\pi}} \ensemble{\KLLigne{\bar{\pi}}{\bar{\pi}^{2n}}}{\bar{\pi}_N = \pjref},} \\
  &\textstyle{\bar{\pi}^{2n+2} =  \argmin_{\bar{\pi}}  \ensemble{\KLLigne{\bar{\pi}}{\bar{\pi}^{2n+1}}}{\bar{\pi}_0 = \pjoin}}
\end{align}
initialized at $\bar{\pi}^0=\bar{p}$. For $\bar{p}^n=\bar{\pi}^{2n}$ and $\bar{q}^n=\bar{\pi}^{2n+1}$, we have the following representation of the IPF iterates.
\begin{proposition}
  \label{prop:IPFrecursion}
  Assume that $\KLLigne{\pjoin \otimes \pjref}{\bar{p}_{0,N}} < +\infty$. Then
  we have $\bar{p}^0(x_{0:N},y_{0:N})=\bar{p}(x_{0:N},y_{0:N})$ and for any
  $n>0$, $\bar{q}^n(x_{0:N},y_{0:N})=\bpobs(y_{0:N})\bar{q}^n(x_{0:N}|y_{N})$,
  $\bar{p}^{n+1}(x_{0:N},y_{0:N})= \bpobs(y_{0:N})\bar{p}^{n+1}(x_{0:N}|y_0)$
  with 
 \begin{align}
      &\textstyle{\bar{q}^n(x_{0:N}|y_{N})= \pref(x_N)\prod_{k=0}^{N-1} \bar{p}^n_{k|k+1}(x_k|x_{k+1},y_{N}),} \\
    &\textstyle{\bar{p}^{n+1}(x_{0:N}|y_0)= p(x_0|y_0) \prod_{k=0}^{N-1} \bar{q}^{n}_{k+1|k}(x_{k+1}|x_{k},y_{0}).}
  \end{align} 
\end{proposition}
Here we simplify notation and write $Y$ for all the random variables
$Y_0,Y_1,...,Y_N$ as they are all equal almost surely under $\bar{p}^{n}$ and
$\bar{q}^{n}$. We approximate the transition kernels as in DSB and refer to the
supplementary for more details. In particular, the transition kernels satisfy
$\bar{q}^n_{k|k+1}(x|x',y)=
\mathcal{N}(x;\mathbf{B}^{y}_{\theta^n}(k+1,x'),2\gamma_{k+1} \Id)$ and
$\bar{p}_{k+1|k}^n(x'|x,y) = \mathcal{N}(x';\mathbf{F}^{y}_{\phi^n}(k,x),
2\gamma_{k+1} \Id)$, where $\theta^n$ is obtained by minimizing
\begin{equation} \textstyle{\ell^{b}_n(\theta)=\mathbb{E}_{\bar{p}^{n}}[\sum_{k}\normLigne{\mathbf{B}_\theta^Y(k+1,X_{k+1})-G_{n,k}^Y(X_{k},X_{k+1})}^2]}\label{eq:regressionbcond}
\end{equation}
for $G_{n,k}^y(x,x')=x'+\mathbf{F}^{y}_{\phi^n}(k,x)-\mathbf{F}^{y}_{\phi^n}(k,x')$ and $\phi^{n+1}$ by minimizing 
\begin{equation} \textstyle{\ell^{f}_{n+1}(\phi)=\mathbb{E}_{\bar{q}^{n}}[\sum_{k}\normLigne{\mathbf{F}_\phi^Y(k,X_k)-H_{n,k}^Y(X_k,X_{k+1})}^2]}\label{eq:regressionfcond},
\end{equation}
$H_{n,k}^y(x,x')=x + \mathbf{B}^{y}_{\theta^n}(k+1,x')-\mathbf{B}^{y}_{\theta^n}(k+1,x)$.

The resulting CDSB scheme is summarized in \Cref{algo:ipf_score_cond} where $Z^j_k,\tilde{Z}^j_k \overset{\textup{i.i.d.}}\sim \mathcal{N}(0,\Id)$. After $L$ iterations of CDSB, we have learned $\theta^L$. For any observation
$Y=\yobs$, we can then sample $X_N \sim \pref(x_N)$ and then compute
$X_{k} = \mathbf{B}_{\theta^L}^{\yobs}(k+1, X_{k+1})+ \sqrt{2 \gamma_{k+1}} Z_{k+1}$ with
$Z_k \overset{\textup{i.i.d.}}\sim \mathcal{N}(0,\Id)$ for $k=N-1,...,0$.  The resulting sample
$X_0$ will be approximately distributed from $p(x|\yobs)$.
\begin{algorithm}
    \caption{Conditional Diffusion \schro Bridge}
    \label{algo:ipf_score_cond}
    \begin{algorithmic}[1] 
      \FOR{$n \in \{0, \dots,L\}$} \WHILE{not converged}
      \STATE Sample $\{X^j_{k}\}_{k,j=0}^{N,M},\{Y^j\}_{j=0}^{M}$ where\\ $X^j_0 \sim \pdata, Y^j \sim g(\cdot|X^j_0)$, and \\
      $X^{j}_{k+1} = \mathbf{F}_{\phi^n}^{Y^j}(k, X^{j}_{k})+\sqrt{2
        \gamma_{k+1}} Z^{j}_{k+1}$ 
        \STATE Compute $\hlb_n(\theta^n)$ approximating \eqref{eq:regressionbcond}
        \STATE $\theta^{n} \leftarrow \textrm{Gradient Step}(\hlb_n(\theta^n))$ 
      \ENDWHILE \WHILE{not
        converged}
      \STATE Sample $\{X^j_{k}\}_{k,j=0}^{N,M}$,  $\{Y^j\}_{j=0}^{M}$ where\\ $X^j_N \sim \pref, Y^j\sim \pobs$, and \\
      $X^j_{k}=\mathbf{B}_{\theta^n}^{Y^j}(k+1, X^{j}_{k+1})+\sqrt{2 \gamma_{k+1}}
      \tilde{Z}^{j}_{k+1}$ 
      \STATE Compute $\hlf_{n+1}(\phi^{n+1})$ approximating \eqref{eq:regressionfcond}
      \STATE
      $\phi^{n+1} \leftarrow \textrm{Gradient Step}(\hlf_{n+1}(\phi^{n+1}))$
      \ENDWHILE \ENDFOR \STATE \textbf{Output: } $(\theta^{L},\phi^{L+1})$
    \end{algorithmic}
  \end{algorithm}
\section{CDSB Improvements}
\subsection{Conditional Reference Measure}
\label{subsec:targetawareinitial}
In standard SGMs and for the unconditional SB, we typically select
$\pref(x)=\mathcal{N}(x;0,\sigma_{\textup{ref}}^{2}\Id)$. However, initializing
ancestral sampling from random noise to eventually obtain samples from $p(x|y)$
can be inefficient as $y$ already contains useful information about $X$. Fortunately, it is easy to use a joint reference measure of the form
$\pjref(x,y)=\pref(x|y)\pobs(y)$ instead of $\pjref(x,y)=\pref(x)\pobs(y)$ in CSB and CDSB. The
only modification in \Cref{algo:ipf_score_cond} is that line 8 becomes
$Y^j \sim \pobs(y), X^j_N \sim \pref(x|Y^j)$.  In some interesting scenarios, we
can select $\pref(x|y)$ as an approximation to $p(x|y)$ in order
to accelerate the sampling process. This means we construct a CSB between $p(x|y)$ and its approximation $\pref(x|y)$, instead of between $p(x|y)$ and noise. We refer to this extension of CDSB as
CDSB-C.

As a simple example, consider obtaining super-resolution (SR) image samples from a low-resolution image $Y=y$. Assume that $y$
has been suitably upsampled to have the same dimensionality as $X$. In this
case, $y$ itself can serve as an approximate
initialization for sampling $X_N$. A simple model is to take
$\pref(x|y)=\vois(x;y,\sigma_{\textup{ref}}^{2} \Id)$ with $\sigma_{\textup{ref}}^{2}=\rho \sigma_{x|y}^{2}$, where $\rho$ is a variance
inflation parameter and $\sigma_{x|y}^{2}$ is an estimate of the conditional variance of $X$ given $Y$. 
See Figure \ref{fig:csbdiagram} for an illustration. 
In our experiments, we also explore other $\pref(x|y)$ obtained using the Ensemble Kalman Filter (EnKF) as well as neural network models.

\subsection{Conditional Forward Process}
\label{subsec:targetawareforward}
To accelerate the convergence of IPF, we also have the flexibility
to make the initial forward noising process dynamics dependent on $Y=y$, i.e. 
$p_{y}(x_{0:N})=p(x_{0}|y)\prod_{k=0}^{N-1}p_{k+1|k}(x_{k+1}|x_{k},y)$.
As shown below, it is  beneficial to initialize $p_{y}$ close to the CSB solution $\pi^{c,\star}_y$.
\begin{proposition}\label{prop:fasterconverence}
  For any $n \in \nset$ with $n \geq 1$, we have 
\begin{equation}
  \expeLigne{\KLLigne{\pi^{c,n}_{Y,0}}{p(\cdot|Y)}}
   \leq \tfrac{2}{n} \expeLigne{\KLLigne{\pi^{c,\star}_Y}{p_Y}} ,
\end{equation}
where for any $n \in \nset$, $\bar{\pi}^n = \bpobs \otimes \pi^{c,n}$ is the
$n^{\textup{th}}$ IPF iterate and the expectations are w.r.t. $Y \sim \pobs$.
\end{proposition}

As a result, we should choose the initial forward noising process $p_y$ such that its
terminal marginal $p_{y,N}$ targets $\pref(\cdot|y)$. However, contrary to
diffusion models, we recall that our framework does not strictly  require
$p_{y,N}\approx \pref(\cdot|y)$ to provide approximate samples from the
posterior of interest.

For tractable $\pref(x|y)$, we can define
$p_{y}(x_{0:N})$ using an unadjusted Langevin dynamics; i.e.
$p_{k+1|k}(x'|x,y)=\mathcal{N}(x';x+\gamma_{k+1}\nabla\log\pref(x|y),2\gamma_{k+1}\Id)$. In the case $\pref(x|y)=\mathcal{N}(x;\mu(y),\sigma^{2}(y)\Id)$, this reduces to a discretized Ornstein--Uhlenbeck process admitting $\pref(x|y)$ as limiting distribution as $\gamma\to0$ and $N\to\infty$ \citep{durmus2017nonasymp}.

\subsection{Forward-Backward Sampling}

When we use an unconditional $\pref(x)$, our
proposed method also shares connections with the conditional transport methodology
developed by \cite{marzouk2016sampling,spantini2019coupling}. They propose
methods to learn a deterministic invertible transport map
$\mathcal{S}(x,y):\mathcal{X}\times\mathcal{Y}\to\mathcal{X}$ which maps samples
from $p(x|y)$ to $\pref(x)$. To sample from $p(x|\yobs)$, one samples
$X^{\textup{ref}}\sim\pref(x)$, then transports back the sample through the
inverse map
$X^{\textup{pos}}={\mathcal{S}}(\cdot,\yobs)^{-1}(X^{\textup{ref}})$. 

As noted by \cite{spantini2019coupling}, an alternative method to sample from
$p(x|\yobs)$ consists of first sampling
$(X,Y)\sim\pjoin$, then following the two-step
transformation
$\hat{X}^{\textup{ref}}=\mathcal{S}(X,Y),~~
\hat{X}^{\textup{pos}}=\mathcal{S}(\cdot,\yobs)^{-1}(\hat{X}^{\textup{ref}})$.  By
definition of $\mathcal{S}$, $\hat{X}^{\textup{ref}}$ is also distributed
according to $\pref$.  However, since the transport map ${\mathcal{S}}$ may be
imperfect in practice, this sampling strategy provides the advantage of
cancellation of errors between $\mathcal{S}$ and $\mathcal{S}(\cdot,\yobs)^{-1}$.

We also explore an analogous forward-backward sampling scheme in our
framework, which first samples $(X,Y)\sim\pjoin$,  
followed by sampling  $\hat{X}_{N}\sim \bar{p}^{L}_{N|0}(x_{N}|X,Y)$ through the forward half-bridge, then $\hat{X}_{0}\sim
\bar{q}^{L}_{0|N}(x_{0}|\hat{X}_{N},\yobs)$ through the backward half-bridge.
Since $\bar{q}^{L}$ is the approximate time-reversal of $\bar{p}^{L}$, this
strategy shares similar advantages as the method of \cite{spantini2019coupling}
when the half-bridge $\bar{q}^{L}(x_{0:N}|\yobs)$ does not solve the CSB problem
exactly. We call this extension CDSB-FB.
     
\section{Related work}

\textbf{Approximate Bayesian computation (ABC)}, also known as
likelihood-free inference, has been developed to approximate the posterior when
the likelihood is intractable but one can simulate synthetic data from it; see
\eg \citep{beaumont2019approximate}. However, these
methods typically require knowing the prior, while CDSB only needs to have access to
joint samples and learns about the posterior directly. For tasks such as image inpainting, the prior is indeed implicit. 

\textbf{\schro bridges} techniques to perform both static and sequential Bayesian inference for state-space models have been developed by \cite{bernton2019schr} and \cite{reich2018data}. However, these methods require being able to evaluate pointwise an unnormalized version of the target posterior distribution contrary to the CDSB-based methods developed here.

\textbf{Conditional transport}. Performing conditional simulation by
learning a transport map between joint distributions on $X,Y$ having the same
$Y$-marginals (as $\pjoin$ and $\pref$) has been first proposed by
\cite{marzouk2016sampling}. Various techniques have
been subsequently developed to approximate such maps such as polynomial or radial basis
representations \citep{marzouk2016sampling,baptista2020adaptive}, Generative
Adversarial Networks \citep{kovachki2021conditional,zhou2021deep} or normalizing
flows \citep{kruse2021hint}. 
CDSB also fits into this framework, but instead utilizes stochastic transport maps. 
Recently, \cite{taghvaei2022optimal} have also proposed independently using conditional transport ideas to perform optimal filtering for state-space models.

\textbf{Conditional SGMs}. SGMs have been applied to perform posterior
simulation, primarily for images, as described in Section
\ref{sec:condSGMs} and references therein.
An alternative line of work for image editing
\citep{song2019generative,choi2021ilvr,chung2021comecloserdiffusefaster,meng2022sdedit} 
utilizes the denoising property of SGMs to iteratively denoise noisy versions of
a reference image $y$ while restricted to retain particular features of $y$. 
However, $\pref(x)=\mathcal{N}(x;0,\sigma_{\textup{ref}}^{2}\Id)$ so image generation is started
from noise and typically hundreds or thousands of refinement steps are required. Our framework can incorporate in a principled way information given by $y$ in the reverse process's initialization (see Section \ref{subsec:targetawareinitial}). Recently \cite{zheng2022truncated,lu2022conditional} have also proposed suitable choices for $\pref(x)$ or $\pref(x|y)$ to shorten the diffusion process. In comparison, the CDSB framework is more flexible and allows for general $\pref(x|y)$ which can be non-Gaussian and different from the initial forward diffusion's terminal distribution $p_N(x_N|y)$. For instance, we explore 
using noiseless pre-trained super-resolution models as $\pref(x|y)$ in Section \ref{subsubsec:nongaussianref}, where CDSB further improves the SR samples closer to the data distribution. 
Finally, for linear Gaussian inverse
problems, \cite{kadkhodaie2021stochastic,kawar2021snips,kawar2022denoising} develop efficient methodologies using unconditional SGMs when the linear degradation model and the Gaussian noise level are known.

\textbf{SGM acceleration techniques}.  Many techniques have been proposed to
accelerate SGMs and CSGMs. For example, \cite{luhman2021knowledge,
  salimans2022progressive} propose to learn a distillation network on top of SGM
models, while \cite{song2020denoising} perform a subsampling of the timesteps in
a variational setting. \cite{watson2022learning} optimize the timesteps with a
fixed budget using dynamic programming.  
\cite{xiao2021tackling} perform multi-steps denoising
using GANs while \cite{dockhorn2021score} consider underdamped Langevin dynamics
as forward process. We emphasize that many of these techniques are complementary
to and can be readily applied in the SB setting; \eg  one could distill the last CDSB network $\mathbf{B}^{y}_{\theta^L}$.  Additionally, SB
and CSB provide a framework to perform few-step sampling.

\section{Experiments}
\label{sec:experiments}
\begin{figure}[h]
\centering
\begin{minipage}{\linewidth}
\raisebox{0.8cm}{\rotatebox[origin=t]{90}{CSGM}}\enskip{}\includegraphics[height=2.13cm]{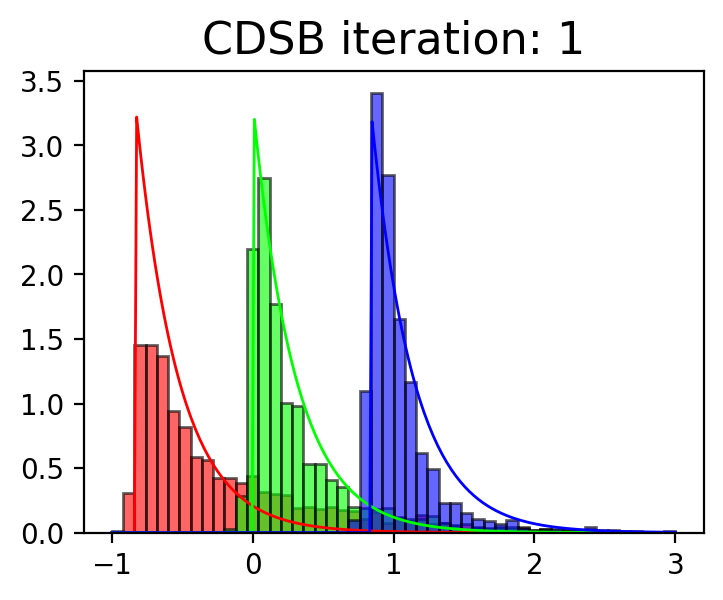}\includegraphics[height=2.13cm]{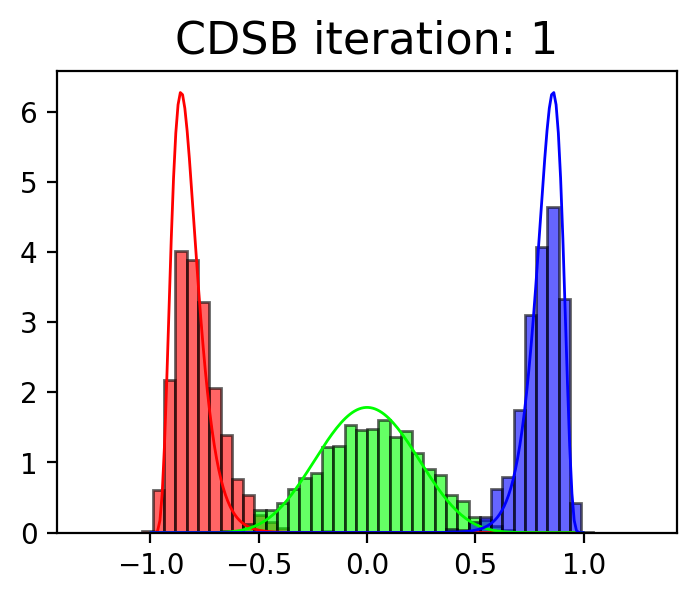}\includegraphics[height=2.13cm]{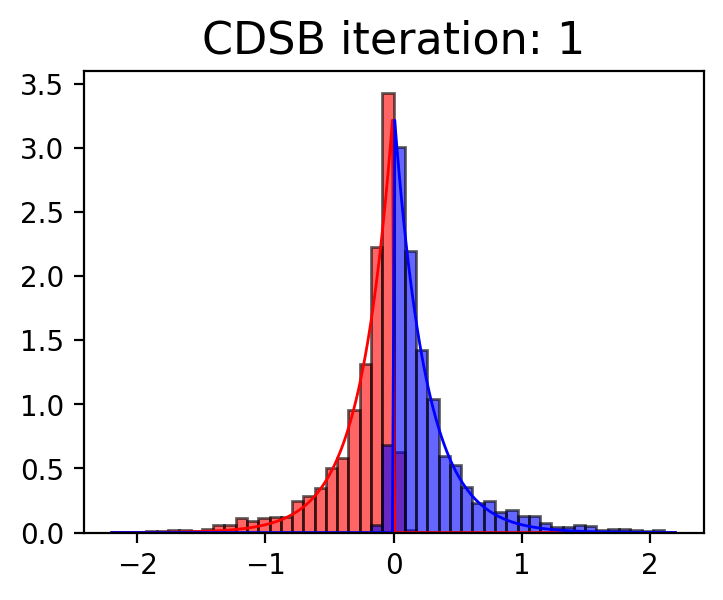}

\raisebox{0.8cm}{\rotatebox[origin=t]{90}{CDSB}}\enskip{}\includegraphics[height=2.13cm]{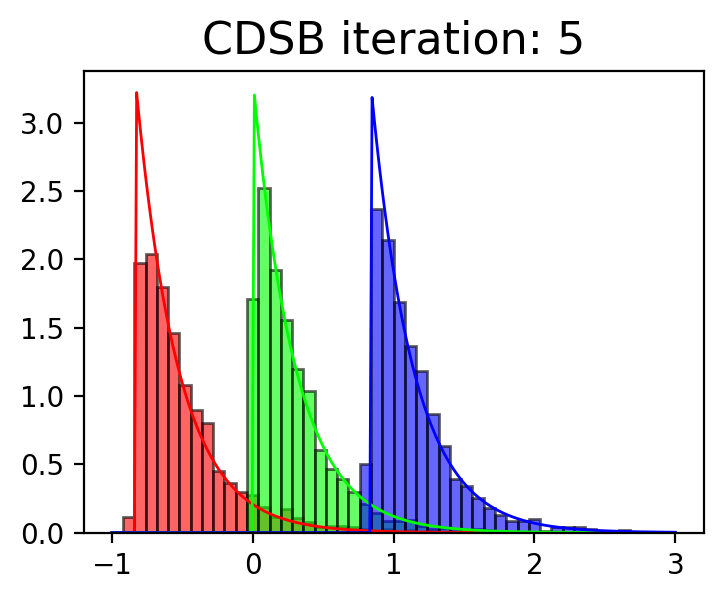}\includegraphics[height=2.13cm]{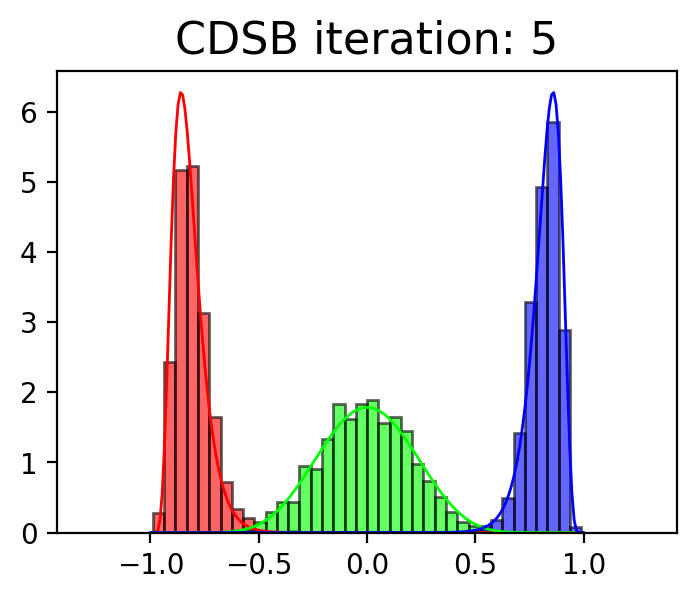}\includegraphics[height=2.13cm]{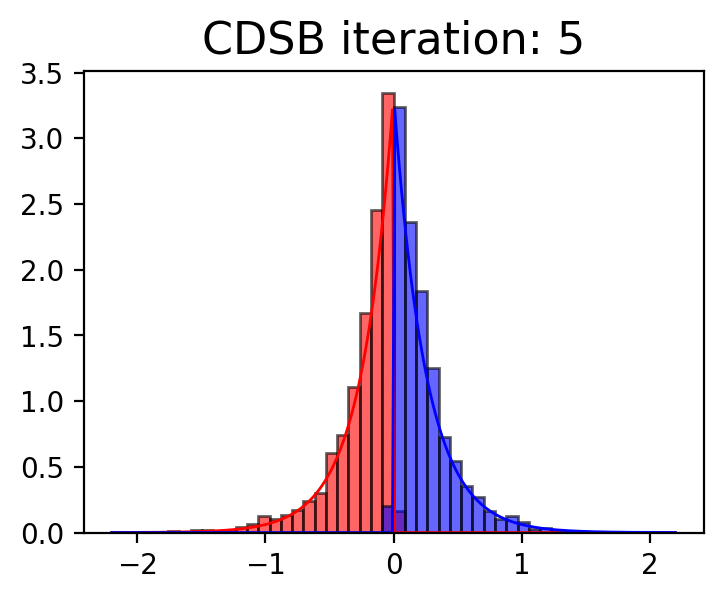}

\raisebox{0.8cm}{\rotatebox[origin=t]{90}{CDSB-FB}}\enskip{}\includegraphics[height=2.13cm]{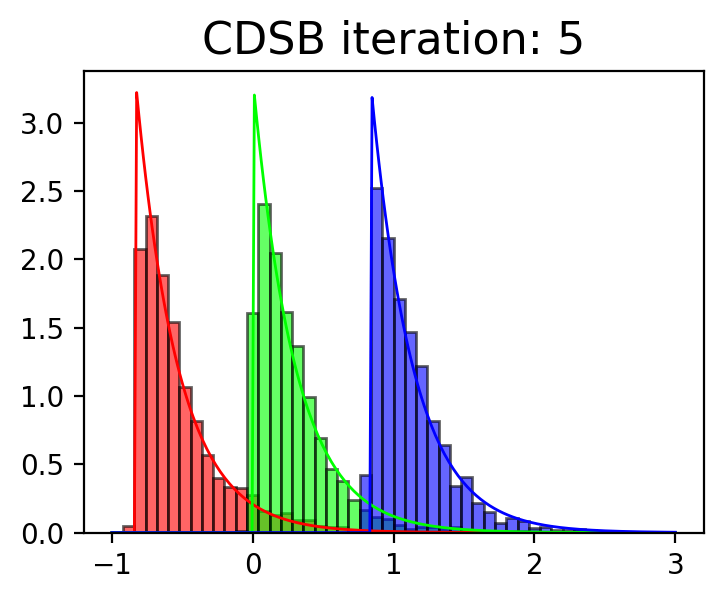}\includegraphics[height=2.13cm]{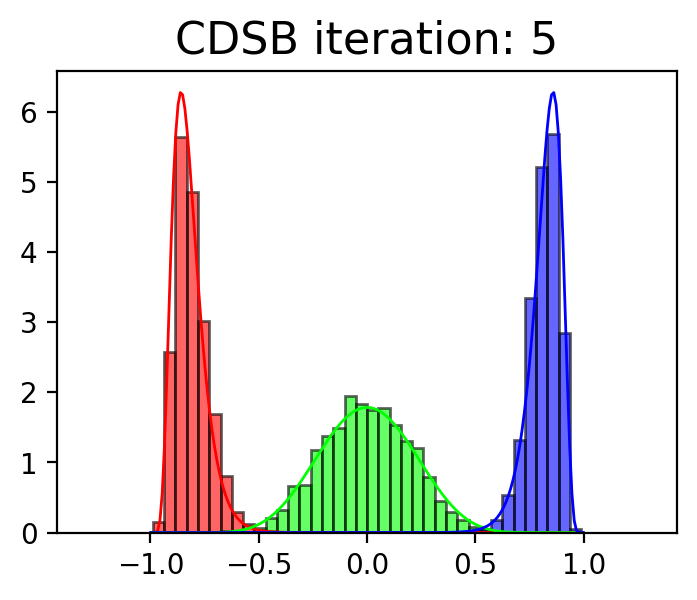}\includegraphics[height=2.13cm]{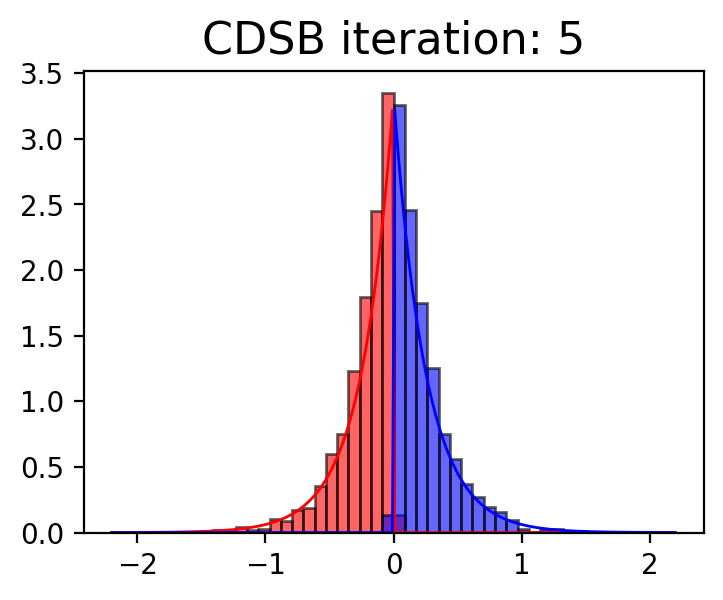}

\raisebox{0.8cm}{\rotatebox[origin=t]{90}{MGAN}}\enskip{}\includegraphics[height=2.13cm]{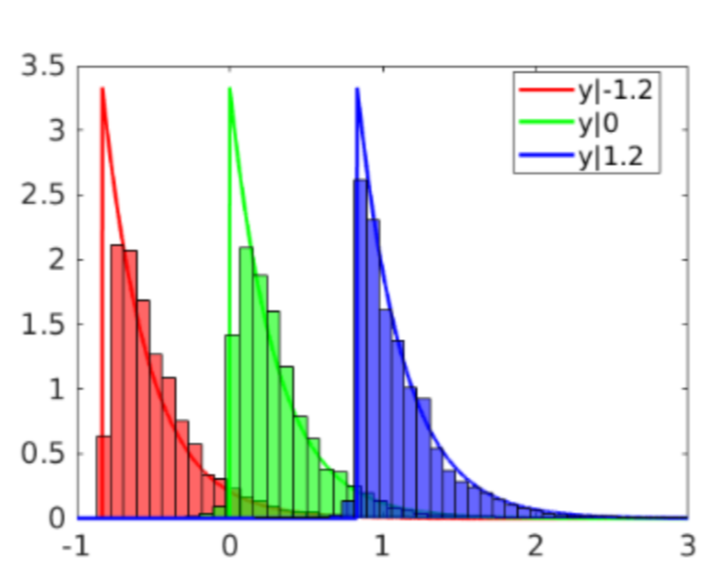}\includegraphics[height=2.13cm]{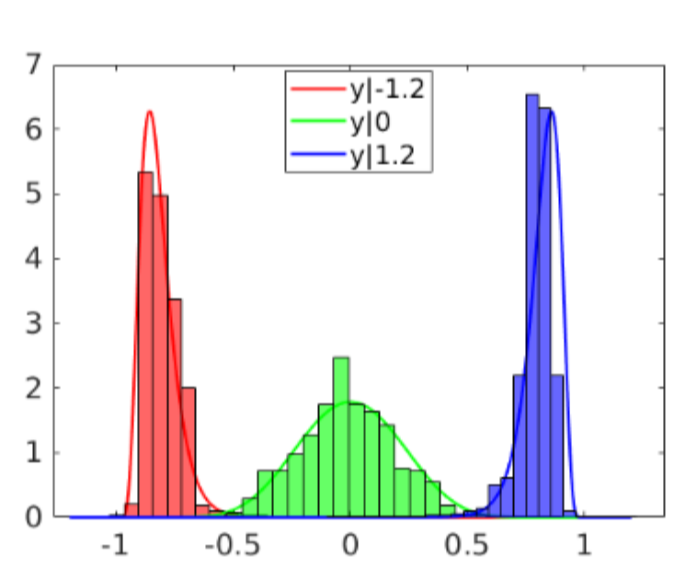}\includegraphics[height=2.13cm]{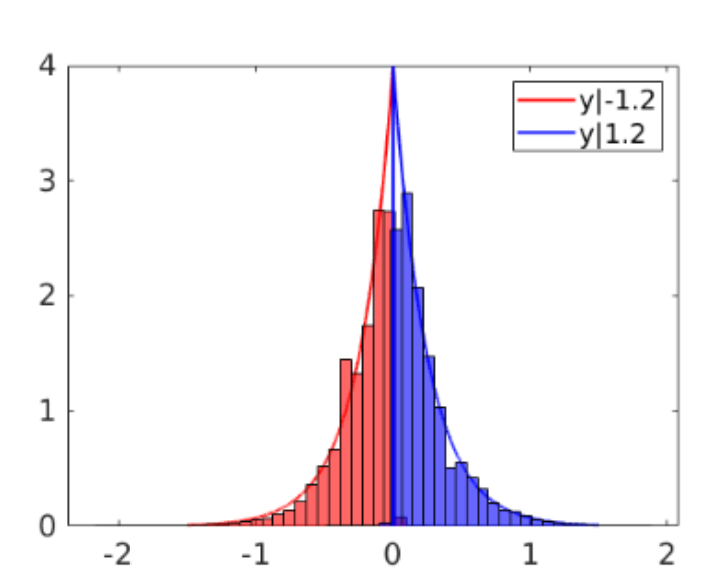}

\captionof{figure}{\label{fig:2dconditional}True posterior $p(x|\yobs)$ for $\yobs\in\{-1.2,0,1.2\}$ (solid lines) and approximations for the 2D examples.}
\end{minipage}

\vspace{0.25cm}

\begin{minipage}{\linewidth}
\tabcolsep=0.06cm
\small
    \begin{centering}
    \begin{tabular}{|c|c|c|c|c|c|c|c|}
    \hline 
    \multirow{1}{*}{} &  & MCMC & CDSB & CDSB-FB & CDSB-C & MGAN & IT\tabularnewline
    \hline 
    \multirow{2}{*}{Mean} & $x_{1}$ & .075 & .066 & .068 & \textbf{.072} & .048 & .034\tabularnewline
    \cline{2-8} \cline{3-8} \cline{4-8} \cline{5-8} \cline{6-8} \cline{7-8} \cline{8-8} 
     & $x_{2}$ & .875 & .897 & .897 & \textbf{.891} & .918 & .902\tabularnewline
    \hline 
    \multirow{2}{*}{Var} & $x_{1}$ & .190 & .184 & \textbf{.190} & .188 & .177 & .206\tabularnewline
    \cline{2-8} \cline{3-8} \cline{4-8} \cline{5-8} \cline{6-8} \cline{7-8} \cline{8-8} 
     & $x_{2}$ & .397 & .387 & .391 & \textbf{.393} & .419 & .457\tabularnewline
    \hline 
    \multirow{2}{*}{Skew} & $x_{1}$ & 1.94 & \textbf{1.90} & 2.01 & \textbf{1.90} & 1.83 & 1.63\tabularnewline
    \cline{2-8} \cline{3-8} \cline{4-8} \cline{5-8} \cline{6-8} \cline{7-8} \cline{8-8} 
     & $x_{2}$ & .681 & .591 & .628 & .596 & \textbf{.630} & .872\tabularnewline
    \hline 
    \multirow{2}{*}{Kurt} & $x_{1}$ & 8.54 & 7.85 & \textbf{8.54} & 8.00 & 7.64 & 7.57\tabularnewline
    \cline{2-8} \cline{3-8} \cline{4-8} \cline{5-8} \cline{6-8} \cline{7-8} \cline{8-8} 
     & $x_{2}$ & 3.44 & 3.33 & \textbf{3.51} & 3.27 & 3.19 & 3.88\tabularnewline
    \hline 
    \end{tabular}
    \par\end{centering}
\captionof{table}{\label{tab:bodresult} Estimated posterior moments for the BOD example. The closest estimates to MCMC are highlighted in bold. }
\end{minipage}
\end{figure}

\subsection{2D Synthetic Examples}

\begin{figure*}[h!]
\centering
\small
\begin{minipage}{\textwidth}
\tabcolsep=0.04cm
    \subfloat[]{
    \begin{tabular}{ccc}
    \toprule 
     & $N=5$ & $N=10$\tabularnewline
    \midrule
    \midrule 
    CSGM & 17.22/0.672 & 20.03/0.795\tabularnewline
    \midrule 
    CDSB & 18.55/0.746 & 20.69/0.792\tabularnewline
    \midrule 
    CSGM-C & 18.61/0.749 & 20.83/0.838\tabularnewline
    \midrule 
    CDSB-C & \textbf{19.67}/\textbf{0.753} & \textbf{20.95}/\textbf{0.840}\tabularnewline
    \bottomrule
    \end{tabular}
    }\hspace{-3pt}\subfloat[]{
    \begin{tabular}{cc}
    \toprule 
    $N=10$ & $N=20$\tabularnewline
    \midrule
    \midrule 
    14.77/0.599 & 16.31/0.706\tabularnewline
    \midrule 
    16.24/0.618 & 16.61/0.657\tabularnewline
    \midrule 
    16.38/\textbf{0.701} & 16.53/0.730\tabularnewline
    \midrule 
    \textbf{16.60}/0.700 & \textbf{16.65}/\textbf{0.747}\tabularnewline
    \bottomrule
    \end{tabular}
    }\hspace{-3pt}\subfloat[]{
    \begin{tabular}{cc}
    \toprule 
    $N=20$ & $N=50$\tabularnewline
    \midrule
    \midrule 
    19.52/0.471/92.02 & 20.52/0.567/48.68\tabularnewline
    \midrule 
    19.72/0.504/57.22 & 20.70/0.590/40.08\tabularnewline
    \midrule 
    20.44/0.566/44.44 & 20.84/0.592/22.89\tabularnewline
    \midrule 
    \textbf{21.11}/\textbf{0.614}/\textbf{28.41} & \textbf{21.46}/\textbf{0.646}/\textbf{13.71}\tabularnewline
    \bottomrule
    \end{tabular}
    }\hspace{-3pt}\subfloat[]{
    \begin{tabular}{cc}
    \toprule 
    $N=20$ & $N=50$\tabularnewline
    \midrule
    \midrule 
    24.22/0.844/17.62 & 25.29/0.878/7.18\tabularnewline
    \midrule 
    24.88/0.850/19.85 & 26.61/0.894/3.87\tabularnewline
    \midrule 
    \textbf{28.26}/0.914/3.63 & \textbf{28.14}/0.913/1.31\tabularnewline
    \midrule 
    28.19/\textbf{0.915}/\textbf{2.28} & 28.06/\textbf{0.914}/\textbf{1.14}\tabularnewline
    \bottomrule
    \end{tabular}
    }
    \captionof{table}{\label{tab:imagemetrics}Results for (a) MNIST 4x SR;
    (b) MNIST 14x14 inpainting; (c) CelebA 4x SR with Gaussian
    noise; (d) CelebA 32x32 inpainting. Reported results are denoted in the format PSNR↑/SSIM↑(/FID↓). }
\end{minipage}

\setcounter{subfigure}{0}

\begin{minipage}{0.46\textwidth}
\pdfpxdimen=\dimexpr 1in/300\relax
\begin{centering}
\subfloat[$\yobs$]{\includegraphics[trim={40px 40px 40px 40px}, clip, width=.5\linewidth]{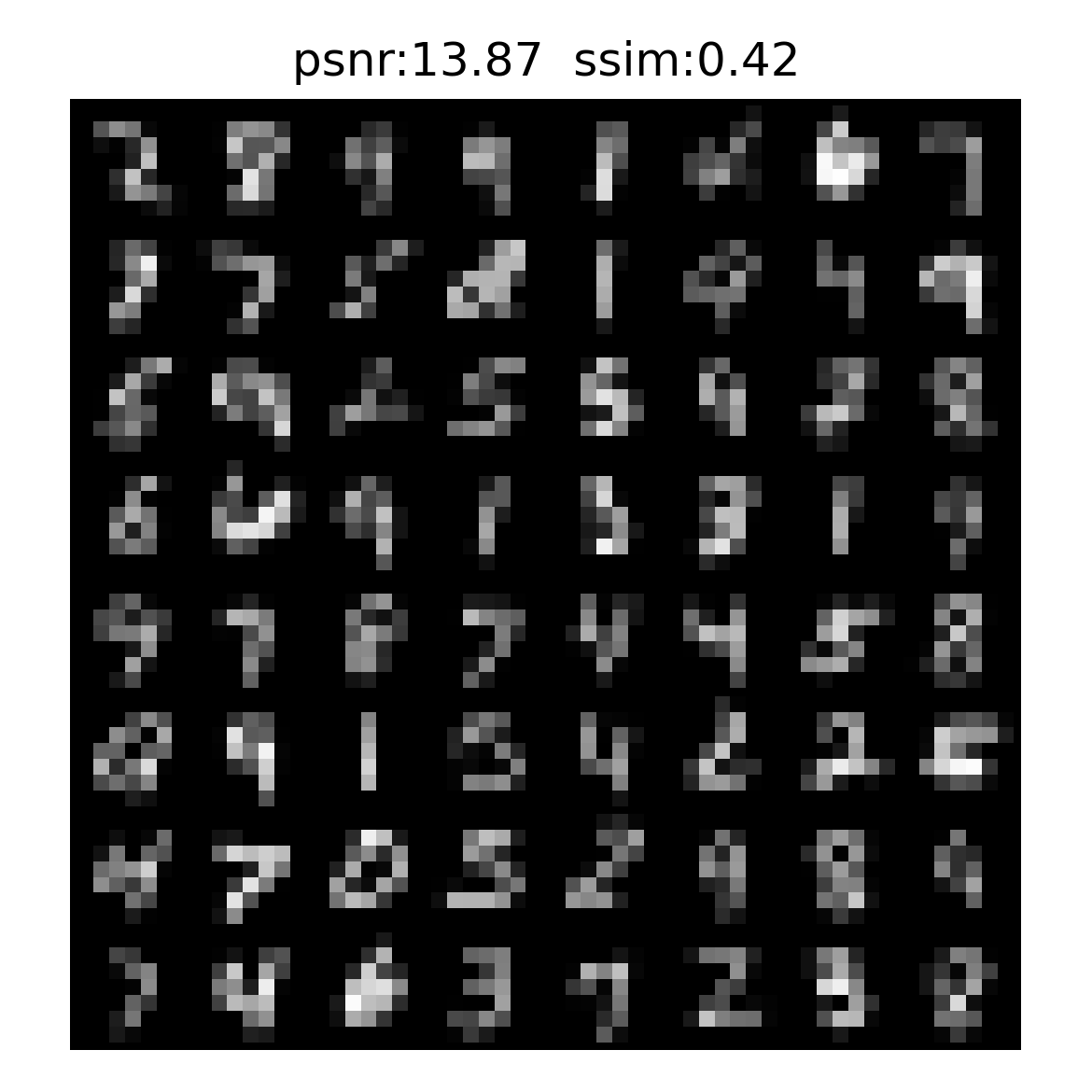}} 
\subfloat[Ground truth]{\includegraphics[trim={40px 40px 40px 40px}, clip, width=.5\linewidth]{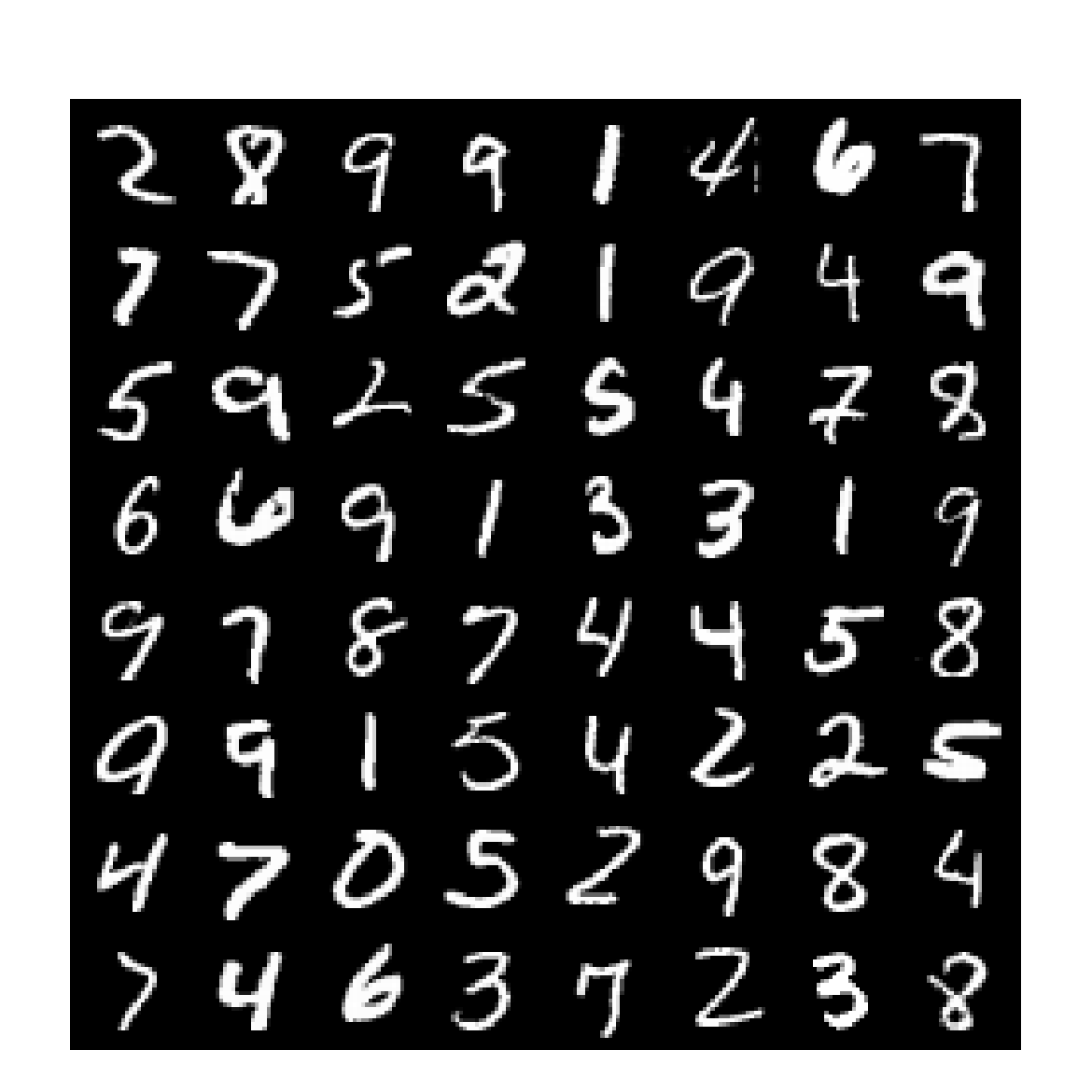}} \\
\vspace{-0.2cm}
\subfloat[CSGM]{\includegraphics[trim={40px 40px 40px 40px}, clip, width=.5\linewidth]{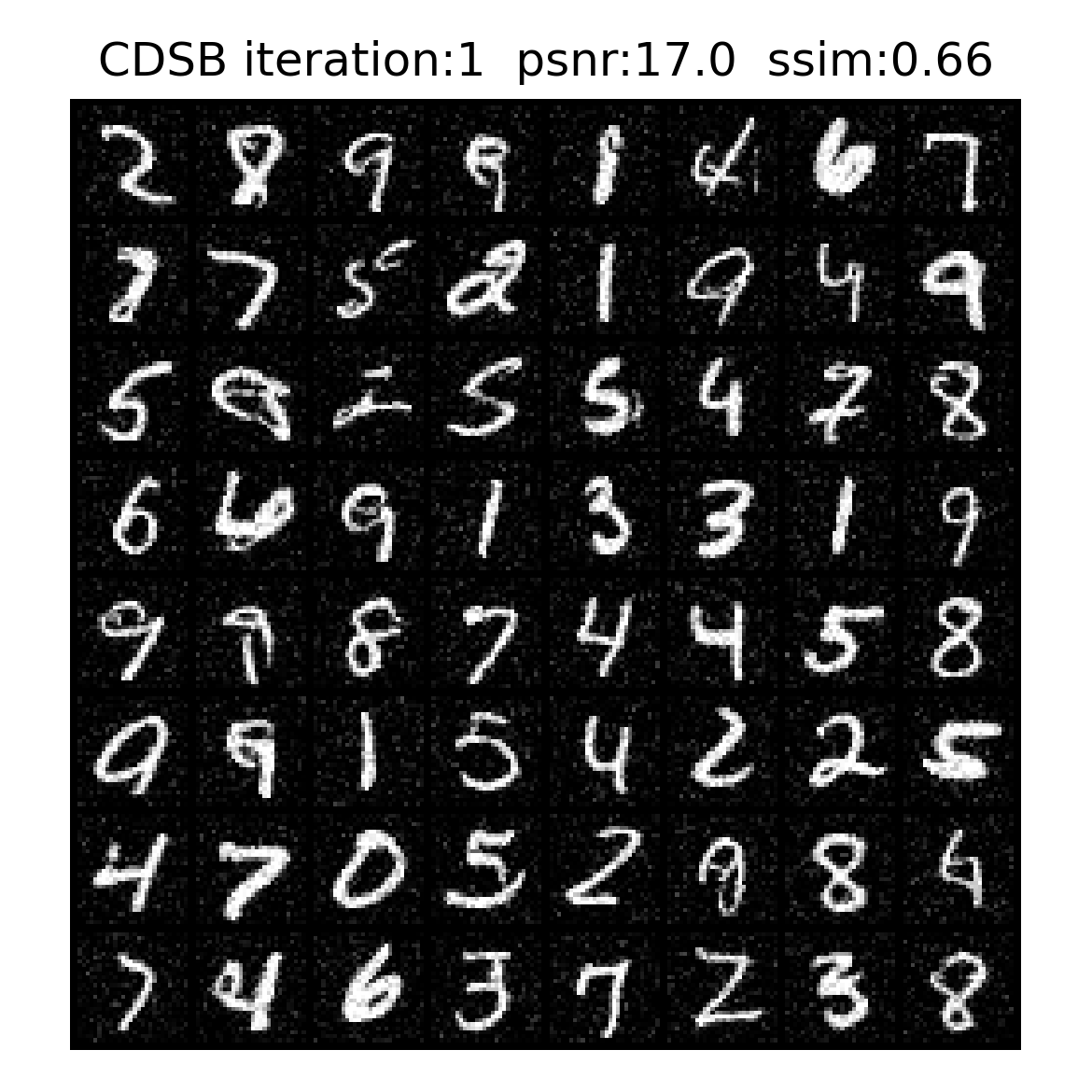}}
\subfloat[CDSB-C]{\includegraphics[trim={40px 40px 40px 40px}, clip, width=.5\linewidth]{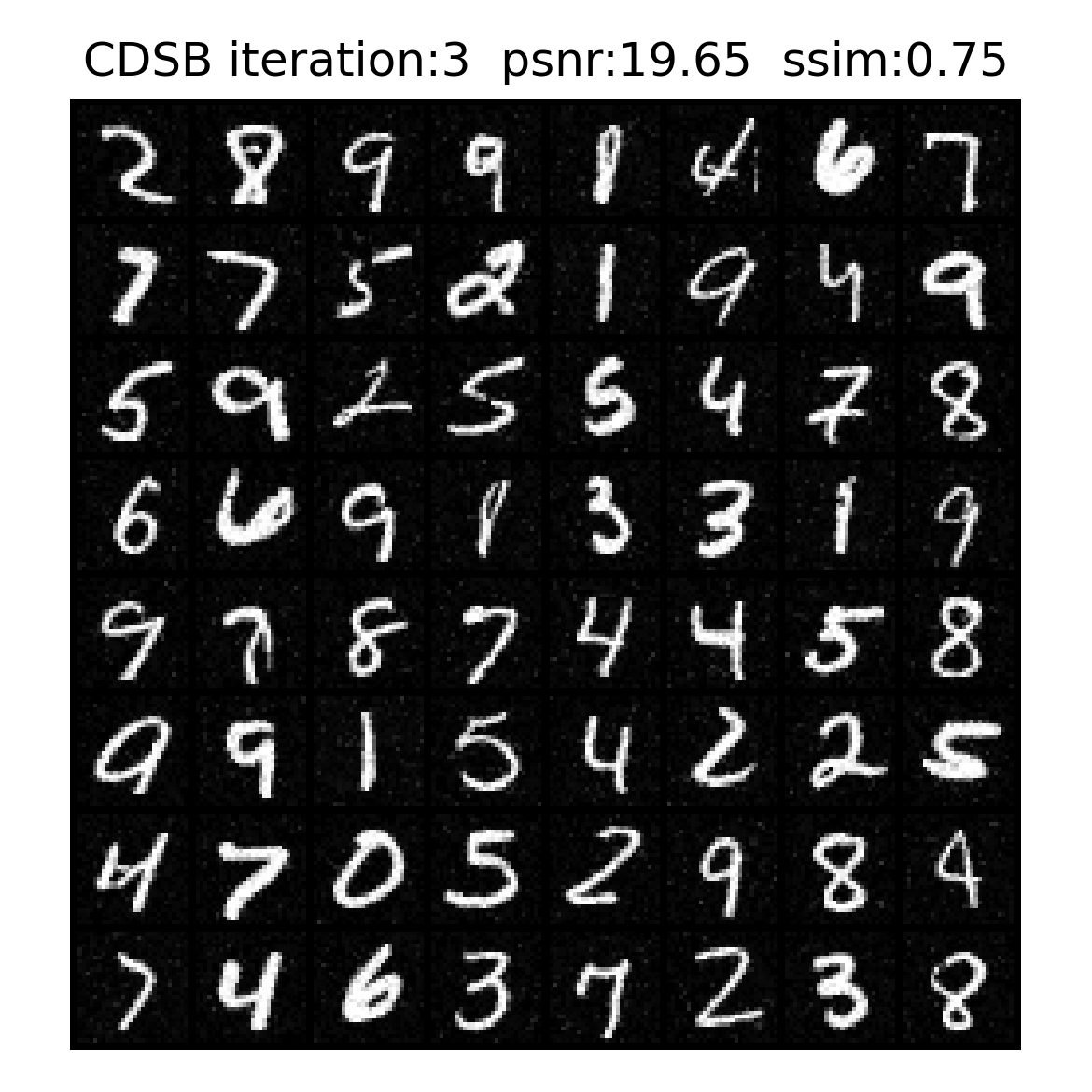}}
\par\end{centering}
\caption{\label{fig:imagecomparison-mnist}Uncurated samples for the MNIST 4x SR task with $N=5$.}
\end{minipage}  \qquad\quad
\begin{minipage}{0.46\textwidth}
\begin{centering}
\subfloat[$\yobs$]{\includegraphics[height=.5\linewidth]{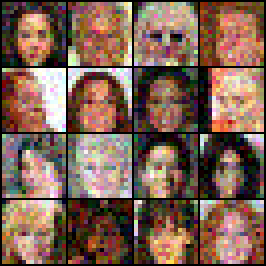}}~~
\subfloat[Ground truth]{\includegraphics[height=.5\linewidth]{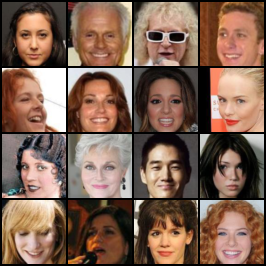}} \\
\vspace{-0.2cm}
\subfloat[CSGM]{\includegraphics[height=.5\linewidth]{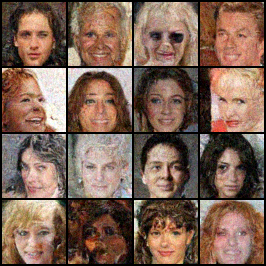}}~~
\subfloat[CDSB-C]{\includegraphics[height=.5\linewidth]{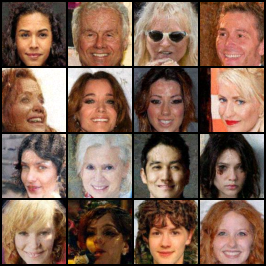}}
\par\end{centering}
\caption{\label{fig:imagecomparison-celeba}Uncurated samples for the CelebA 4x SR with Gaussian noise task with $N=20$. }
\pdfpxdimen=\dimexpr 1in/72\relax
\end{minipage} 
\end{figure*}

We first demonstrate the validity and accuracy of our method
using the two-dimensional examples of \cite{kovachki2021conditional}.
We consider three nonlinear, non-Gaussian examples for $\pjoin(x,y)$: 
define $\pobs(y)=\textup{Unif}(y;[-3,3])$ for all examples and $p(x|y)$ is defined through
\begin{align}
\text{Example 1: } & X=\tanh(Y)+Z, &  & Z\sim\Gamma(1,0.3),\\
\text{Example 2: } & X=\tanh(Y+Z), &  & Z\sim\mathcal{N}(0,0.05),\\
\text{Example 3: } & X=Z\tanh(Y), &  & Z\sim\Gamma(1,0.3).
\end{align}
We run CDSB on each of the examples with 50,000 training points and compare with
the Monotone GAN (MGAN) algorithm \citep{kovachki2021conditional}.  CDSB uses a
neural network model with 32k parameters (approximately 6x less parameters than
MGAN) with $N=50$ diffusion steps. Figure \ref{fig:2dconditional} shows the
resulting histogram of the learned $p(x|\yobs)$ and the true posterior for
$\yobs\in\{-1.2,0,1.2\}$.  As can be observed, the empirical density of CDSB
samples is sharper and aligns more closely with the ground truth density.  We
also observe that using more CDSB iterations corrects the sampling bias compared
to using only one CDSB iteration (which corresponds to CSGM). Using
forward-backward sampling (CDSB-FB) further improves the sample quality.

\subsection{Biochemical Oxygen Demand Model}
We now consider a Bayesian inference problem on biochemical oxygen demand
(BOD) from \citet{marzouk2016sampling}. Let
$X_{1},X_{2}\overset{\textup{i.i.d.}}\sim\mathcal{N}(0,1)$,
$A=0.8+0.4 \erf(X_{1}/\sqrt{2})$, $B=0.16+0.15 \erf(X_{2}/\sqrt{2})$ and $Y=\{Y(t)\}_{t=1}^5$ satisfy
$Y(t)=A(1-\exp(-Bt))+Z$ with
$Z\sim\mathcal{N}(0,{10}^{-3})$.  Table \ref{tab:bodresult} displays moment statistics of the estimated posterior $p(x|y)$ (standard deviations are reported in the supplementary), in comparison with
the ``ground truth'' statistics computed using $6\times{10}^{6}$ MCMC steps as
reported in \citet{marzouk2016sampling}. To match the evaluation in
\citet{kovachki2021conditional}, the reported statistics are computed using
30,000 samples and averaged across the last 10 CDSB iterations. The resulting
posterior displays high skewness and high kurtosis, but all CDSB-based methods achieve more accurate posterior estimation than MGAN
and the inverse transport (IT) method in \cite{marzouk2016sampling}.

\subsection{Image Experiments}
\subsubsection{Gaussian Reference Measure}
We now apply CDSB to a range of inverse problems on image
datasets. We consider the following tasks: (a) MNIST 4x SR (7x7 to
28x28), (b) MNIST center 14x14 inpainting, (c) CelebA 4x SR (16x16 to
64x64) with Gaussian noise of $\sigma_{y}=0.1$, (d) CelebA center 32x32 inpainting. 
For CSGM-C and CDSB-C, we consider the
following choices for conditional $\pref(x|y)$: for tasks (a) and (c), we use
the upsampled $y$ directly as described in \Cref{subsec:targetawareinitial}; for inpainting tasks
(b) and (d), we use a separate neural network with the same architecture as
$\mathbf{F},\mathbf{B}$ to output the initialization mean. In \Cref{tab:imagemetrics} we report PSNR and SSIM (the higher the better), as well as FID scores (the lower the better) for RGB images only. We display a
visual comparison between the methods in Figures \ref{fig:imagecomparison-mnist} and \ref{fig:imagecomparison-celeba}, and additional image samples in the supplementary. CDSB and CDSB-C
both provide significant improvement in terms of quantitative metrics as well as visual evaluations, and high-quality images can be generated quickly under few iterations $N$. 

\subsubsection{Pre-trained SR Model for Reference Measure}
\label{subsubsec:nongaussianref}
We further explore here the possibility of using a non-Gaussian $\pref(x|y)$ to further bridge the gap towards the true posterior $p(x|y)$. We utilize the super-resolution model SRFlow \citep{lugmayr2020srflow}, which produces
a probability distribution over possible SR images using a conditional normalizing flow. We use their pre-trained model checkpoints for the 8x SR task for CelebA (160x160). We then train a short CDSB model 
with SRFlow as $\pref(x|y)$,
in order to take advantage of the high sampling quality of diffusion models.
As can be seen from Figure \ref{fig:imagecomparison-celeba160}, 
with only $N=10$ steps the CDSB model is able to make 
meaningful improvements to the SRFlow samples, 
especially in the finer details such as facial features and hair texture. 
Quantitatively, CDSB-C produces significant improvement over the
FID score at the cost of a decrease in PSNR; see Table \ref{tab:imagemetrics-celeba160}. 
Note that this choice of non-Gaussian $\pref(x|y)$ is not compatible with CSGM. Interestingly CSGM-C still improves the PSNR compared to SRFlow, but produces worse FID scores than CDSB-C and blurry samples. 

\begin{figure}[t]
\centering
\small
\begin{minipage}{0.46\textwidth}
    \begin{tabular}{ccc}
    \toprule 
    $\pref(x|y)$ & CSGM-C & CDSB-C\tabularnewline
    \midrule
    \midrule 
    Gaussian & 22.21/0.521/87.02 & 23.86/0.628/31.65 \tabularnewline 
    \midrule 
    SRFlow $\tau=0.8$ & \textbf{24.97}/0.701/26.83 & 24.34/0.674/\textbf{15.00} \tabularnewline
    \midrule
    \midrule 
    SRFlow $\tau=0.8$ & \multicolumn{2}{c}{24.83/\textbf{0.702}/30.92}\tabularnewline
    \bottomrule
    \end{tabular}
\captionof{table}{\label{tab:imagemetrics-celeba160}Results for CelebA 8x SR. Reported results are denoted in the format PSNR↑/SSIM↑/FID↓. The final row reports our evaluated results of the SRFlow model. }
\end{minipage}

\begin{minipage}{0.46\textwidth}
\subfloat[$\yobs$]{\includegraphics[width=.5\linewidth]{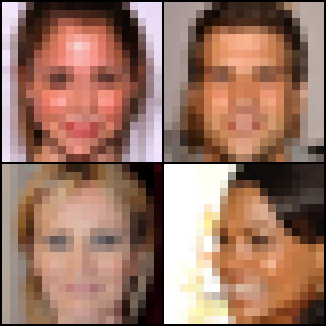}}
\subfloat[Ground truth]{\includegraphics[width=.5\linewidth]{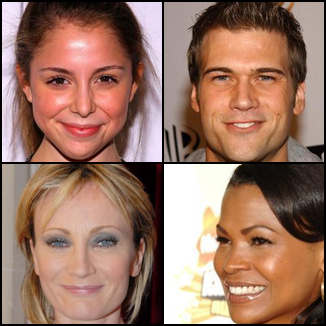}}
\vspace{-0.2cm}
\\
\subfloat[SRFlow]{\includegraphics[width=.5\linewidth]{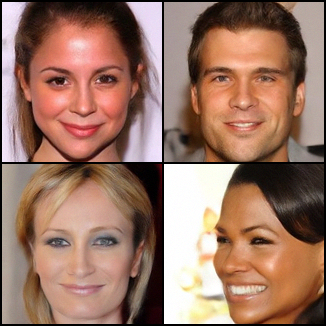}}
\subfloat[CDSB-C]{\includegraphics[width=.5\linewidth]{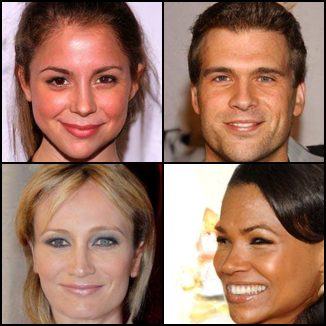}}
\caption{\label{fig:imagecomparison-celeba160} Paired samples for CelebA 8x SR. The SRFlow samples (c) are inputted as conditional initialization into CDSB-C (d), which produces fine modifications over $N=10$ steps (Best viewed when zoomed in).}
\end{minipage}

\end{figure}

\subsection{Filtering in State-Space Models}\label{sec:filtering}
Consider a state-space model defined by a bivariate Markov chain
$(X_t,Y_t)_{t\geq 1}$ of initial density $\mu(x_1)g(y_1|x_1)$ and transition
density $f(x_{t+1}|x_{t})g(y_{t+1}|x_{t+1})$ where $X_t$ is latent while $Y_t$ is observed. We are interested in estimating
sequentially in time the filtering distribution $p(x_t|\yobs_{1:t})$,
that is the posterior of $X_t$ given the observations $Y_{1:t}=\yobs_{1:t}$. We
show here how CDSB can be used at each time $t$ to obtain a sample approximation
of these filtering distributions. This CDSB-based algorithm only requires us being
able to sample from the transition density $f(x_{t+1}|x_{t})g(y_{t+1}|x_{t+1})$ and is
thus more generally applicable than standard techniques such as particle filters
\citep{doucet2009tutorial}.

Assume at time $t$, one has a collection of samples $\{X^i_{t}\}_{i=1}^M$
distributed (approximately) according to $p(x_{t}|\yobs_{1:t})$. We 
sample $X^i_{t+1} \sim f(x_{t+1}|X^i_{t})$ and $Y^i_{t+1} \sim g(y_{t+1}|X^i_{t+1})$. The
resulting samples $\{X^i_{t+1},Y^i_{t+1}\}_{i=1}^M$ are thus distributed according to
$\pjoin(x_{t+1},y_{t+1}):=p(x_{t+1},y_{t+1}|\yobs_{1:t})$. We can also easily obtain samples
from $\pjref(x_{t+1},y_{t+1}):=\pref(x_{t+1}|y_{t+1},\yobs_{1:t}) p(y_{t+1}|\yobs_{1:t})$ where
$\pref(x_{t+1}|y_{t+1},\yobs_{1:t})$ is an easy-to-sample distribution designed by the
user. Thus we can use CDSB to obtain a (stochastic) transport map between
$\pjoin(x_{t+1},y_{t+1})$ and $\pjref(x_{t+1},y_{t+1})$ and applying it to $Y_{t+1}=\yobs_{t+1}$, we can
obtain new samples from $p(x_{t+1}|\yobs_{1:t+1})$.  A similar strategy for filtering
based on deterministic transport maps was recently proposed by
\cite{spantini2019coupling}.

We apply CSGM and CDSB to the Lorenz-63 model \citep{law2015data}
following the procedure above for a time series of length 2000. 
We consider a short diffusion process with $N=20$
steps, as well as a long one with $N=100$. To accelerate the sequential
inference process, in this example we use analytic basis regression instead of
neural networks for all methods, and we only run 5 iterations of CDSB. 
As the EnKF is applicable
to this model, we can use the resulting approximate Gaussian filtering
distribution it outputs for $\pref(x_{t+1}|y_{t+1},\yobs_{1:t})$ in CSGM-C and CDSB-C.

Table \ref{tab:filtering} shows that for $N=20$ both CDSB and
CDSB-C successfully perform filtering and outperform the EnKF,
whereas both CSGM and CSGM-C fail to track the state accurately and diverge after a few
hundred times steps. CDSB-C achieves the lowest error consistently.
When using $N=100$, CSGM can achieve RMSE comparable with CDSB-C using $N=20$, but CDSB still provides advantages compared to CSGM. CSGM-C achieves comparable RMSE as CDSB-C with suitably long diffusion process in this case. 
For lower ensemble size, e.g. $M=200$, occasional large errors occur for some of the runs; see supplementary for details. We conjecture that this is due to overfitting. 

\begin{table}[h]
\begin{centering}
\small
    \begin{tabular}{|c|c|c|c|c|}
    \hline 
    $M$ & 500 & 1000 & 2000\tabularnewline
    \hline 
    \hline 
    EnKF & .354\textpm 0.006 & .355\textpm .005 & .354\textpm .003\tabularnewline
    \hline 
    \hline 
    CSGM(-C) (short)& \multicolumn{3}{c|}{Diverges}\tabularnewline
    \hline 
    CDSB (short)& .251\textpm .011 & .218\textpm .008 & .196\textpm .005\tabularnewline
    \hline 
    CDSB-C (short)& \textbf{.236\textpm .012} & \textbf{.207\textpm .014} & \textbf{.178\textpm .007}\tabularnewline
    \hline 
    \hline 
    CSGM (long) & .232\textpm .008 & .203\textpm .009 & .182\textpm .009\tabularnewline
    \hline 
    CDSB (long) & .220\textpm .012 & .195\textpm .007 & .166\textpm .004\tabularnewline
    \hline 
    CSGM-C (long) & \textbf{.210\textpm .009} & \textbf{.185\textpm.005} & .162\textpm.004\tabularnewline
    \hline 
    CDSB-C (long) & .218\textpm .014 & \textbf{.185\textpm .008} & \textbf{.160\textpm .003}\tabularnewline
    \hline 
    \end{tabular}
\par\end{centering}
\caption{\label{tab:filtering}RMSEs over 10 runs between each algorithm's filtering means
and the ground truth filtering means for $N=20$ (short) and $N=100$ (long). }
\end{table}

\section{Discussion}
We have proposed a SB formulation of conditional simulation and an algorithm,
CDSB, to approximate its solution. The first iteration of CDSB coincides with
CSGM while subsequent ones can be thought of as refining it. 
This theoretically grounded approach is complementary to the many other techniques
that have been recently proposed to accelerate SGMs  and could
be used in conjunction with them.  However, it also suffers from limitations. As
CDSB approximates numerically the diffusion processes output by IPF, the minimum
$N$ one can pick to obtain reliable approximations is related to the steepness
of the drift of these iterates which is practically unknown. Additionally CSGM
and CDSB are only using $\yobs$ when we want to sample from $p(x|\yobs)$ but not
at the training stage. Hence if $\yobs$ is not an observation ``typical'' under
$\pobs(y)$, the approximation of the posterior can be unreliable. In the ABC
context, the best available methods rely on procedures which sample synthetic
observations in the neighbourhood of $\yobs$. It would be interesting but
challenging to extend such ideas to CSGM and CDSB. Other interesting potential
extensions include developing an amortized version of CDSB for filtering that
would avoid having to solve a SB problem at each time step, and a conditional
version of the multimarginal SB problem.
\label{sec:conclusion}

\begin{acknowledgements}
We thank James Thornton for his helpful comments. We are also grateful to the authors of \citep{kovachki2021conditional} for sharing their code with us.
\end{acknowledgements}

\bibliography{refs}

\newpage
\appendix

\onecolumn

\section{Organization of the supplementary}

The supplementary is organized as follows. We recall the DSB algorithm for unconditional simulation from
\cite{debortoli2021neurips} in \Cref{sec:DSBalgorithm}. The proofs of our
propositions are given in \Cref{sec:proof-propositions}. In
\Cref{sec:lossfunctions}, we give details on the loss functions we use to train
CDSB. A continuous-time version of the conditional time-reversal and conditional
DSB is presented in \Cref{sec:cont-time-vers}. The forward-backward technique
used in our experiments is detailed in \Cref{sec:forw-backw-sampl}. Finally, we
provide experimental details and guidelines in \Cref{sec:experimental-details}.

\section{Diffusion \schro bridge}\label{sec:DSBalgorithm}
We recall here the DSB algorithm introduced by \cite{debortoli2021neurips} which is a numerical approximation of IPF\footnote{For discrete measures, IPF is also known as the Sinkhorn algorithm and can be implemented exactly \citep{peyre2019computational}.}. 
\begin{algorithm}[H]
    \caption{Diffusion \schro Bridge \citep{debortoli2021neurips}}
    \label{algo:ipf_score}
    \begin{algorithmic}[1] 
      \FOR{$n \in \{0, \dots,L\}$} \WHILE{not converged}
      \STATE Sample $\{X^j_{k}\}_{k,j=0}^{N,M}$, where  $X^j_0 \sim \pdata$, and \\
      $X^{j}_{k+1} = \mathbf{F}_{\phi^n}(k, X^{j}_{k})+\sqrt{2
        \gamma_{k+1}} Z^{j}_{k+1}$ 
        \STATE Compute $\hlb_n(\theta^n)$ approximating \eqref{eq:regressionbuncond}
        \STATE $\theta^{n} \leftarrow \textrm{Gradient Step}(\hlb_n(\theta^n))$ 
      \ENDWHILE \WHILE{not
        converged}
      \STATE Sample $\{X^j_{k}\}_{k,j=0}^{N,M}$, where $X^j_N \sim \pref$, and \\
      $X^j_{k-1}=\mathbf{B}_{\theta^n}(k, X^{j}_k)+\sqrt{2 \gamma_{k}}
      \tilde{Z}^{j}_{k}$ 
      \STATE Compute $\hlf_{n+1}(\phi^{n+1})$ approximating \eqref{eq:regressionfuncond}
      \STATE
      $\phi^{n+1} \leftarrow \textrm{Gradient Step}(\hlf_{n+1}(\phi^{n+1}))$
      \ENDWHILE \ENDFOR \STATE \textbf{Output: } $( \theta^{L},\phi^{L+1})$
    \end{algorithmic}
\end{algorithm}
\hfill

In this (unconditional) SB scenario, the transition kernels satisfy
$q^n_{k|k+1}(x|x')=
\mathcal{N}(x;\mathbf{B}_{\theta^n}(k+1,x'),2\gamma_{k+1} \Id)$ and
$p_{k+1|k}^n(x'|x) = \mathcal{N}(x';\mathbf{F}_{\phi^n}(k,x),
2\gamma_{k+1} \Id)$ where $\theta^n$ is obtained by minimizing
\begin{equation} \textstyle{\ell^{b}_n(\theta)=\mathbb{E}_{p^{n}}[\sum_{k}\normLigne{\mathbf{B}_\theta(k+1,X_{k+1})-G_{n,k}(X_{k},X_{k+1})}^2]}\label{eq:regressionbuncond}
\end{equation}
for $G_{n,k}(x,x')=x'+\mathbf{F}_{\phi^n}(k,x)-\mathbf{F}_{\phi^n}(k,x')$ and $\phi^{n+1}$ by minimizing 
\begin{equation} \textstyle{\ell^{f}_{n+1}(\phi)=\mathbb{E}_{q^{n}}[\sum_{k}\normLigne{\mathbf{F}_\phi(k,X_k)-H_{n,k}(X_k,X_{k+1})}^2]}\label{eq:regressionfuncond}
\end{equation}
for $H_{n,k}(x,x')=x + \mathbf{B}_{\theta^n}(k+1,x')-\mathbf{B}_{\theta^n}(k+1,x)$. See \cite{debortoli2021neurips} for a derivation of these loss functions.

\section{Proofs of Propositions}
\label{sec:proof-propositions}
\subsection{Proof of Proposition \ref{prop:SBreformulation}}

Let $\bar{\pi}$ such that $\textup{KL}(\bar{\pi}|\bar{p}) < +\infty$, which
exists since we have that $\textup{KL}(\bar{\pi}^\star|\bar{p}) < +\infty$, and
$\bar{\pi}_0= \pjoin,~ \bar{\pi}_N =\pjref$,  where we define the joint forward process
$\bar{p}(x_{0:N},y_{0:N}):=p_{y_0}(x_{0:N})\bpobs(y_{0:N})$. Recall that 
$p_{y_0}(x_{0:n}):=p(x_0|y_0) \prod_{k=0}^{N-1}p_{k+1|k}(x_{k+1}|x_k)$ is 
the forward process starting from the posterior $p(x_0|y_0)$, and
$\bpobs(y_{0:N}):=\pobs(y_0) \prod_{k=0}^{N-1}\delta_{y_k}(y_{k+1})$ is the extended $y$-process. 
Since $\KLLigne{\bar{\pi}}{\bar{p}}<+\infty$ we have using the transfer theorem
\cite[Theorem 2.4.1]{kullback1997information} that
$\KLLigne{\bar{\pi}_{\textup{obs}}}{\bpobs}<+\infty$, where
$\bar{\pi}_{\textup{obs}}(y_{0:N}): = \int_{(\rset^d)^N} \bar{\pi}(x_{0:N}, y_{0:N}) \rmd
x_{0:N}$. In addition, using the chain rule for the Kullback--Leibler
divergence, see \cite[Theorem 2.4]{leonard2014some}, we get that
\begin{equation}
  \textstyle{
  \KLLigne{\bar{\pi}_{\textup{obs}}}{\bpobs} =  \KLLigne{\bar{\pi}_{\textup{obs},0}}{\pobs} + \int_{\mcy} \KLLigne{\bar{\pi}_{\textup{obs}|0}}{\bar{p}_{\textup{obs}|0}} \pobs(y) \rmd y < +\infty},
\end{equation}
where $\bar{p}_{\textup{obs}|0} = \prod_{k=0}^{N-1}\delta_{y_k}(y_{k+1})$ and
therefore
$\bar{\pi}_{\textup{obs}|0} =\bar{p}_{\textup{obs}|0}$. Since we also have that
$\bar{\pi}_{\textup{obs},0} = \pobs$ we get that
$\bar{\pi}_{\textup{obs}} = \bpobs$. Hence, letting $\pi^{c}$ be the
kernel such that $\bar{\pi} = \pi^c \otimes \bpobs$ we have using
\cite[Theorem 2.4]{leonard2014some} that
\begin{equation}
  \label{eq:KL_eq}
  \textstyle{\KLLigne{\bar{\pi}}{\bar{p}} = \int_{\mcy} \KL{\pi^c_y}{p_y} \pobs(y) \rmd y  . }
\end{equation}
In addition, we have $\bar{\pi}_0 = \pi_0^c \otimes \pobs =
\pjoin$. Similarly, we have
$\bar{\pi}_N = \pi_N^c \otimes \pobs = \pjref$. Hence,
$\pi_{y,0}^c = p(\cdot|y)$ and $\pi_{y,N}^c = \pref$, $\pobs$-almost surely.  Let
$\bar{\pi}^\star = \pi^{\star,c} \otimes \bpobs$ be the minimizer of
\eqref{eq:conditionalSBextended} and $\hat{\pi}^{c}$ be the minimizer of
\eqref{eq:SBuncondreformulated}. Then, we have that
$\bar{\pi} = \hat{\pi}^{c} \otimes \bpobs$ satisfies
$\KLLigne{\bar{\pi}^\star}{\bar{p}} \leq \KLLigne{\bar{\pi}}{\bar{p}}$. Using
\eqref{eq:KL_eq}, we have that
$\expeLigne{\KL{\pi^{\star,c}_Y}{p_Y}} \leq
\expeLigne{\KL{\hat{\pi}^c_Y}{p_Y}}$. But we have that
$\expeLigne{\KL{\hat{\pi}^{c}_Y}{p_Y}} \leq
\expeLigne{\KL{\pi^{\star,c}_Y}{p_Y}}$ since $\hat{\pi}^{c}$ is the minimizer of
\eqref{eq:SBuncondreformulated}. Using the uniqueness of the minimizer of
\eqref{eq:SBuncondreformulated} we have that $\pi^{\star,c} = \hat{\pi}^{c}$,
which concludes the proof.

\subsection{Proof of Proposition \ref{prop:IPFrecursion}}

Let $n \in \nset$ and $\bar{q}$ be such that
$\KLLigne{\bar{q}}{\bar{p}^n}<+\infty$ and $\bar{q}_N = \pjref$ (note that the
existence of such a distribution is ensured since
$\KLLigne{\pjoin \otimes \pjref}{\bar{p}^n_{0,N}} < +\infty$).  Using the chain
rule for the Kullback--Leibler divergence, see \cite[Theorem
2]{leonard2014some}, we have
\begin{equation}
  \label{eq:KL_eq_y_traj}
  \textstyle{\KLLigne{\bar{q}}{\bar{p}^n} = \KL{\bar{q}_{\textup{obs}}}{\bpobs} + \int_{\mcy^{N+1}} \KLLigne{\bar{q}_{|\textup{obs}}}{\bar{p}_{|\textup{obs}}^n} \rmd \bar{q}_{\textup{obs}}(y_{0:N})  , }
\end{equation}
where
$\bar{q}_{\textup{obs}} = \int_{\mcx^{N+1}} \bar{q}(x_{0:N}, y_{0:N}) \rmd
x_{0:N} $ and $\bar{q}_{|\textup{obs}}$ and $\bar{p}_{|\textup{obs}}^n$ are the
conditional distribution of $\bar{q}$, respectively $\bar{p}^n$ w.r.t. to
$y_{0:N}$. Since
$ \KL{\bar{q}_{\textup{obs}}}{\bar{p}_{\textup{obs}}}< +\infty$, we can use
\cite[Theorem 2.4]{leonard2014some} and we have
\begin{equation}
  \textstyle{\KL{\bar{q}_{\textup{obs}}}{\bar{p}_{\textup{obs}}} = \KL{\bar{q}_{\textup{obs},N}}{\bar{p}_{\textup{obs},N}} + \int_{\mcy} \KL{\bar{q}_{\textup{obs}|N}}{\bar{p}_{\textup{obs}|N}} \rmd \bar{q}_{\textup{obs},N}(y_N),}
\end{equation}
with
$\bar{p}_{\textup{obs}|N}(y_{0:N-1}|y_N)
=\prod_{k=0}^{N-1}\delta_{y_{k+1}}(y_{k})$. Therefore, since
$\KL{\bar{q}_{\textup{obs}}}{\bar{p}_{\textup{obs}}}<+\infty$, we get that
$\bar{q}_{\textup{obs}|N}(y_{0:N-1}|y_N) =
\prod_{k=0}^{N-1}\delta_{y_{k+1}}(y_{k})$. Since
$\bar{q}_{\textup{obs}, N} = \pobs$, we get that
$\bar{q}(x_{0:N}, y_{0:N}) = \bpobs(y_{0:N})\bar{q}(x_{0:N}|y_{0:N}) =
\bpobs(y_{0:N})\bar{q}(x_{0:N}|y_{N})$, where we have used that $y_N = y_k$
for $k \in \{0, \dots, N\}$, $\bpobs(y_{0:N})$ almost surely. Combining this
result and \eqref{eq:KL_eq_y_traj} we get that
\begin{align}
  \KLLigne{\bar{q}}{\bar{p}^n} &=  \textstyle{\int_{\mcy^{N+1}} \KLLigne{\bar{q}_{|\textup{obs}}}{\bar{p}_{|\textup{obs}}^n} \rmd \pobs(y_{0:N})  }
  =  \textstyle{\int_{\mcy} \KLLigne{\bar{q}(\cdot|y_N)}{\bar{p}^n(\cdot|y_N)} \rmd \pobs(y_{N})  , }
  \end{align}
Using \cite[Theorem 2]{leonard2014some}, we have that for any $y_N \in \mcy$
\begin{equation}
  \textstyle{
  \KLLigne{\bar{q}(\cdot|y_N)}{\bar{p}^n(\cdot|y_N} = \KLLigne{\pref}{\bar{p}_N^n(\cdot|y_N)} + \int_{\mcy} \KLLigne{\bar{q}(\cdot|y_N, x_N)}{\bar{p}^n(\cdot|y_N, x_N)} \pref(x_N) \rmd x_N .}
\end{equation}
For the IPF solution $\bar{q}^n$, we get that $\bar{q}^n(\cdot|y_N, x_N) = \bar{p}^n(\cdot|y_N, x_N)$. Therefore for any $x_{0:N} \in \mcx^{N+1}$ and $y_N \in \mcy$,
\begin{equation}
 \textstyle{\bar{q}^n(x_{0:N}|y_{N})= \pref(x_N)\prod_{k=0}^{N-1}
  \bar{p}^n_{k|k+1}(x_k|x_{k+1},y_{N})} .
\end{equation}
The proof is similar for any $x_{0:N} \in \mcx^{N+1}$ and $y_0 \in \mcy$, we have
\begin{equation}
  \textstyle{\bar{p}^{n+1}(x_{0:N}|y_0)= p(x_0|y_0) \prod_{k=0}^{N-1} \bar{q}^{n}_{k+1|k}(x_{k+1}|x_{k},y_{0}).}
\end{equation}

\subsection{Proof of Proposition \ref{prop:fasterconverence}}

Using \cite[Corollary 1]{leger2020gradient}, we get that for any $n \in \nset$
with $n \geq 1$
\begin{equation}
  \label{eq:leger_res}
  \KLLigne{\bar{\pi}^n_0}{\pjoin} + \KLLigne{\bar{\pi}^n_N}{\pjref} \leq \frac{2}{n}\KLLigne{\bar{\pi}^\star}{\bar{p}}.
\end{equation}
Similarly to \Cref{prop:IPFrecursion}, we have that for any $n \in \nset$, there
exists a Markov kernel $\pi^{c,n}$ such that
$\bar{\pi}^n = \bpobs \otimes \pi^{c,n}$. Recall that there exists a Markov
kernel $\pi^{c,\star}$ such that $\bar{\pi}^\star = \bpobs \otimes \pi^{c,\star}$
and that $\bar{p} = \bpobs \otimes p_y$. Hence, using \cite[Theorem
2.4]{leonard2014some}, we get that for any $n \in \nset$,
\begin{equation}
  \label{eq:leonard_1}
  \KLLigne{\bar{\pi}^n_0}{\pjoin}  = \expeLigne{\KLLigne{\pi^{c,n}_{Y,0}}{p(\cdot|Y)}}, \qquad \KLLigne{\bar{\pi}^n_N}{\pjref}  = \expeLigne{\KLLigne{\pi^{c,n}_{Y,N}}{\pref}}.
\end{equation}
Similarly, we have that
\begin{equation}
  \label{eq:leonard_2}
  \KLLigne{\bar{\pi}^\star}{\bar{p}} = \expeLigne{\KLLigne{\pi^{c,\star}_Y}{p_Y}} .
\end{equation}
We conclude the proof upon combining \eqref{eq:leger_res}, \eqref{eq:leonard_1}
and \eqref{eq:leonard_2}.

\section{Details on the loss functions}\label{sec:lossfunctions}

In this section, we simplify notation and write $Y$ for all the random variables
$Y_0,Y_1,...,Y_N$ as they are all equal almost surely under $\bar{p}^{n}$ and
$\bar{q}^{n}$, similarly to \Cref{sec:cond-simul-sb}. 
In \Cref{sec:cond-simul-sb}, the transitions satisfy
$\bar{q}^n_{k|k+1}(x|x',y)=
\mathcal{N}(x;\mathbf{B}^{y}_{\theta^n}(k+1,x'),2\gamma_{k+1} \Id)$ and
$\bar{p}_{k+1|k}^n(x'|x,y) = \mathcal{N}(x';\mathbf{F}^{y}_{\phi^n}(k,x),
2\gamma_{k+1} \Id)$ where $\theta^n$ is obtained by minimizing
\begin{equation} \textstyle{\ell^{b}_n(\theta)=\mathbb{E}_{\bar{p}^{n}}[\sum_{k}\normLigne{\mathbf{B}_\theta^Y(k+1,X_{k+1})-G_{n,k}^Y(X_{k},X_{k+1})}^2]}\label{eq:regressionbcond_app}
\end{equation}
for $G_{n,k}^y(x,x')=x'+\mathbf{F}^{y}_{\phi^n}(k,x)-\mathbf{F}^{y}_{\phi^n}(k,x')$ 
and $\phi^{n+1}$ by minimizing 
\begin{equation} \textstyle{\ell^{f}_{n+1}(\phi)=\mathbb{E}_{\bar{q}^{n}}[\sum_{k}\normLigne{\mathbf{F}_\phi^Y(k,X_k)-H_{n,k}^Y(X_k,X_{k+1})}^2]}\label{eq:regressionfcond_app}
\end{equation}
for $H_{n,k}^y(x,x')=x +
\mathbf{B}^{y}_{\theta^n}(k+1,x')-\mathbf{B}^{y}_{\theta^n}(k+1,x)$. We
justify these formulas by proving the following result which is a
straightforward extension of \cite{debortoli2021neurips}. We recall that for any
$n \in \nset$, $k \in \{0, \dots, N\}$, $x_k,x_{k+1} \in \rset^d$ and
$y \in \mcy$,
$b^{n,y}_{k+1}(x_{k+1}) = -f^{n,y}_{k}(x_{k+1})+2 \nabla \log
\bar{p}^{n}_{k+1}(x_{k+1}|y)$ and
$f^{n+1,y}_{k}(x_{k}) = -b^{n,y}_{k+1}(x_{k})+2 \nabla \log
\bar{q}^{n}_{k}(x_k|y)$.\footnote{We should have conditioned w.r.t. $y_N$ and
  $y_0$ but since $y_0 = y_1 = \dots = y_N$ under $\pobs$ we simply conditioned
  by $y$ which can be any of these values.}

\begin{proposition}\label{prop:generalizedscorematching} 
  Assume that for any $n \in \nset$ and $k \in \{0, \dots, N-1\}$,
  $\bar{q}_k(\cdot|y)$ and $\bar{p}_k(\cdot|y)$ are bounded and
  \begin{equation}
    \bar{q}_{k|k+1}^n(x_k|x_{k+1},y) = \mathcal{N}(x_k;B_{k+1}^{n,y}(x_{k+1}), 2\gamma_{k+1}
  \Id)  ,\ \bar{p}_{k+1|k}^n(x_{k+1}|x_{k},y) = \mathcal{N}(x_{k+1};F_{k}^{n,y}(x_{k}), 2\gamma_{k+1}
  \Id) ,
  \end{equation}
 with $B^{n,y}_{k+1}(x) = x +\gamma_{k+1}b^{n,y}_{k+1}(x)$,
  $F^{n,y}_{k}(x)= x +\gamma_{k+1}f_k^{n,y}(x)$ for any $x \in \rset^d$. Then we
  have for any $n \in \nset$ and $k\in \{0, \dots, N-1\}$
\begin{align}
&\textstyle{B^{n}_{k+1}=\argmin_{\mathrm{B}\in \rmL^2(\rset^d \times \mcy, \rset^d)} \expeMarkovLigne{ \bar{p}^{n}}{\normLigne{\mathrm{B}(X_{k+1},Y)-G_{n,k}^Y(X_k,X_{k+1})}^2}},\label{eq:regressionb}\\
  &\textstyle{F^{n+1}_{k}=\argmin_{\mathrm{F}\in \rmL^2(\rset^d \times \mcy, \rset^d)} \expeMarkovLigne{\bar{q}^{n}}{\normLigne{\mathrm{F}(X_k,Y)-H_{n,k}^Y(X_k,X_{k+1})}^2 }},\label{eq:regressionf}\\
  &G_{n,k}^y(x,x') = x' + F^{n,y}_k(x)-F^{n,y}_{k}(x') , \qquad H_{n,k}^y(x,x') = x + B^{n,y}_{k+1}(x')-B^{n,y}_{k+1}(x) .
 \end{align} 
\end{proposition}

\begin{proof}
  We only prove \eqref{eq:regressionb} since the proof \eqref{eq:regressionf} is
similar. Let $n \in \nset$ and $k \in \{0, \dots, N-1\}$. For any
$x_{k+1} \in \rset^d$ we have
\begin{equation}
  \textstyle{
    \bar{p}^n_{k+1}(x_{k+1}|y) = (4 \uppi \gamma_{k+1})^{-d/2} \int_{\rset^d} \bar{p}^n(x_k|y) \exp[-\normLigne{F_k^{n,y}(x_k) - x_{k+1}}^2/(4\gamma_{k+1})] \rmd x_k  ,
    }
\end{equation}
with $F_k^{n,y}(x_k) = x_k + \gamma_{k+1} f_k^{n,y}(x_k)$. Since $\bar{p}^n_k>0$ is bounded
using the dominated convergence theorem we have for any $x_{k+1} \in \rset^d$
\begin{equation}
  \textstyle{\nabla_{x_{k+1}} \log \bar{p}^n_{k+1} (x_{k+1}|y) = \int_{\rset^d} (F_k^{n,y}(x_k) - x_{k+1})/(2 \gamma_{k+1})~\bar{p}_{k|k+1}(x_k | x_{k+1},y) \rmd x_{k}  . }
\end{equation}
Therefore we get that for any $x_{k+1} \in \rset^d$
\begin{equation}
  \textstyle{b_{k+1}^{n,y}(x_{k+1}) = \int_{\rset^d} (F_k^{n,y}(x_k) - F_{k}^{n,y}(x_{k+1}))/\gamma_{k+1}~  \bar{p}_{k|k+1}(x_k | x_{k+1},y) \rmd x_{k}  . }
\end{equation}
This is equivalent to
\begin{equation}
  \textstyle{B_{k+1}^{n,y}(x_{k+1})  = \CPELigne{X_{k+1} + F_k^{n,Y}(X_k) - F_{k}^{n,Y}(X_{k+1})}{X_{k+1} = x_{k+1}, Y=y}}  ,
\end{equation}
Hence, we get that
\begin{equation}
\textstyle{B^n_{k+1}=\argmin_{\mathrm{B}\in \rmL^2(\rset^d \times \mcy, \rset^d)} \expeMarkovLigne{ \bar{p}^{n}}{\normLigne{\mathrm{B}(X_{k+1},Y)-(X_{k+1} + F^{n,Y}_k(X_{k})-F^{n,Y}_{k}(X_{k+1}))}^2}}  ,
\end{equation}
which concludes the proof.
\end{proof}

\section{Continuous-time versions of CSGM and CDSB}
\label{sec:cont-time-vers}

In the following section, we consider the continuous-time version of CSGM and CDSB. The continuous-time dynamics we recover can be seen as the
extensions of the continuous-time dynamics obtained in the unconditional
setting, see \cite{song2020score,debortoli2021neurips}.

\subsection{Notation}

We start by introducing a few notations.  The space of continuous functions from
$\ccint{0,T}$ to $\rset^d \times \mcy$ is denoted
$\contspace = \rmc(\ccint{0,T}, \rset^d \times \mcy)$ and we denote
$\Pens(\contspace)$ the set of probability measures defined on $\contspace$.  A
probability measure $\Pbb \in \Pens(\contspace)$ is \emph{associated with a
  diffusion} if it is a solution to a martingale problem, i.e.
$\Pbb \in \Pens(\contspace)$ is associated with
$\rmd \bfX_t = b(t, \bfX_t) \rmd t + \sqrt{2} \rmd \bfB_t$  if for any
$\varphi \in \rmc_c^2(\rset^d, \rset)$, $(\bfZ_t^\varphi)_{t \in \ccint{0,T}}$ is a
$\Pbb$-local martingale, where for any $t \in \ccint{0,T}$
\begin{equation}
\label{eq:martingale_pbm}
\textstyle{\bfZ_t^\varphi = \varphi(\bfX_t) - \int_0^t \generator_s(\varphi)(\bfX_s) \rmd s }, \qquad  \generator_t(\varphi)(x) = \langle b(t, x) , \nabla \varphi(x) \rangle +  \Delta \varphi(x).
\end{equation}
Here $\rmc_c^2(\rset^d, \rset)$ denotes the space of twice differentiable functions from $\rset^d$ to $\rset$ with compact support.
Doing so, $\Pbb$ is uniquely defined up to the initial distribution $\Pbb_0$. Finally, for any
$\Pbb \in \Pens(\contspace)$, we introduce $\Pbb^R$ the time reversal of $\Pbb$,
\ie \ for any $\msa \in \mcb{\contspace}$ we have $\Pbb^R(\msa) = \Pbb(\msa^R)$
where $\msa^R = \ensembleLigne{t \mapsto\omega(T-t)}{\omega \in \msa}$.

\subsection{Continuous-time CSGM}

Recall that in the unconditional setting, we consider a forward noising dynamics
$(\bfX_t)_{t \in \ccint{0,T}}$ initialized with $\bfX_0 \sim \pdata$ and
satisfying the following Stochastic Differential Equation (SDE)
$\rmd \bfX_t = -\bfX_t \rmd t + \sqrt{2} \rmd \bfB_t$, i.e. an
Ornstein--Uhlenbeck process. In this case, under entropy condition on
$(\bfX_t)_{t \in \ccint{0,T}}$ (see \cite{cattiaux2021time} for instance) we
have that the time-reversal process
$(\tbfX_t)_{t \in \ccint{0,T}} = (\bfX_{T-t})_{t \in \ccint{0,T}}$ also satisfy
an SDE given by
$\rmd \tbfX_t = \{\tbfX_t + 2 \nabla \log p_{T-t}(\tbfX_t)\} \rmd t + \sqrt{2} \rmd
\bfB_t$, where $p_t$ is the density of $\bfX_t$ w.r.t. the Lebesgue
measure, and $(\tbfX_t)_{t \in \ccint{0,T}}$ is initialized with
$\tbfX_0 \sim \mathcal{L}(\bfX_T)$, the law of $\bfX_T$ of density $q_T$. Using the geometric ergodicity of the
Ornstein--Uhlenbeck process, $\mathcal{L}(\bfX_T)$ is close (w.r.t.
to the Kullback--Leibler divergence for instance) to $\pref= \mathcal{N}(0,\Id)$. Hence,
we obtain that considering $(\bfZ_t)_{t \in \ccint{0,T}}$ such that
$\bfZ_0 \sim \mathcal{N}(0,\Id)$ and
$\rmd \bfZ_t = \{\bfZ_t + 2 \nabla \log p_{T-t}(\bfZ_t)\} \rmd t + \sqrt{2} \rmd
\bfB_t$, $\bfZ_T$ is approximately distributed according to $\pdata$. The
Euler--Maruyama discretization of $(\bfZ_t)_{t \in \ccint{0,T}}$ is the SGM used
in existing work.

In the conditional setting, we consider the following dynamics
$\rmd \bfX_t = -\bfX_t \rmd t + \sqrt{2} \rmd \bfB_t$ and $\rmd \bfY_t = 0$,
where $(\bfX_0, \bfY_0) \sim \pjoin$. Note that we have $\bfY_t = \bfY_0$ for
all $t \in \ccint{0,T}$. Using the ergodicity of the Ornstein--Uhlenbeck process,
we get that $\mathcal{L}(\bfX_T, \bfY_t)$ is close (w.r.t. to the
Kullback--Leibler divergence for instance) to $\pjref$. Let
$(\tbfX_t, \tbfY_t)_{t \in \ccint{0,T}} = (\bfX_{T-t}, \bfY_{T-t})_{t \in
  \ccint{0,T}}$. We have that
$\rmd \tbfX_t = \{\tbfX_t + 2 \nabla \log p_{T-t}(\tbfX_t|\tbfY_t)\} \rmd t +
\sqrt{2} \rmd \bfB_t$ and $\rmd \tbfY_t = 0$ with $\tbfX_0,\tbfY_0 \sim \mathcal{L}(\bfX_T,\bfY_T)$. Hence, we obtain that considering
$(\bfZ_t)_{t \in \ccint{0,T}}$ such that $(\bfZ_0, \bfY_0) \sim \pjref$ and
$\rmd \bfZ_t = \{\bfZ_t + 2 \nabla \log p_{T-t}(\bfZ_t|\bfY_0)\} \rmd t + \sqrt{2}
\rmd \bfB_t$, $\bfZ_T$ is approximately distributed according to $\pdata$. The
Euler--Maruyama discretization of $(\bfZ_t,\bfY_t)_{t \in \ccint{0,T}}$ is the
conditional SGM.

\subsection{Connection with normalizing flows and estimation of the evidence}\label{sec:NFevidence}

It has been shown that SGMs can be used for log-likelihood computation. Here, we
further show that they can be used to estimate the evidence
$\log p(\yobs)$ when $g(\yobs|x)$ can be computed pointwise. This is the case for many models considered in the diffusion literature, see for
instance \cite{kadkhodaie2021stochastic,kawar2021snips,kawar2022denoising}. Indeed, we have that for any $x \in \rset^d$,
$\log p(\yobs) = \log g(\yobs|x) + \log p(x) - \log p(x|\yobs)$. The term $\log p(x)$ can be estimated using an unconditional SGM whereas the term
$\log p(x|\yobs)$ can be estimated using a CSGM. Note that both conditional and unconditional SGM can be trained simultaneously adding a
``sink'' state to $\mcy$, i.e. considering $\mcy \cup \{\emptyset\}$, see
\cite{ho2021classifier} for instance.

We briefly explain how one can compute $\log p(x|\yobs)$ and refer to
\cite{song2020score} for a similar discussion in the unconditional
setting. Recall that the forward noising process is given by
$\rmd \bfX_t = -\bfX_t \rmd t + \sqrt{2} \rmd \bfB_t$ and $\rmd \bfY_t = 0$,
where $(\bfX_0, \bfY_0) \sim \pjoin$. We introduce another process
$(\hbfX_t, \hbfY_t)_{t \in \ccint{0,T}}$ with deterministic dynamics which has the same marginal distributions, i.e. $\mathcal{L}(\bfX_T,\bfY_T)=\mathcal{L}(\hbfX_T,\hbfY_T)$. This process is defined by $\rmd \hbfX_t = \{-\hbfX_t - \nabla \log p_t(\hbfX_t|\hbfY_t)\} \rmd t$ and $\rmd \hbfY_t = 0$ with
$(\hbfX_0, \hbfY_0) \sim \pjoin$. As one has
$\rmd \log p_t(\hbfX_t|\hbfY_t) = \mathrm{div}(-\hbfX_t - \nabla \log
p_t(\hbfX_t|\hbfY_t)) \rmd t$, we can approximately compute
$\log p(\hbfX_0|\hbfY_0)$ by integrating numerically this Ordinary Differential Equation (ODE). There are practically three sources of errors, one is the score approximation, one is the numerical integration error and the last one one is due to the fact that $\mathcal{L}(\hbfX_T)$ is unknown so we use the approximation $\mathcal{L}(\hbfX_T) \approx \pref$.

\subsection{Continuous-time CDSB}

In this section, we introduce an IPF algorithm for solving CSB
problems in continuous-time. The following results are a generalization to the conditional framework of the continuous-time results of
\cite{debortoli2021neurips}. The CDSB algorithm described in \Cref{algo:ipf_score_cond} can be seen as a Euler--Maruyama discretization of
this IPF scheme combined to neural network approximations of the drifts. Let $\Pbb \in \Pens(\contspace)$ be a given reference measure
(thought as the continuous time analog of $\bar{p}$).  The dynamical continuous
formulation of the SB problem can be written as follows
\begin{equation}
  \label{eq:dynamic_schro}
  \textstyle{
    \Pi^\star = \argmin \ensemble{\KLLigne{\Pi}{\Pbb}}{\Pi \in \Pens(\mathcal{C}), \ \Pi_0 = \pjoin, \ \Pi_T = \pjref}.
    }
\end{equation}
We define the IPF $(\Pi^n)_{n \in \nset}$ such that $\Pi^0 = \Pbb$ and
associated with $\rmd \bfX_t = - \bfX_t + \sqrt{2} \rmd \bfB_t$ and
$\rmd \bfY_t = 0$, with $(\bfX_0, \bfY_0) \sim \pjoin$. Next for any
$n \in \nset$ we define
\begin{align}
  \textstyle{\Pi^{2n+1}} &= \textstyle{\argmin \ensemble{\KLLigne{\Pi}{\Pi^{2n}}}{\Pi \in \Pens(\mathcal{C}), \ \Pi_T = \pjref}, } \\
  \textstyle{\Pi^{2n+2}} &= \textstyle{\argmin \ensemble{\KLLigne{\Pi}{\Pi^{2n+1}}}{\Pi \in \Pens(\mathcal{C}), \ \Pi_0 = \pjoin}.}
\end{align}

The following result is the continuous
counterpart of \Cref{prop:IPFrecursion}.
\begin{proposition}
  \label{prop:continuous_schro}
  Assume that $p_N, \pref >0$, $\mathrm{H}(\pref)<+\infty$ and
  $\int_{\rset^d} \absLigne{\log p_{N|0}(x_N|x_0)} \pdata(x_0) \pref(x_N) <
  +\infty$. In addition, assume that there exist $\Mbb \in \Pens(\contspace)$,
  $U \in \rmc^1(\rset^d, \rset)$, $C \geq 0$ such that for any $n \in \nset$,
  $x \in \rset^d$, $\KLLigne{\Pi^n}{\Mbb} < +\infty$,
  $\langle x, \nabla U(x) \rangle \geq - C(1+\normLigne{x}^2)$ and $\Mbb$ is
  associated with $(\bfX_t, \bfY_t)_{t \in \ccint{0,T}}$ such that 
  \begin{equation}
    \label{eq:diff_q}
    \textstyle{
      \rmd \bfX_t = -\nabla U(\bfX_t) \rmd t + \sqrt{2} \rmd \bfB_t, \qquad \rmd \bfY_t = 0
      }
    \end{equation}
    with $\bfX_0$ distributed according to the invariant distribution of \eqref{eq:diff_q}.  Then,
    for any $n \in \nset$ we have:
  \begin{enumerate}[wide, labelwidth=!, itemindent=!, labelindent=0pt, label=(\alph*)]
  \item $(\Pi^{2n+1})^R$ is associated with
    $(\bfX_t^{2n+1},\bfY_t^{2n+1})_{t \in \ccint{0,T}}$ such that
    $\rmd \bfX_t^{2n+1} = b^n_{T-t}(\bfX_t^{2n+1},\bfY_t^{2n+1}) \rmd t + \sqrt{2} \rmd
    \bfB_t$ and $\rmd \bfY_t^{2n+1} = 0$ with
    $(\bfX_0^{2n+1},\bfY_0^{2n+1}) \sim \pjref$;
  \item $\Pi^{2n+2}$ is associated with
    $\rmd \bfX_t^{2n+2} = f^{n+1}_t( \bfX_t^{2n+2},\bfY_t^{2n+2}) \rmd t + \sqrt{2} \rmd
    \bfB_t$ with $(\bfX_0^{2n+2},\bfY_0^{2n+2}) \sim \pjoin$;
  \end{enumerate}
  \vspace{-.3cm} where for any $n \in \nset$, $t \in \ccint{0,T}$,
  $x \in \rset^d$ and $y \in \mcy$,
  $b^{n}_t(x,y) = -f^{n}_t(x,y) +2 \nabla \log p^{n}_t(x|y)$,
  $f^{n+1}_t(x,y) = -b^n_t(x,y) +2 \nabla \log q^n_t(x|y)$, with
  $f^0_t(x) = -x$, and $p^n_t(\cdot|y)$, $q_t^n(\cdot|y)$ the densities of
  $\Pi^{2n}_{t|y}$ and $\Pi_{t|y}^{2n+1}$.
\end{proposition}

\begin{proof}
  The proof of this proposition is a straightforward extension of \cite[Proposition 6]{debortoli2021neurips}.
\end{proof}

We have seen in \Cref{sec:NFevidence} that it is possible to use CSGM to evaluate numerically the evidence when $g(\yobs|x)$ can be computed pointwise. The same strategy can be applied to both DSB and CDSB; see \cite[Section H.3]{debortoli2021neurips} for details for DSB. In both cases, there exists an ordinary differential equation admitting the same marginals as the diffusion solving the SB, resp. the CSB, problem. By integrating these ODEs, we can obtain $\log p(x)$ and $\log p(x|\yobs)$ for any $x$ and thus can compute the evidence. Contrary to SGM and CSGM, the terminal state of the diffusion is exactly equal to the reference measure by design. So practically, we only have two instead of three sources of errors for SGM/CSGM: one is the drift approximation, one is the numerical integration error.

\section{Forward-Backward Sampling}
\label{sec:forw-backw-sampl}

We detail in this section the forward-backward sampling approach and its connection with \cite{spantini2019coupling} when using an unconditional $\pref$. In \cite{spantini2019coupling}, it is proposed to first learn a
deterministic transport map
$\mathcal{U}(x,y):\mathcal{X}\times\mathcal{Y}\to\mathcal{X}\times\mathcal{Y}$
from $(X,Y)\sim\pjoin$ to $\pjref$,
then transport back the $X$-component through
$\mathcal{S}(\cdot,\yobs)^{-1}$ where $\mathcal{S}:\mathcal{X}\times\mathcal{Y}\to\mathcal{X}$ is the $X$-component of $\mathcal{U}$.  In other words, this is to say sampling
$\hat{X}^{\textup{pos}} \sim p(x|\yobs)$ corresponds to the two-step transformation
\begin{equation}\label{eq:composedmap-supp}
\hat{X}^{\textup{ref}},\hat{Y}^{\textup{ref}}=\mathcal{U}(X,Y),~~ \hat{X}^{\textup{pos}}=\mathcal{S}(\cdot,\yobs)^{-1}(\hat{X}^{\textup{ref}}).
\end{equation}

The proposed CSB \eqref{eq:conditionalSBextended} can be thought of as the SB version of this idea. We learn a stochastic transport map from
$\pjoin(x,y)$ to $\pref(x,y)$. The CSB $\pi^{\star}$ defines, when conditioned on
$x_{0}$ and $\yobs$, a (stochastic) transport map $\pi^{c,\star}_{\yobs}(x_{N}|x_{0})$
from $p(x_{0}|\yobs)$ to $\pref(x_{N})$; and, when
conditioned on $x_{N}$ and $\yobs$, a (stochastic) transport map
$\pi^{c,\star}_{\yobs}(x_{0}|x_{N})$ from $\pref(x_{N})$ to $p(x_{0}|\yobs)$. In
practice, we learn using CDSB separate half-bridges
$\bar{p}^{L}(x_{1:N}|x_{0},\yobs)$ and $\bar{q}^{L}(x_{0:N-1}|x_{N},\yobs)$.

\cite{spantini2019coupling} remarked that, since the estimator ${\mathcal{S}}$
may be imperfect, $\hat{X}^{\textup{ref}}$ may not have distribution $\pref$
exactly. In this case, \eqref{eq:composedmap-supp} allows for the cancellation of errors between $\mathcal{S}$ and
$\mathcal{S}(\cdot,\yobs)^{-1}$. 

We can exploit a similar idea in the CSB framework by defining an analogous forward-backward sampling procedure 
\begin{equation}
\hat{X}_{N}\sim \bar{p}^{L}_{N|0}(x_{N}|X,Y),~~\hat{X}_{0}\sim \bar{q}^{L}_{0|N}(x_{0}|\hat{X}_{N},\yobs).\label{eq:fwdbwdsampling-supp}
\end{equation}
As $\bar{q}^{L}$ is the approximate time reversal of $\bar{p}^{L}$,
\eqref{eq:fwdbwdsampling-supp} exhibits similar advantages as \eqref{eq:composedmap-supp}
when the half-bridge $\bar{p}^{L}(x_{0:N}|\yobs)$ is only an approximation to the CSB solution. While the forward and backward processes are stochastic
and are not exact inverses of each other, using this forward-backward sampling
may inevitably lead to increased variance. However, we found in practice that
this forward-backward sampling procedure can still improve sampling quality (see \eg ~ Figures \ref{fig:2dconditional}, \ref{fig:mnistinpainting}).

\section{Experimental Details}
\label{sec:experimental-details}

\subsection{Experimental Setup}

\textbf{Network parameterization}. Two parameterizations are possible for learning $\mathbf{F}$ and $\mathbf{B}$. In the main text, we described one parameterization in which we parameterize $\mathbf{F},\mathbf{B}$ directly as $\mathbf{F}_\phi^y(k,x),\mathbf{B}_\theta^y(k,x)$ and learn the network parameters $\phi,\theta$. Alternatively, we can parameterize $\mathbf{F}^y(k,x)=x+\gamma_{k+1}\mathbf{f}_\phi^y(k,x),\mathbf{B}^y(k+1,x)=x+\gamma_{k+1}\mathbf{b}_\theta^y(k+1,x)$ and learn the network parameters $\phi,\theta$ for $\mathbf{f}_\phi^y,\mathbf{b}_\theta^y$ instead. 
For the 2D and BOD examples, we use a fully connected network with positional encodings as in \cite{debortoli2021neurips} to learn $\mathbf{f}_\phi^y,\mathbf{b}_\theta^y$, with $y$ as an additional input by concatenation with $x$. 
For the MNIST and CelebA examples, we follow earlier work and utilize the conditional U-Net architecture in \cite{nichol2021beatgans}. Since residual connections are already present in the U-Net architecture, we can adopt the $\mathbf{F}_\phi^y,\mathbf{B}_\theta^y$ parameterization. In our experiments, we experiment with both parameterizations and find that the $\mathbf{f}_\phi^y,\mathbf{b}_\theta^y$ parameterization is more suitable for neural network architectures without residual connections. On the other hand, both parameterizations obtained good results when using the U-Net architecture. For consistency, all reported image experiment results use the $\mathbf{F}_\phi^y,\mathbf{B}_\theta^y$ parameterization, and we leave the choice of optimal parameterization as future research. 

\textbf{Network warm-starting}. As observed by \cite{debortoli2021neurips}, since the networks at IPF iteration $n$ are close to the networks at iteration $n-1$, it is possible to warm-start $\phi^n,\theta^n$ at $\phi^{n-1},\theta^{n-1}$ respectively. Empirically, we observe that this approach can significantly reduce training time at each CDSB iteration. Compared to CSGM, we usually observe immediate improvement in $
\mathbf{B}_{\theta^2}$ during CDSB iteration 2 when the network is warm-started at $\theta^1$ after CDSB iteration 1 (see \eg ~ \Cref{fig:psnrimprove}). As CSGM corresponds to the training objective of $\theta^1$ at CDSB iteration 1, this shows that the CDSB framework is a generalization of CSGM with observable benefits starting CDSB iteration 2. 
\begin{wrapfigure}{r}{5cm}
    \centering
    \vspace{-0.3cm}
    \includegraphics[width=.98\linewidth]{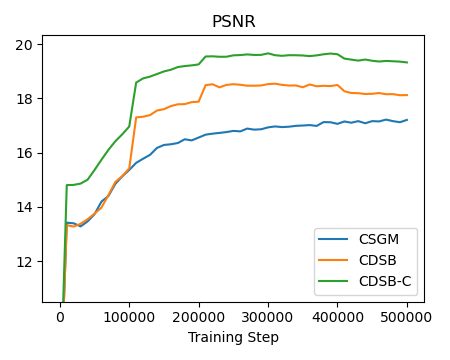} \\
    \includegraphics[width=.98\linewidth]{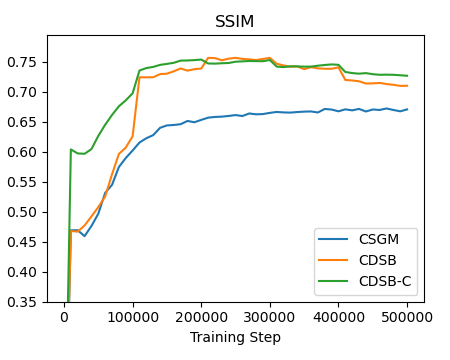}
    \vspace{-0.2cm}
    \caption{Test set PSNR and SSIM against the number of training steps for MNIST 4x SR.}
    \label{fig:psnrimprove}
    \vspace{-1.2cm}
\end{wrapfigure}

\textbf{Conditional initialization}. 
In the main text, we considered joint reference measures of the form
$\pjref(x,y)=\pref(x|y)\pobs(y)$ 
and simple choices for $\pref(x|y)$ such as 
$\vois(x;y,\sigma_{\textup{ref}}^{2} \Id)$ for image super-resolution. We also explore two more choices for $\pref(x|y)$ in our experiments. The first choice simply replaces the initialization mean from $y$ to a neural network function $\mu_\textup{ref}(y)$. 
This neural network can be pre-trained directly to estimate the conditional mean of $p(x|y)$ using standard regression with MSE loss. In the case of multi-modal $p(x|y)$ such as in the case of image inpainting, we can also train $\mu_\textup{ref}(y)$ to estimate the conditional mean of $p_N(x_N|y)$, where $x_N$ follows a standard diffusion process. In essence, we can train $\mu_\textup{ref}(y)$ to facilitate $p_N(x_N|y)\approx\pref(x_N|y)$ and shorten the noising process. Note that the CDSB framework is still useful in this context since $p_N(x_N|y)$ may not be well-approximated by a Gaussian distribution, which is precisely the issue CDSB is designed to tackle. 
Another class of conditional initialization we consider is the Ensemble Kalman Filter (EnKF), which is an ensemble-based method approximating linear Gaussian posterior updates. In this case, $\pref(x|y)$ is taken to be $\vois(x;\mu_{\textup{ref}}(y),\textup{diag}(\sigma_{\textup{ref}}^{2}(y))$ where $\mu_{\textup{ref}}(y),\sigma_{\textup{ref}}^{2}(y)$ are the sample mean and variance of the EnKF posterior ensemble. Intuitively, $\pref(x|y)$ is now an approximation of the true posterior $p(x|y)$ using linear prior-to-posterior mappings, which is further corrected for non-linearity and non-Gaussianity by the CDSB.

\textbf{Time step schedule}. For the selection of the time step sequence $\{\gamma_k\}_{k=1}^{N}$, we follow \cite{ho2020denoising,nichol2021beatgans} and consider a linear schedule where $\gamma_1=\gamma_\textup{min}$, $\gamma_N=\gamma_\textup{max}$, and $\gamma_k=\gamma_\textup{min}+\frac{k-1}{N-1}(\gamma_\textup{max}-\gamma_\textup{min})$. In this way, the diffusion step size gets finer as the reverse process approaches $\pi_0=\pdata$, so as to increase the accuracy of the generated samples.

\subsection{2D Synthetic Examples}
For the 2D examples, we use $N=50$ diffusion steps and choose the time step schedule such that $\gamma_\textup{min}={10}^{-4},\gamma_\textup{max}=0.005$. At each IPF iteration, we train the network for 30,000 iterations using the Adam optimizer with learning rate ${10}^{-4}$ and a batch size of 100.

\subsection{Biochemical Oxygen Demand Model}
For the BOD example, we again use $N=50$ diffusion steps with time schedule $\gamma_\textup{min}=\gamma_\textup{max}=0.01$. For CDSB-C, we use the shortened time schedule $\gamma_\textup{min}=\gamma_\textup{max}=0.005$ and a neural network regressor of the same architecture (with $x$ and $k$ components removed) as the conditional initialization. The batch size and optimizer settings are the same as above. 

We report the estimated posterior moments as well as their standard deviation in \Cref{tab:bodresult-supp}. We further plot the convergence of RMSE for each of the statistics in \Cref{fig:bodconvergence}.
As can be observed, IPF converges after about 20 iterations, and errors for all statistics are improved compared with CSGM (corresponding to IPF iteration 1).
Using conditional initialization also helps with localizing the problem and reduces estimation errors especially in early iterations.

\tabcolsep=0.2cm
\begin{table}[t]
\small
    \begin{centering}
    \begin{tabular}{|c|c|c|c|c|c|c|c|}
    \hline 
    \multirow{1}{*}{} &  & MCMC & CDSB & CDSB-FB & CDSB-C & MGAN & IT\tabularnewline
    \hline 
    \multirow{2}{*}{Mean} & $x_{1}$ & .075 & .066\textpm.010 & .068\textpm.010 & \textbf{.072\textpm.007} & .048 & .034\tabularnewline
    \cline{2-8} \cline{3-8} \cline{4-8} \cline{5-8} \cline{6-8} \cline{7-8} \cline{8-8} 
     & $x_{2}$ & .875 & .897\textpm.019 & .897\textpm.017 & \textbf{.891\textpm.013} & .918 & .902\tabularnewline
    \hline 
    \multirow{2}{*}{Var} & $x_{1}$ & .190 & .184\textpm.007 & \textbf{.190\textpm.007} & .188\textpm.005 & .177 & .206\tabularnewline
    \cline{2-8} \cline{3-8} \cline{4-8} \cline{5-8} \cline{6-8} \cline{7-8} \cline{8-8} 
     & $x_{2}$ & .397 & .387\textpm.006 & .391\textpm.006 & \textbf{.393\textpm.005} & .419 & .457\tabularnewline
    \hline 
    \multirow{2}{*}{Skew} & $x_{1}$ & 1.94 & \textbf{1.90\textpm.038} & 2.01\textpm.041 & \textbf{1.90\textpm.028} & 1.83 & 1.63\tabularnewline
    \cline{2-8} \cline{3-8} \cline{4-8} \cline{5-8} \cline{6-8} \cline{7-8} \cline{8-8} 
     & $x_{2}$ & .681 & .591\textpm.018 & .628\textpm.018 & .596\textpm.014 & \textbf{.630} & .872\tabularnewline
    \hline 
    \multirow{2}{*}{Kurt} & $x_{1}$ & 8.54 & 7.85\textpm.210 & \textbf{8.54\textpm.239} & 8.00\textpm.147 & 7.64 & 7.57\tabularnewline
    \cline{2-8} \cline{3-8} \cline{4-8} \cline{5-8} \cline{6-8} \cline{7-8} \cline{8-8} 
     & $x_{2}$ & 3.44 & 3.33\textpm.035 & \textbf{3.51\textpm.041} & 3.27\textpm.035 & 3.19 & 3.88\tabularnewline
    \hline 
    \end{tabular}
    \par\end{centering}
\caption{\label{tab:bodresult-supp} Estimated posterior moments and their standard deviations for the BOD example. The closest estimates to MCMC are highlighted in bold. }
\end{table}

\begin{figure}[t]
\begin{centering}
\includegraphics[width=17cm]{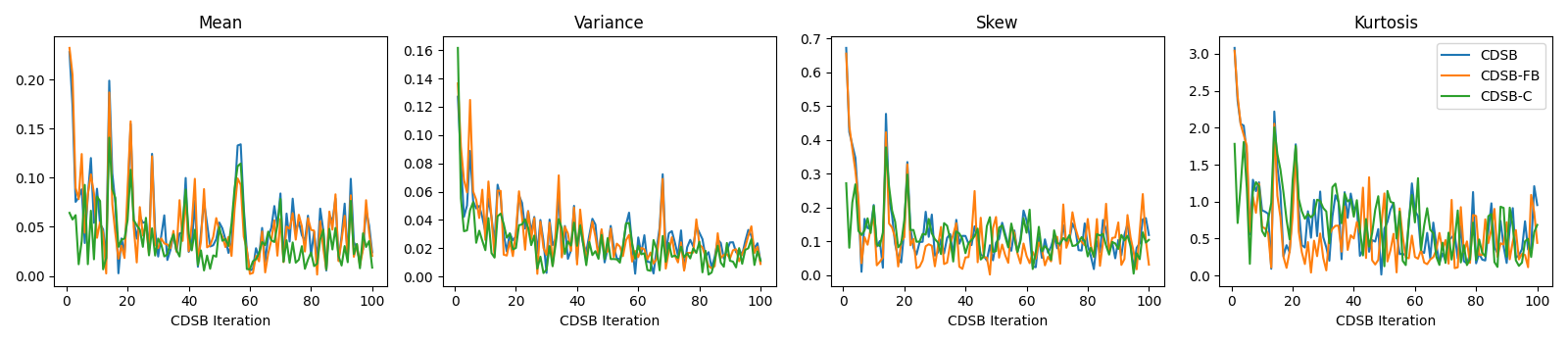}
\par\end{centering}
\caption{\label{fig:bodconvergence} Convergence of estimated posterior moments with increasing number of CDSB iterations. }
\end{figure}

\subsection{Image Experiments}
\label{subsec:imageexperimentdetail}
For all image experiments, we use the Adam optimizer with learning rate ${10}^{-4}$ and train for 500k iterations in total. Since both $\mathbf{F}$ and $\mathbf{B}$ needs to be trained, the training time is approximately doubled for CDSB. Following \cite{song2020improved}, we make use of the exponential moving average (EMA) of the network parameters with EMA rate 0.999 at test time. We use $\gamma_\textup{min}=5\times{10}^{-5}$ for all experiments unless indicated otherwise and perform a parameter sweep for $\gamma_\textup{max}$ in $\{0.005, 0.01, 0.05, 0.1\}$. The optimal $\gamma_\textup{max}$ depends on the number of timesteps $N$ and the discrepancy between $p(x|y)$ and $\pref$. When using large $N$ or conditional $\pref(x|y)$, we find $\gamma_\textup{max}$ can be taken smaller. 

\subsubsection{MNIST}
For the MNIST dataset, we use a U-Net architecture with 3 resolution levels each with 2 residual blocks. The numbers of filters at each resolution level are 64, 128, 128 respectively. The total number of parameters is 6.6m, and we use batch size 128 for training. Since we observe overfitting on the MNIST training set for all methods, we also apply dropout with $p=0.1$ for the MNIST experiments. For each CDSB iteration, 100k or 250k training steps are used, corresponding to $L=5$ or $L=2$ CDSB iterations in total, which we find to be sufficient on this simpler dataset. 

For $N=10$, CDSB generates a minibatch of 100 images in approximately 0.8 seconds when run on a GTX 1080Ti. 
As a baseline comparison, we experimented with the methodology in \cite{kadkhodaie2021stochastic}  on the same MNIST test set and find that it gives PSNR/SSIM values of 15.78/0.72 and 12.49/0.47 for super-resolution and inpainting respectively (\textit{c.f.} \Cref{tab:imagemetrics}). Around 250 iterations are required for generating each image, or approximately 1 second generation time for 1 image on a GTX 1080Ti. In comparison, the CDSB methodology is much more efficient and achieves better image quality on both tasks. 

\begin{figure}[h]
\begin{centering}
\subfloat[$\yobs$]{\includegraphics[width=5cm]{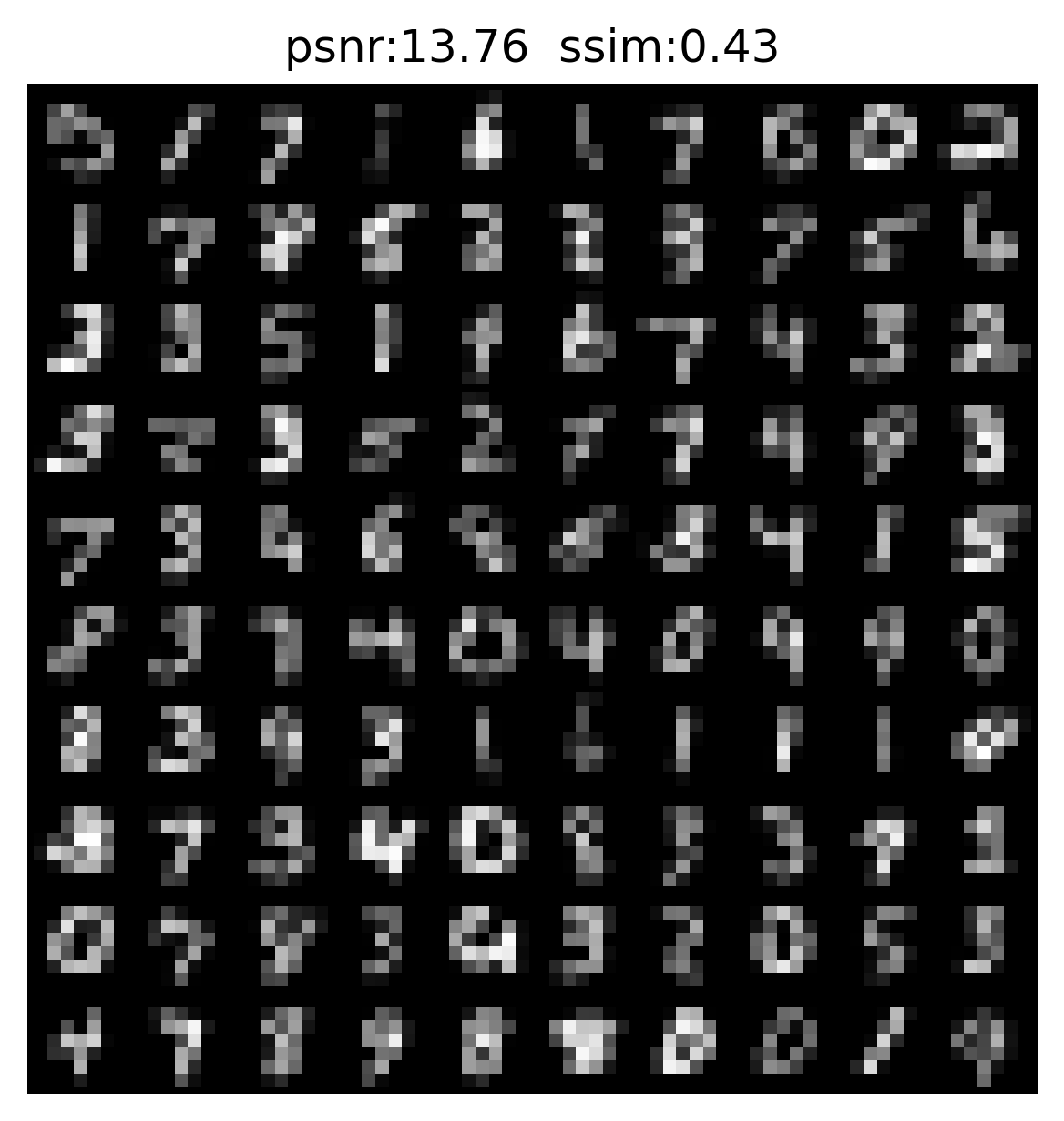}} ~~
\subfloat[Ground truth]{\includegraphics[width=5cm]{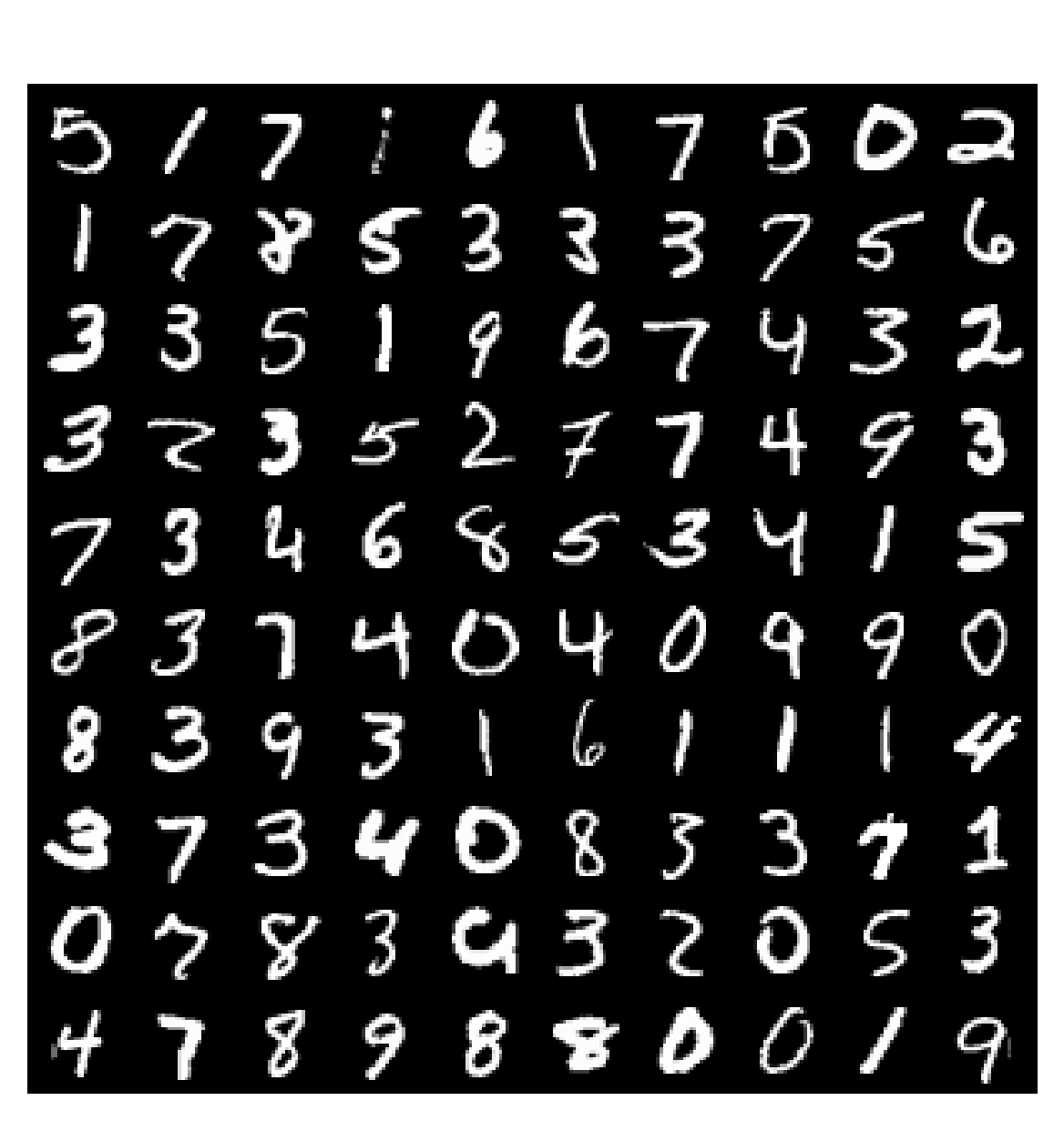}} \\
\subfloat[CSGM $N=5$]{\includegraphics[width=5cm]{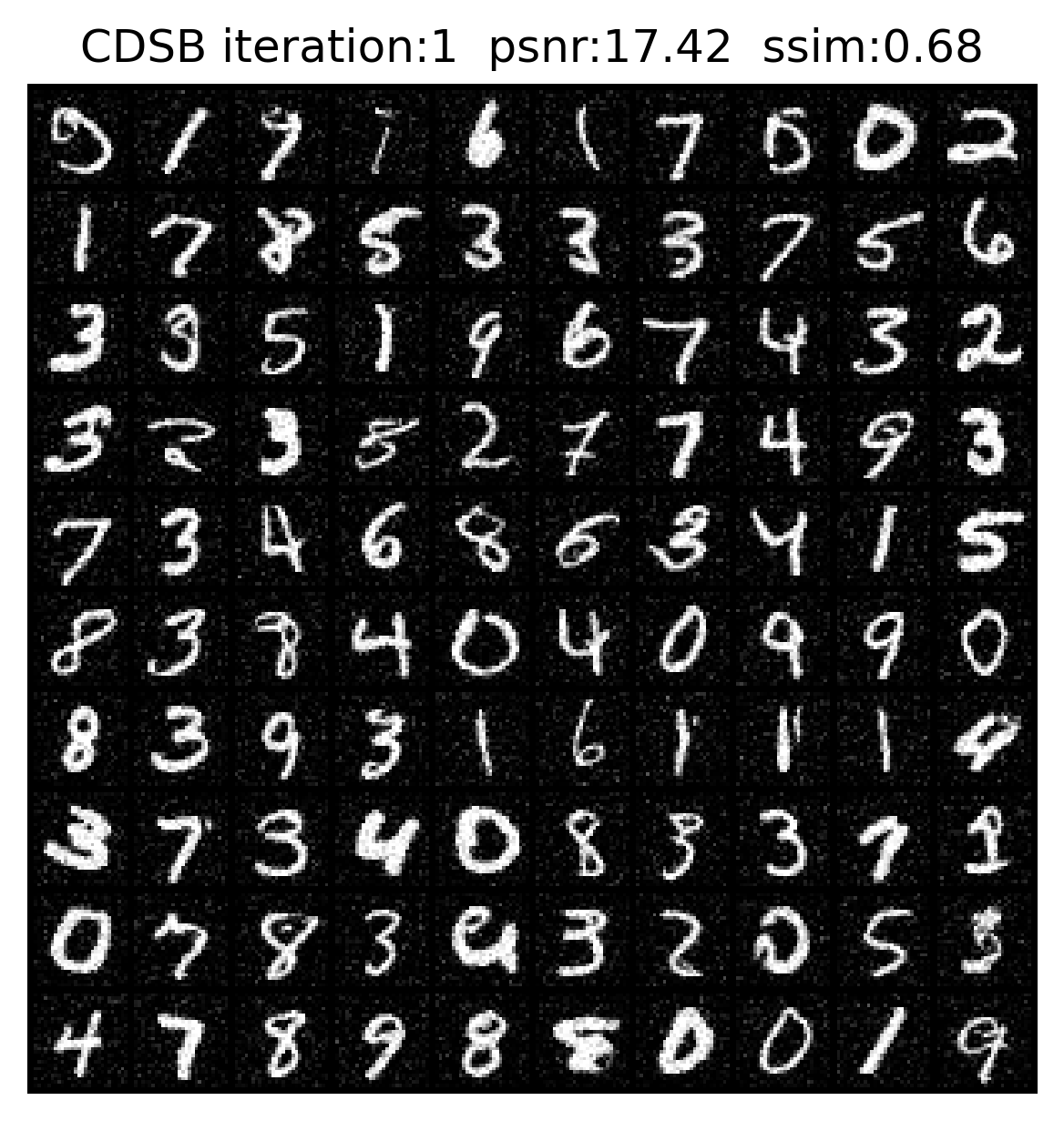}} ~~
\subfloat[CDSB $N=5$]{\includegraphics[width=5cm]{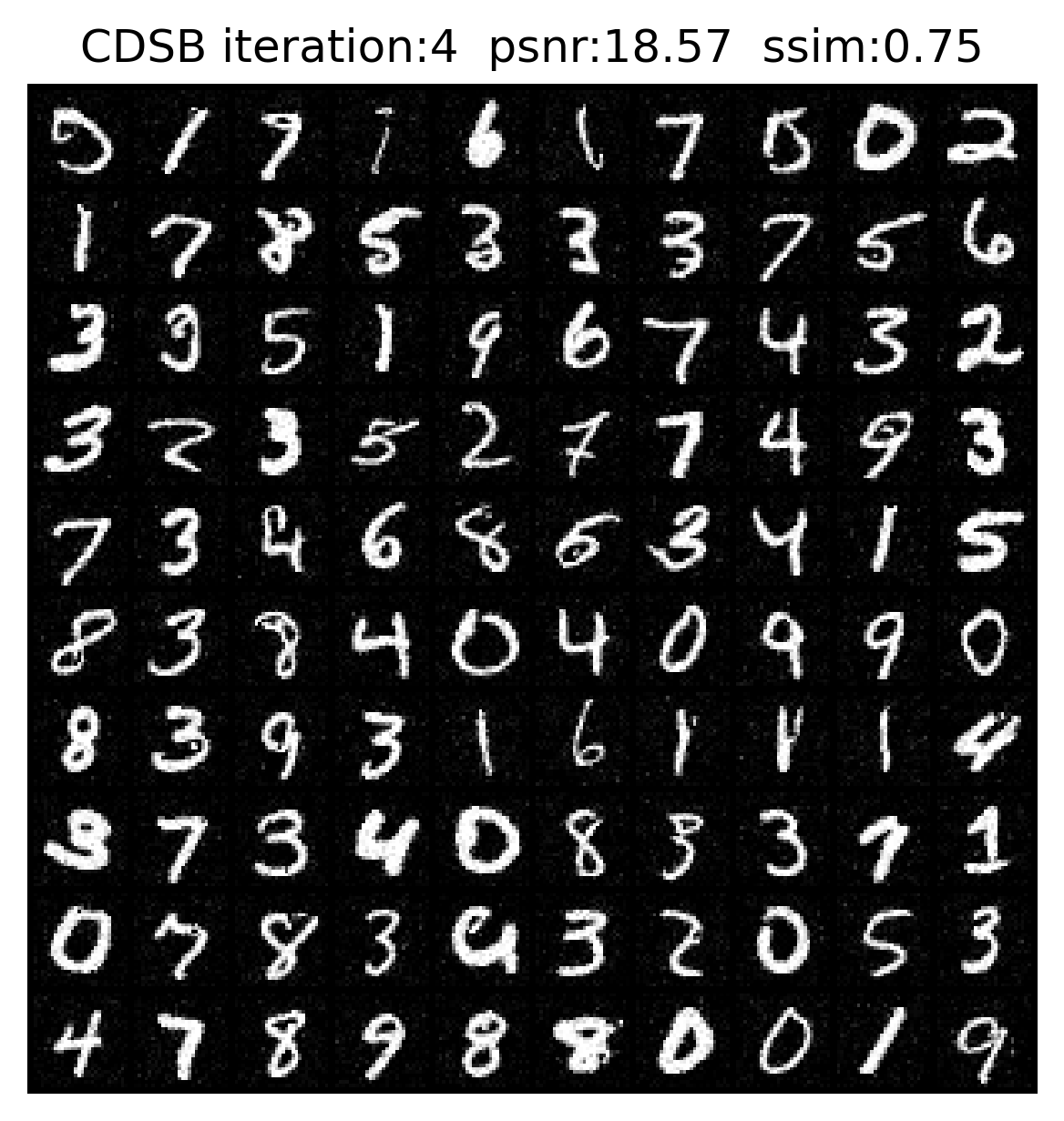}} ~~
\subfloat[CDSB-C $N=5$]{\includegraphics[width=5cm]{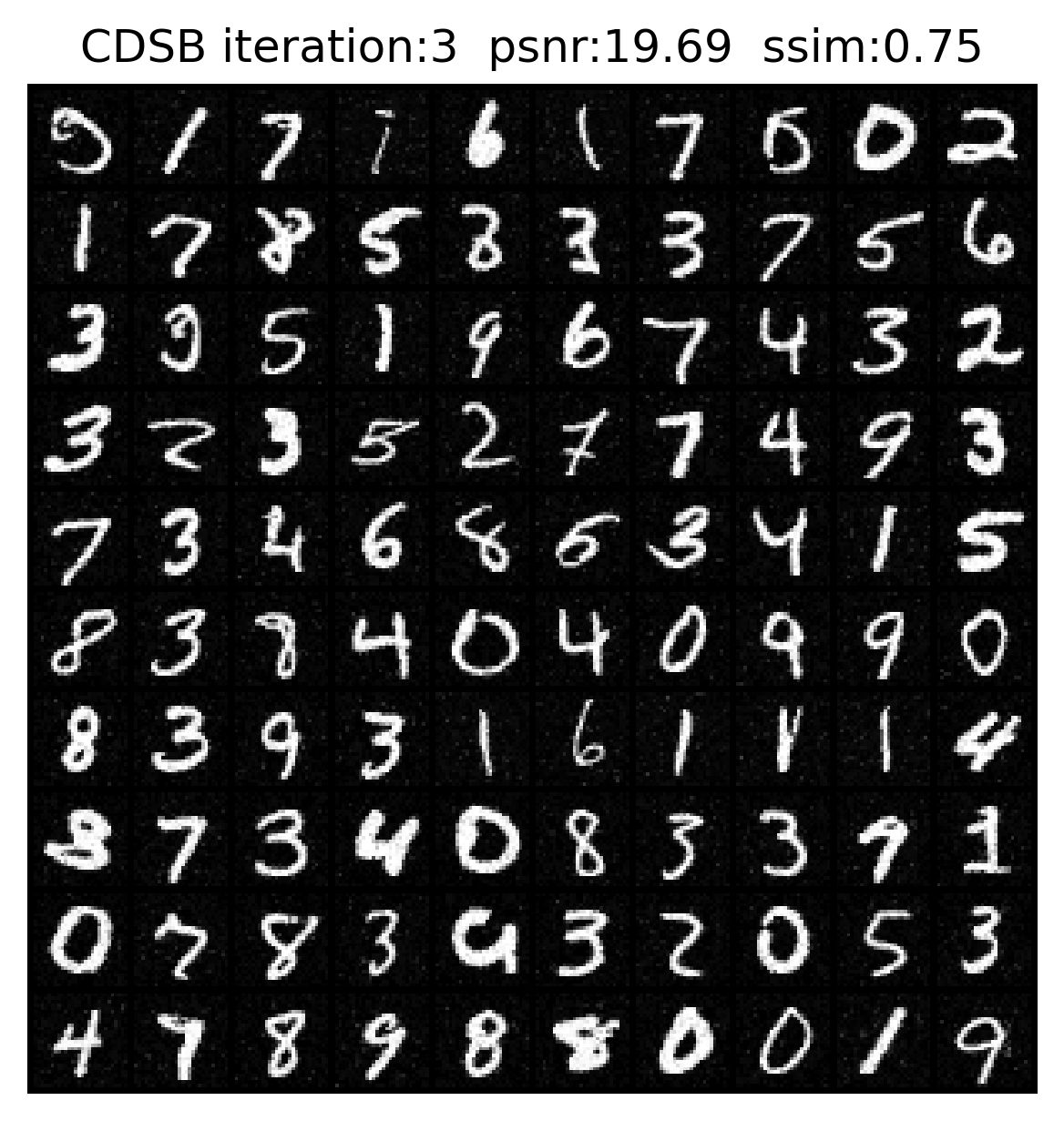}} \\
\subfloat[CSGM $N=10$]{\includegraphics[width=5cm]{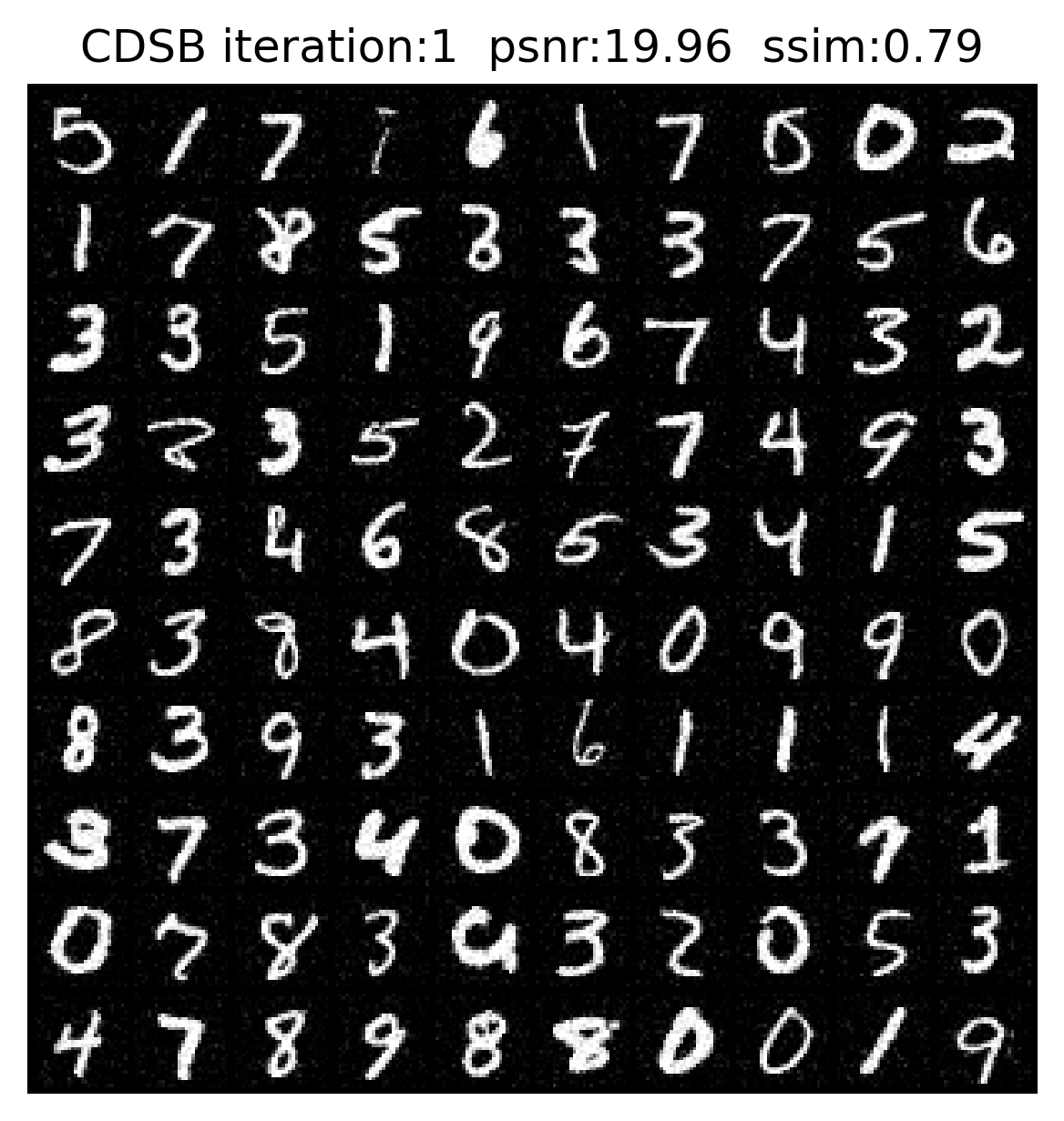}} ~~
\subfloat[CDSB $N=10$]{\includegraphics[width=5cm]{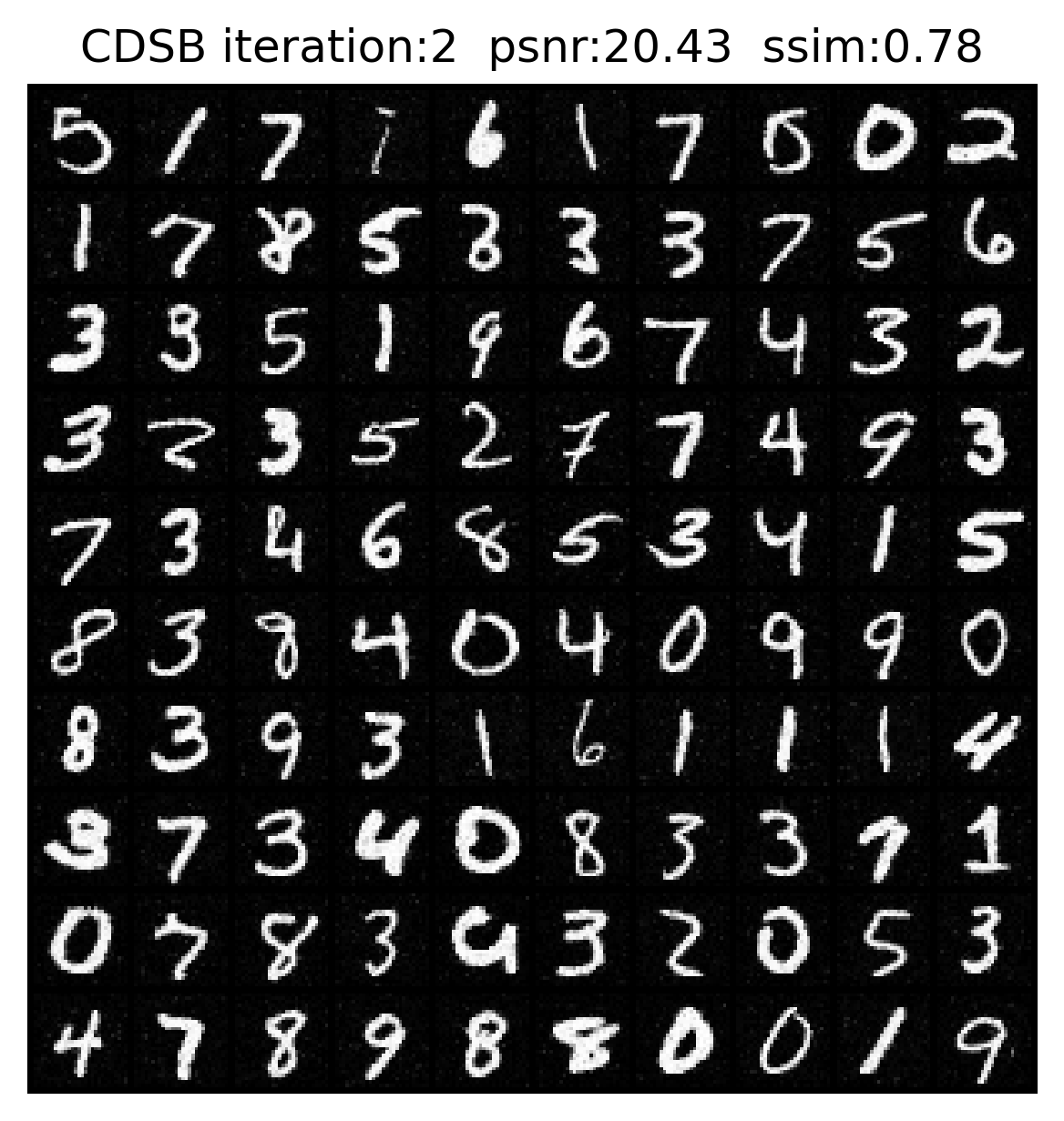}} ~~
\subfloat[CDSB-C $N=10$]{\includegraphics[width=5cm]{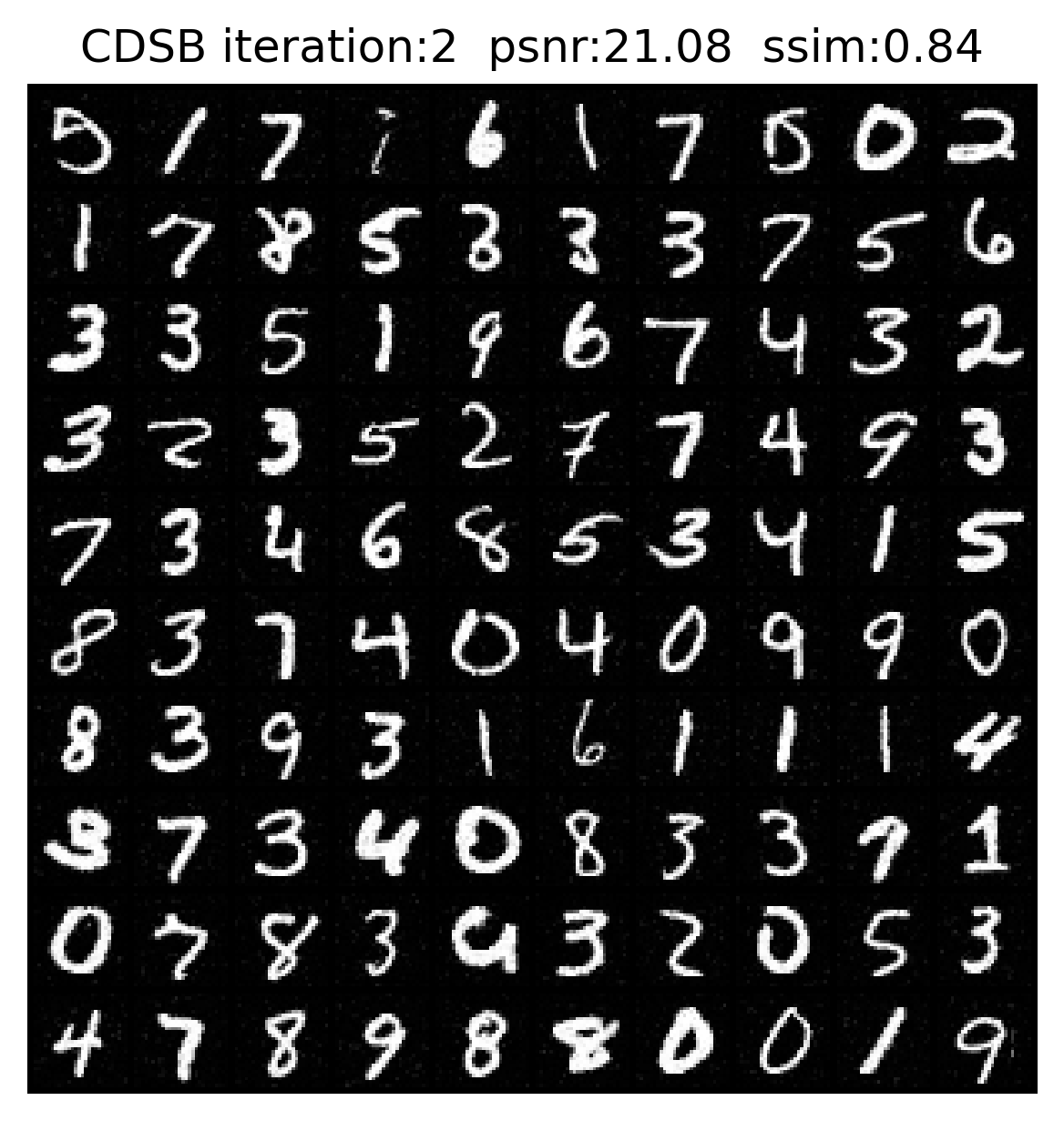}}
\par\end{centering}
\caption{Additional samples for the MNIST 4x SR task.}
\end{figure}  

\begin{figure}[h]
\begin{centering}
\subfloat[CSGM $N=10$]{\includegraphics[width=8cm]{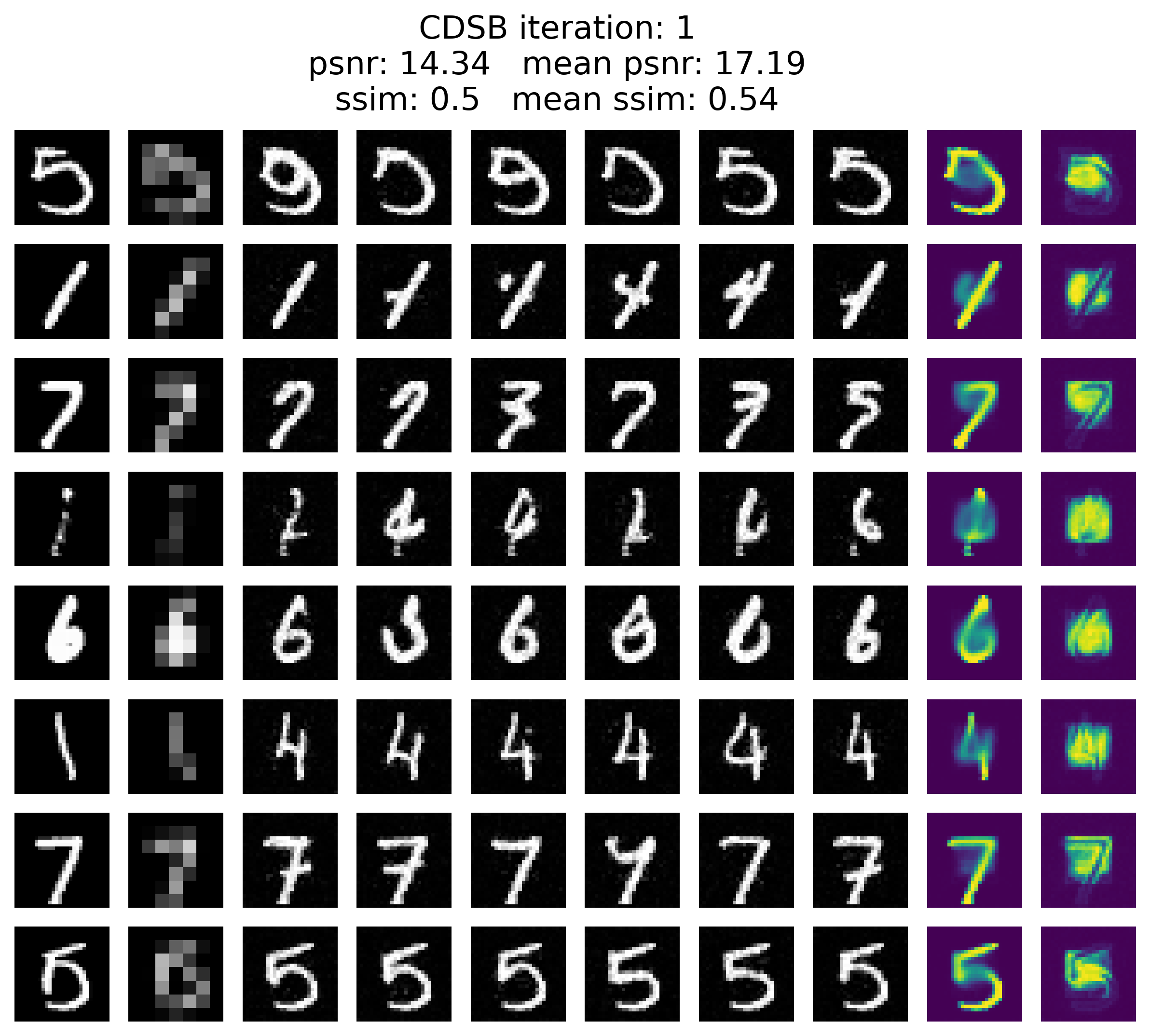}} ~~
\subfloat[CDSB $N=10$]{\includegraphics[width=8cm]{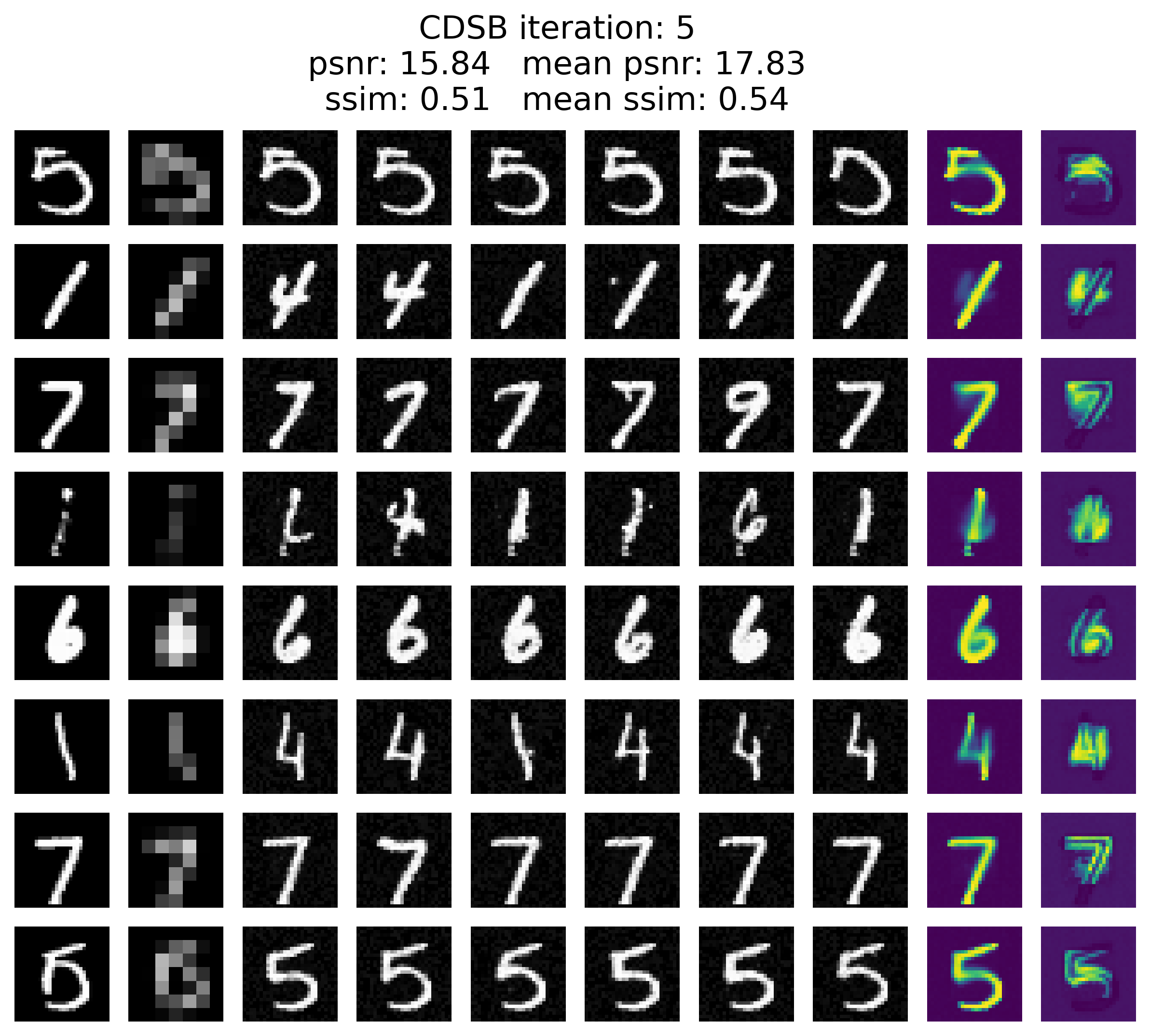}}

\subfloat[CDSB-FB $N=10$]{\includegraphics[width=8cm]{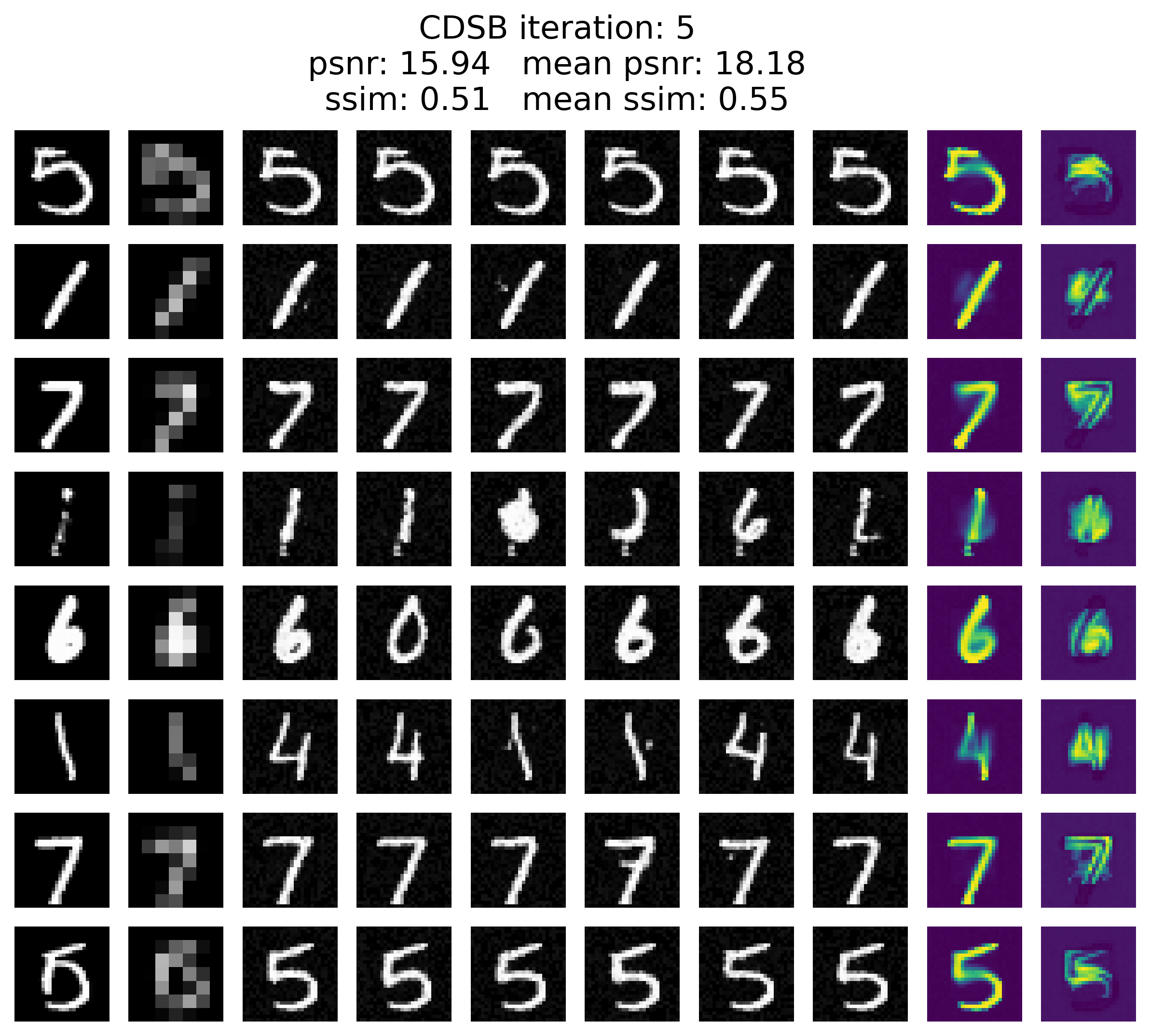}}~~
\subfloat[CDSB-C $N=10$]{\includegraphics[width=8cm]{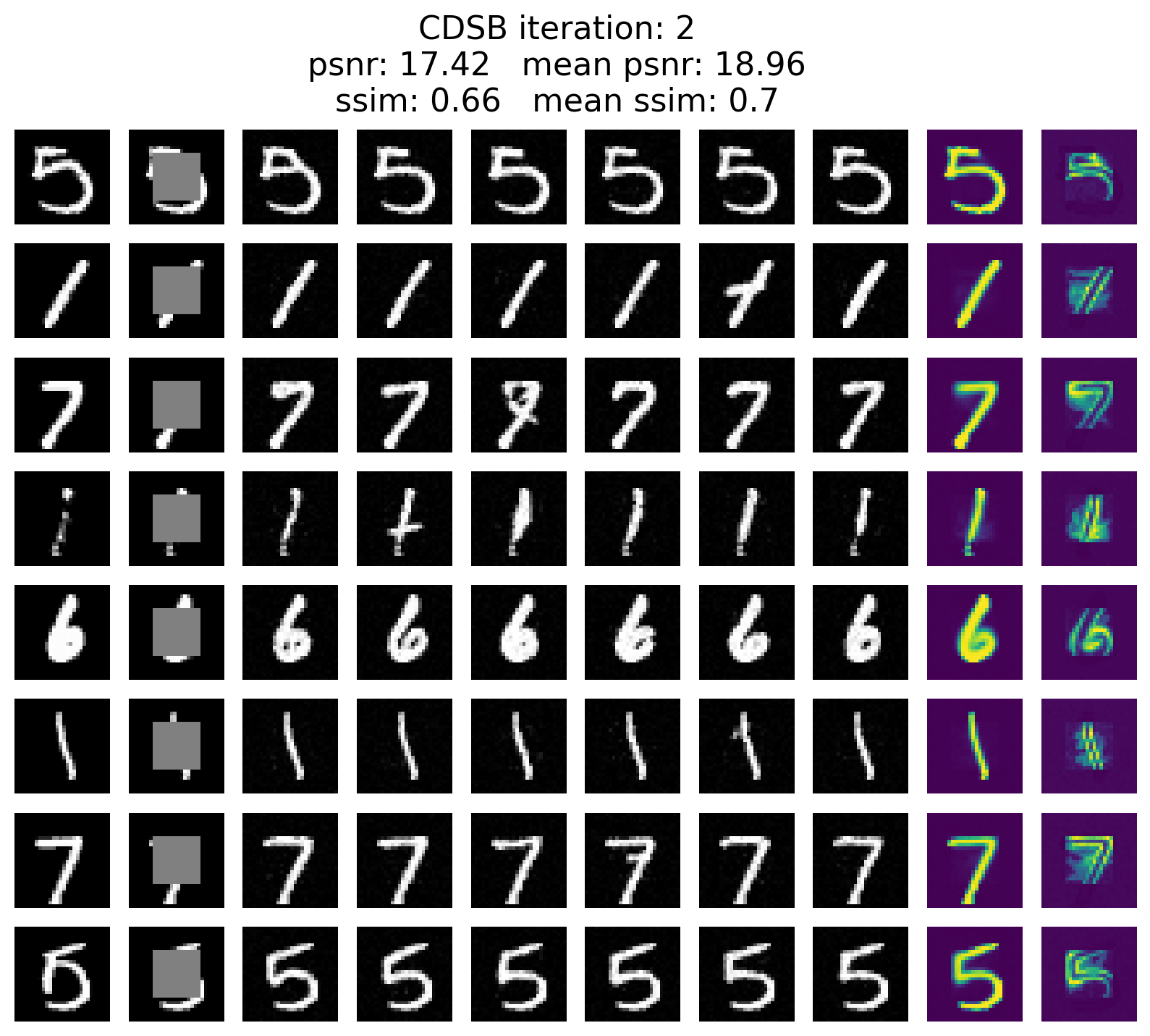}}

\par\end{centering}
\caption{\label{fig:mnistinpainting}Uncurated conditional samples for the MNIST 14x14 inpainting task. The first two columns correspond to ground truth, $\yobs$, and the last two columns correspond to the mean and standard deviation of 100 samples. }

\end{figure}

\subsubsection{CelebA 64x64}
For the CelebA dataset, we use a U-Net architecture with 4 resolution levels each with 2 residual blocks and self-attention blocks at $16\times16$ and $8\times8$ resolutions. The numbers of filters at each resolution level are 128, 256, 256, 256 respectively. The total number of parameters is 39.6m, and we use batch size 128 for training. For each CDSB iteration, 10k or 25k training steps are used, corresponding to $L=50$ or $L=20$  CDSB iterations in total. For smaller $\gamma_\textup{max}$, we find that higher number of CDSB iterations are beneficial. 

For $N=20,50$, CDSB generates a minibatch of 100 images in approximately 12, 30 seconds when run on a Titan RTX.
As a baseline comparison, we find that CDSB-C with $N=20$ even outperforms a standard CSGM with $N=200$, which achieves PSNR/SSIM values of approximately 20.98/0.62. To ensure that conditional initialization is not the sole contributor to the gain in sample quality, we further compare CDSB-C ($N=50$) to a CSGM ($N=50$) with conditional initialization. The forward noising process is also modified to the discretized Ornstein--Uhlenbeck process targeting $\pref(x|y)$ as described in \Cref{subsec:targetawareforward}. This modification  achieved PSNR/SSIM values of 20.84/0.59 (\textit{c.f.} \Cref{tab:imagemetrics}), which indicates that the CDSB framework presents larger benefits in addition to conditional initialization. 

As another baseline comparison, the SNIPS algorithm \citep{kawar2021snips} reports PSNR of 21.90 for 8 CelebA test images and, when averaging across 8 predicted samples for each of the images, a PSNR of 24.31. The algorithm requires 2500 iterations for image generation, or approximately 2 minutes for producing 8 samples when run on an RTX 3080 as reported by \cite{kawar2021snips}. On the same test benchmark, CDSB with $N=50$ achieved PSNR  values of 21.87 and 24.20 respectively in 3.1 seconds, thus achieving similar levels of sample quality using much less iterations. Furthermore, the SNIPS algorithm is applicable specifically for tractable linear Gaussian inverse problems, whereas CDSB is more general and does not rely on tractable likelihoods. 

\subsubsection{CelebA 160x160}
We adopt the official implementation and pre-trained checkpoints of SRFlow\footnote{\href{https://github.com/andreas128/SRFlow}{\texttt {https://github.com/andreas128/SRFlow}}} and make use of a higher resolution version of CelebA (160x160) following \cite{lugmayr2020srflow} in only Section \ref{subsubsec:nongaussianref}. For CSGM and CDSB, we use a U-Net architecture with 4 resolution levels each with 2 residual blocks. The numbers of filters at each resolution level are 128, 256, 256, 512 respectively. The total number of parameters is 71.0m while SRFlow has total number of parameters 40.0m. We use a batch size of 32 for training the CSGM and CDSB models. 

When $\pref(x|y)$ is defined by SRFlow, it is infeasible to use a discretized Ornstein--Uhlenbeck process targeting $\pref(x|y)$ as in Section \ref{subsec:targetawareforward}. We instead use a discretized Brownian motion for $p_{k+1|k}$, or equivalently the Variance Exploding (VE) SDE \citep{song2019generative,song2020score}. This has the interpretation as a entropy
regularized Wasserstein-2 optimal transport problem as discussed in \Cref{subsec:linkwithot}, i.e. CDSB-C seeks to minimize the total squared transport distance between SRFlow $\pref(x|y)$ and the true posterior $p(x|y)$. We use the time schedule $\gamma_\textup{min}=\gamma_\textup{max}=0.005$ with comparatively higher $\gamma_\textup{min}$ in order to accelerate convergence under $N=10$ timesteps. We provide additional samples from SRFlow, CDSB-C as well as CSGM-C in Figures \ref{fig:imagecomparison-celeba160-supp1}, \ref{fig:imagecomparison-celeba160-supp2}, \ref{fig:imagecomparison-celeba160-supp3}. 

\begin{figure}[h]
\begin{centering}
\subfloat[$\yobs$]{\includegraphics[height=5.5cm]{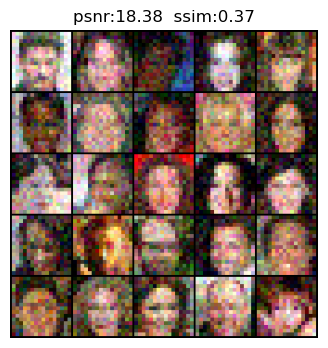}} ~~
\subfloat[Ground truth]{\includegraphics[height=5.5cm]{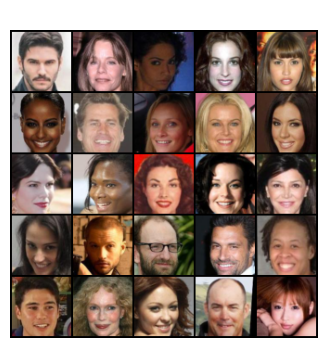}} \\
\subfloat[CSGM $N=20$]{\includegraphics[height=5.5cm]{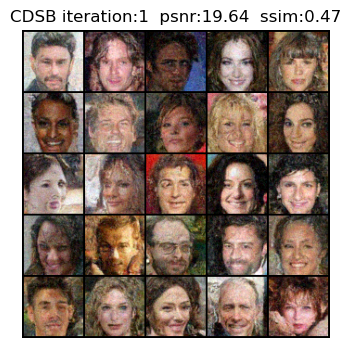}} ~~
\subfloat[CDSB $N=20$]{\includegraphics[height=5.5cm]{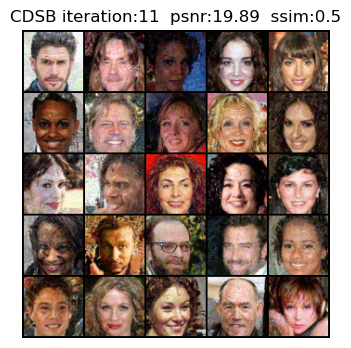}} ~~
\subfloat[CDSB-C $N=20$]{\includegraphics[height=5.5cm]{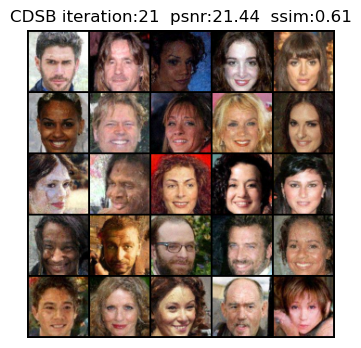}} \\
\subfloat[CSGM $N=50$]{\includegraphics[height=5.5cm]{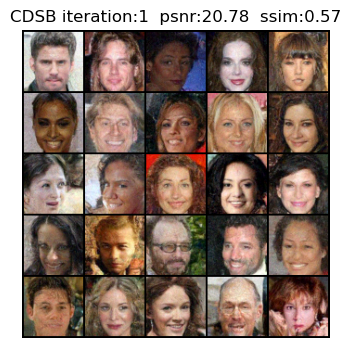}} ~~
\subfloat[CDSB $N=50$]{\includegraphics[height=5.5cm]{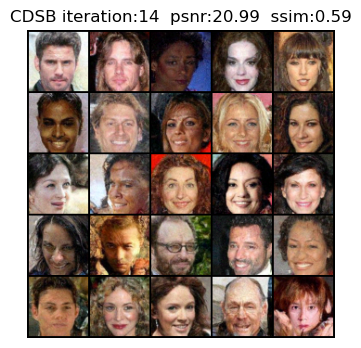}} ~~
\subfloat[CDSB-C $N=50$]{\includegraphics[height=5.5cm]{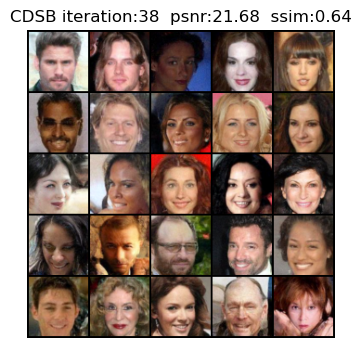}}
\par\end{centering}
\caption{Uncurated samples for the CelebA 4x SR with Gaussian noise task.}
\end{figure}  

\begin{figure}[h]
\begin{centering}
\subfloat[$\yobs$]{\includegraphics[height=5.5cm]{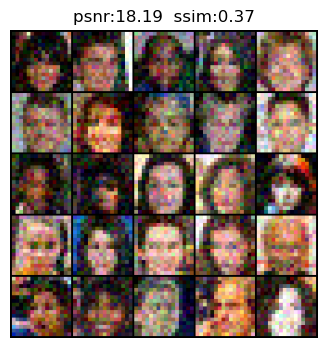}} ~~
\subfloat[Ground truth]{\includegraphics[height=5.5cm]{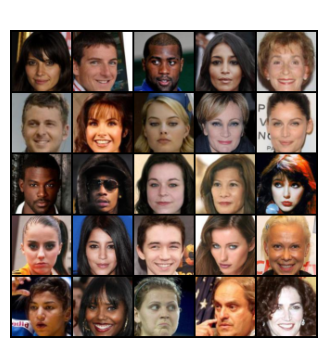}} \\
\subfloat[CSGM $N=20$]{\includegraphics[height=5.5cm]{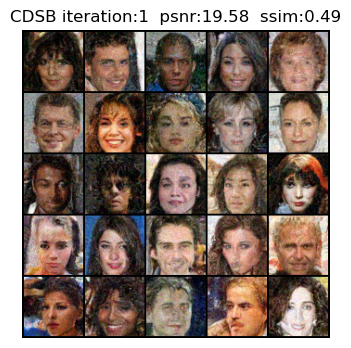}} ~~
\subfloat[CDSB $N=20$]{\includegraphics[height=5.5cm]{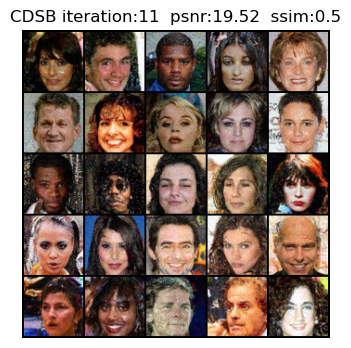}} ~~
\subfloat[CDSB-C $N=20$]{\includegraphics[height=5.5cm]{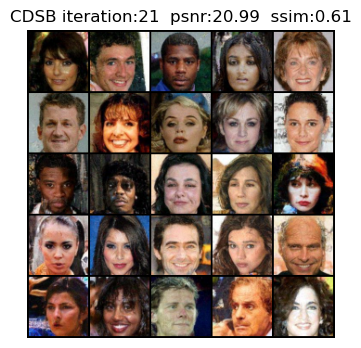}} \\
\subfloat[CSGM $N=50$]{\includegraphics[height=5.5cm]{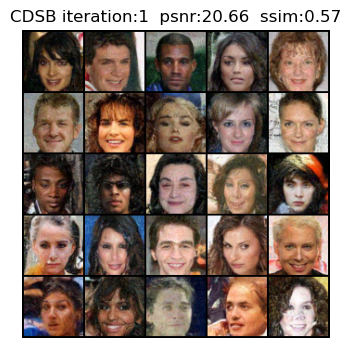}} ~~
\subfloat[CDSB $N=50$]{\includegraphics[height=5.5cm]{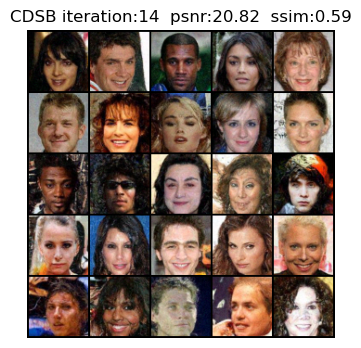}} ~~
\subfloat[CDSB-C $N=50$]{\includegraphics[height=5.5cm]{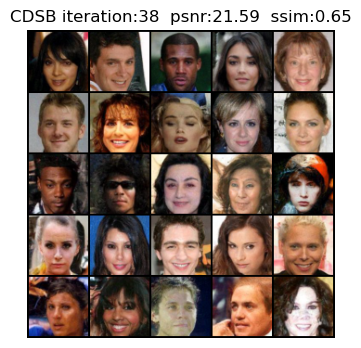}}
\par\end{centering}
\caption{Uncurated samples for the CelebA 4x SR with Gaussian noise task.}
\end{figure}

\begin{figure}[h]
\begin{centering}

\includegraphics[height=1.4cm]{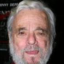}
\includegraphics[height=1.4cm]{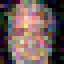} ~~
\includegraphics[trim={0 2px 0 2px}, clip, height=1.4cm]{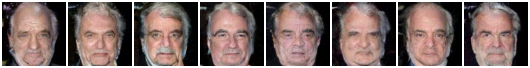} ~~
\includegraphics[height=1.4cm]{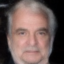}

\includegraphics[height=1.4cm]{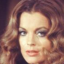}
\includegraphics[height=1.4cm]{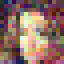} ~~
\includegraphics[trim={0 2px 0 2px}, clip, height=1.4cm]{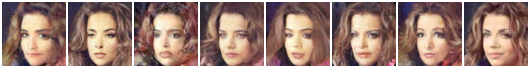} ~~
\includegraphics[height=1.4cm]{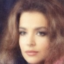}

\includegraphics[height=1.4cm]{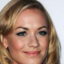}
\includegraphics[height=1.4cm]{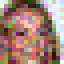} ~~
\includegraphics[trim={0 2px 0 2px}, clip, height=1.4cm]{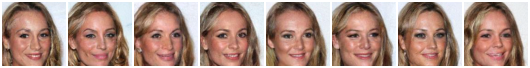} ~~
\includegraphics[height=1.4cm]{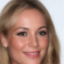}

\includegraphics[height=1.4cm]{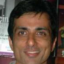}
\includegraphics[height=1.4cm]{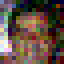} ~~
\includegraphics[trim={0 2px 0 2px}, clip, height=1.4cm]{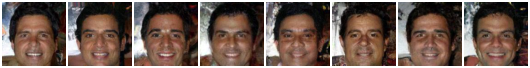} ~~
\includegraphics[height=1.4cm]{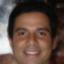}

\includegraphics[height=1.4cm]{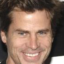}
\includegraphics[height=1.4cm]{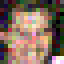} ~~
\includegraphics[trim={0 2px 0 2px}, clip, height=1.4cm]{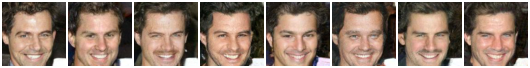} ~~
\includegraphics[height=1.4cm]{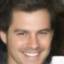}

\includegraphics[height=1.4cm]{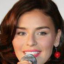}
\includegraphics[height=1.4cm]{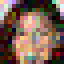} ~~
\includegraphics[trim={0 2px 0 2px}, clip, height=1.4cm]{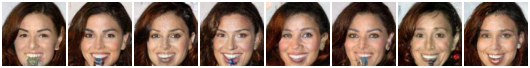} ~~
\includegraphics[height=1.4cm]{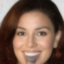}

\includegraphics[height=1.4cm]{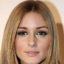}
\includegraphics[height=1.4cm]{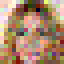} ~~
\includegraphics[trim={0 2px 0 2px}, clip, height=1.4cm]{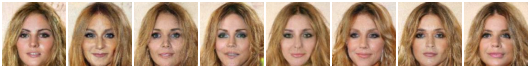} ~~
\includegraphics[height=1.4cm]{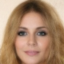}

\includegraphics[height=1.4cm]{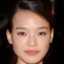}
\includegraphics[height=1.4cm]{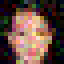} ~~
\includegraphics[trim={0 2px 0 2px}, clip, height=1.4cm]{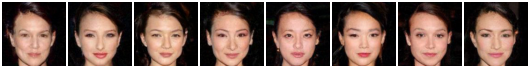} ~~
\includegraphics[height=1.4cm]{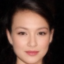}

\includegraphics[height=1.4cm]{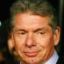}
\includegraphics[height=1.4cm]{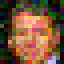} ~~
\includegraphics[trim={0 2px 0 2px}, clip, height=1.4cm]{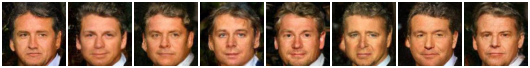} ~~
\includegraphics[height=1.4cm]{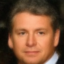}

\includegraphics[height=1.4cm]{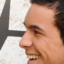}
\includegraphics[height=1.4cm]{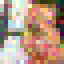} ~~
\includegraphics[trim={0 2px 0 2px}, clip, height=1.4cm]{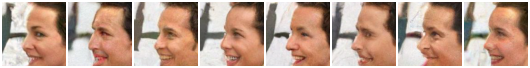} ~~
\includegraphics[height=1.4cm]{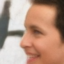}

\includegraphics[height=1.4cm]{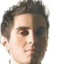}
\includegraphics[height=1.4cm]{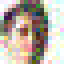} ~~
\includegraphics[trim={0 2px 0 2px}, clip, height=1.4cm]{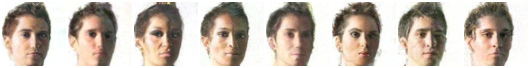} ~~
\includegraphics[height=1.4cm]{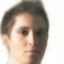}

\includegraphics[height=1.4cm]{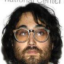}
\includegraphics[height=1.4cm]{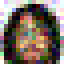} ~~
\includegraphics[trim={0 2px 0 2px}, clip, height=1.4cm]{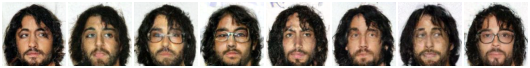} ~~
\includegraphics[height=1.4cm]{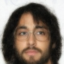}

\includegraphics[height=1.4cm]{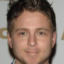}
\includegraphics[height=1.4cm]{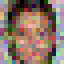} ~~
\includegraphics[trim={0 2px 0 2px}, clip, height=1.4cm]{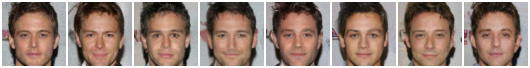} ~~
\includegraphics[height=1.4cm]{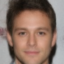}

\par\end{centering}
\caption{Uncurated conditional samples using CDSB-C with $N=50$ for the CelebA 4x SR with Gaussian noise task. The first two columns correspond to ground truth, $\yobs$, and the last column corresponds to the mean of the middle 8 samples. }
\end{figure}

\begin{figure}
\begin{centering}
\subfloat[$\yobs$]{\includegraphics[height=5.5cm]{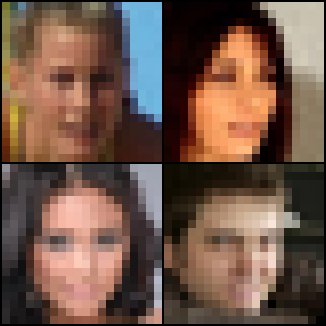}}~~
\subfloat[Ground truth]{\includegraphics[height=5.5cm]{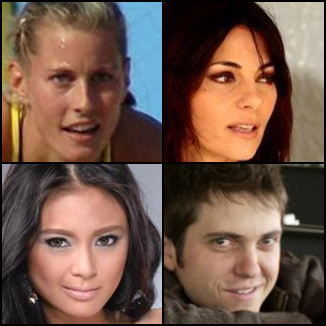}}\\
\subfloat[SRFlow]{\includegraphics[height=5.5cm]{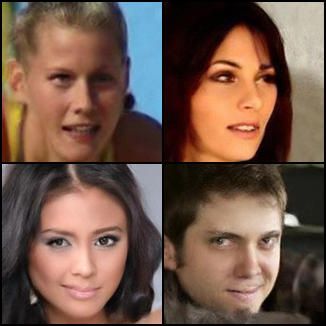}}~~
\subfloat[CSGM-C $N=10$]{\includegraphics[height=5.5cm]{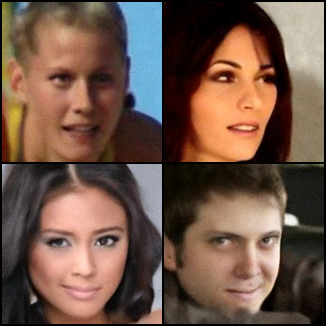}}~~
\subfloat[CDSB-C $N=10$]{\includegraphics[height=5.5cm]{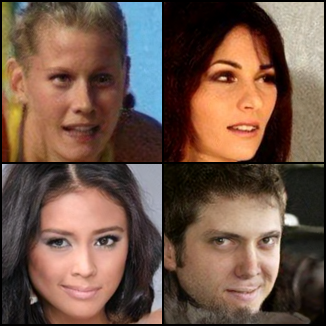}}
\par\end{centering}
\caption{\label{fig:imagecomparison-celeba160-supp1}Additional uncurated samples for the CelebA 8x SR task.}

\end{figure}

\begin{figure}
\begin{centering}
\subfloat[$\yobs$]{\includegraphics[height=5.5cm]{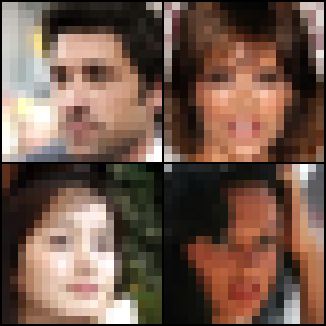}}~~
\subfloat[Ground truth]{\includegraphics[height=5.5cm]{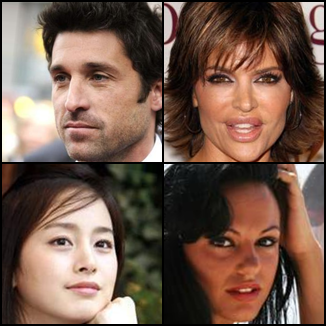}}\\
\subfloat[SRFlow]{\includegraphics[height=5.5cm]{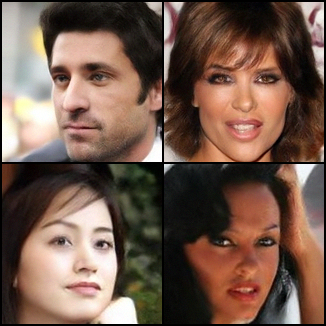}}~~
\subfloat[CSGM-C $N=10$]{\includegraphics[height=5.5cm]{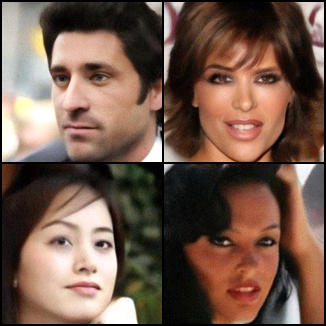}}~~
\subfloat[CDSB-C $N=10$]{\includegraphics[height=5.5cm]{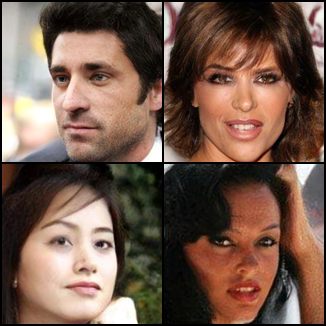}

}
\par\end{centering}
\caption{\label{fig:imagecomparison-celeba160-supp2}Additional uncurated samples for the CelebA 8x SR task.}
\end{figure}

\begin{figure}
\begin{centering}
\subfloat[$\yobs$]{\includegraphics[height=5.5cm]{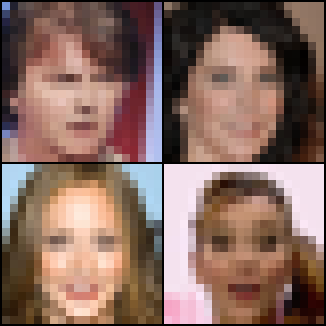}}~~
\subfloat[Ground truth]{\includegraphics[height=5.5cm]{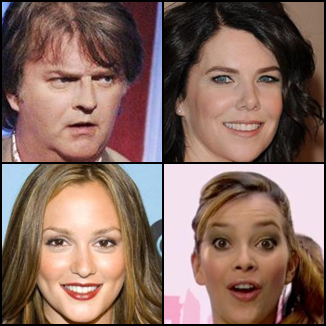}}\\
\subfloat[SRFlow]{\includegraphics[height=5.5cm]{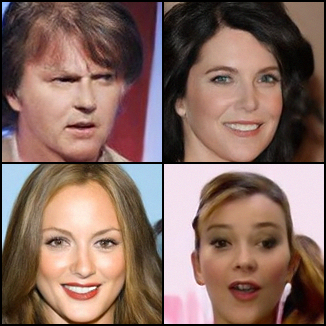}}~~
\subfloat[CSGM-C $N=10$]{\includegraphics[height=5.5cm]{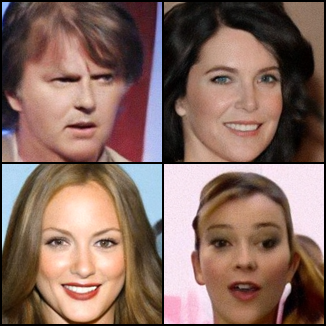}}~~
\subfloat[CDSB-C $N=10$]{\includegraphics[height=5.5cm]{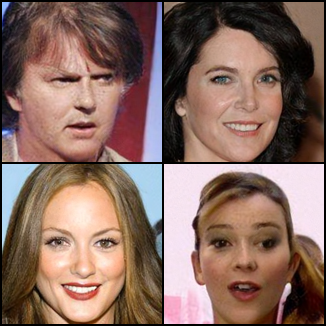}}
\par\end{centering}
\caption{\label{fig:imagecomparison-celeba160-supp3}Additional uncurated samples for the CelebA 8x SR task.}
\end{figure}

\subsection{Optimal Filtering in State-Space Models}
For the sake of completeness, we first give details of the Lorenz-63
model here. It is defined for $x\in\mathbb{R}^3$ under the following ODE
system
\[
\frac{\rmd x[1]}{\rmd \tau}=\sigma(x[2]-x[1]),\quad\frac{\rmd x[2]}{\rmd \tau}=x[1](\rho-x[3])-x[2],\quad\frac{\rmd x[3]}{\rmd \tau}=x[1]x[2]-\theta x[3].
\]
We consider the values $\sigma=10$, $\rho=28$ and $\theta=8/3$,
which results in chaotic dynamics famously known as the Lorenz
attractor. We integrate this system using the 4th order Runge--Kutta
method with step size 0.05. For the state-space model, we define $(X_{t})_{t\geq1}$
as the states $(x[1],x[2],x[3])$ of the system at regular intervals of $\delta \tau=0.1$
with small Gaussian perturbations of mean 0 and variance ${10}^{-4}$,
and $(Y_{t})_{t\geq1}$ as noisy observations of $(X_{t})_{t\geq1}$
with Gaussian noise of mean 0 and variance 4. 
More explicitly,
the transition density is thus defined for $x_t=(x_t[1],x_t[2],x_t[3])\in\mathbb{R}^3$ as 
\[
f(x_{t}|x_{t-1})=\mathcal{N}(x_{t};\textup{RK4}(x_{t-1},0.1),{10}^{-4}\Id),\quad g(y_{t}|x_{t})=\mathcal{N}(y_{t};x_{t},4\Id),
\]
 where $\textup{RK4}(x_{t},0.1)$ is the 4th order Runge--Kutta operator
(with step size 0.05) for the Lorenz-63 dynamics with initial condition
$x_{t}$ and termination time 0.1. 

We run the model for 4,000 time steps and perform Bayesian filtering
for the last 2,000 time steps. To accelerate the sequential inference
process, we use linear regression in this example to fit $\mathbf{F},\mathbf{B}$
with nonlinear feature expansion using radial basis functions. Similar
to \citet{spantini2019coupling}, we experiment with the number of
nonlinear features from 1 to 3 RBFs, in addition to the linear feature.
We find that as the ensemble size $M$ increases, increasing the number 
of features is helpful for lowering filtering errors, suggesting that
bias-variance tradeoff is at play. 

Since the system's dynamics are chaotic and can move far from the
origin and display different scaling for each dimension, it is not
suitable to choose $\pref(x)=\mathcal{N}(x;0,\Id)$.
Therefore, for CSGM and CDSB, we let $\pref(x)=\mathcal{N}(x;\mu_{\textup{ref}},\textup{diag}(\sigma_{\textup{ref}}^{2}))$
where $\mu_{\textup{ref}},\sigma_{\textup{ref}}^{2}$ are the estimated mean and variance
of the prior predictive distribution $p(x_{t}|\yobs_{1:t-1})$ at time $t$.
For CSGM-C and CDSB-C, we let $\pref(x|y)=\mathcal{N}(x;\mu_{\textup{ref}},\textup{diag}(\sigma_{\textup{ref}}^{2}))$
where the estimated posterior mean and variance are returned by EnKF.
Furthermore, we scale the diffusion process's time step dimensionwise
by the variance of the reference measure $\sigma_{\textup{ref}}^{2}$. We consider
a short diffusion process with $N=20$, and a long diffusion process
with $N=100$. We let $\gamma_{\textup{min}}=0.0005\cdot\sigma_{\textup{ref}}^{2}$
and $\gamma_{\textup{max}}=0.05\cdot\sigma_{\textup{ref}}^{2}$ for the short
diffusion process, and reduce $\gamma_{\textup{max}}$ by a half for
the long diffusion process.  

We report the RMSEs between each algorithm's filtering means and the ground truth filtering means in Table \ref{tab:filtering}. We compute the ground truth filtering means using a particle filter with $M={10}^6$ particles. In addition, we report the RMSEs between each algorithm's filtering means and the true states $x_{1:T}$ in Table \ref{tab:filteringa}, and between each algorithm's filtering standard deviations and the ground truth standard deviations in Table \ref{tab:filteringb}. Similarly, we observe that CDSB and CDSB-C achieve lower errors than CSGM and EnKF. Interestingly, CSGM-C performs similarly well as CDSB-C for state estimation when $N=100$ steps, but performs worse for standard deviation estimation.  In the case where the ensemble size $M=200$, however, when using the long diffusion process we observe occasional large errors for CDSB and CDSB-C. We conjecture that since CDSB is an iterative algorithm, inevitably small errors in regression can be accumulated. For small ensemble size and large number of diffusion steps, the model may thus be more prone to overfitting. However, for larger ensemble size $M\geq500$ we do not observe this issue.

\tabcolsep=0.25cm
\begin{table}
\begin{centering}
\subfloat[\label{tab:filteringa}]{

    \begin{tabular}{|c|c|c|c|c|}
    \hline 
    $M$ & 200 & 500 & 1000 & 2000\tabularnewline
    \hline 
    \hline 
    EnKF & .476\textpm .010 & .474\textpm .005 & .475\textpm .005 & .475\textpm .003\tabularnewline
    \hline 
    CSGM (short) & \multicolumn{4}{c|}{Diverges}\tabularnewline
    \hline 
    CDSB (short) & .464\textpm .013 & .391\textpm .010 & .369\textpm .007 & .352\textpm .008\tabularnewline
    \hline 
    CSGM-C (short)& \multicolumn{4}{c|}{Diverges}\tabularnewline
    \hline 
    CDSB-C (short) & \textbf{.428\textpm .016} & \textbf{.378\textpm .012} & \textbf{.359\textpm .015} & \textbf{.340\textpm .007}\tabularnewline
    \hline 
    \hline 
    CSGM (long) & \textbf{.431\textpm.010} & .376\textpm .008 & .360\textpm .012 & .343\textpm .006\tabularnewline
    \hline 
    CDSB (long) & .582\textpm.328 & .370\textpm .012 & .348\textpm .006 & .333\textpm .006\tabularnewline
    \hline 
    CSGM-C (long) & .434\textpm.057 & \textbf{.367\textpm.011}  & .346\textpm.008 & .336\textpm.004\tabularnewline
    \hline 
    CDSB-C (long) & .660\textpm.310 & .368\textpm .016 & \textbf{.344\textpm .010} & \textbf{.331\textpm .006}\tabularnewline
    \hline 
    \end{tabular}

} \qquad
\subfloat[\label{tab:filteringb}]{

    \begin{tabular}{|c|c|c|c|c|}
    \hline 
    $M$ & 200 & 500 & 1000 & 2000\tabularnewline
    \hline 
    \hline 
    EnKF & .255\textpm.003 &  .286\textpm .002 & .296\textpm .001 & .300\textpm .003\tabularnewline
    \hline 
    CSGM (short)& \multicolumn{4}{c|}{Diverges}\tabularnewline
    \hline 
    CDSB (short)& .203\textpm.005 &  .167\textpm .003 & .150\textpm .002 & .137\textpm .002\tabularnewline
    \hline 
    CSGM-C (short)& \multicolumn{4}{c|}{Diverges}\tabularnewline
    \hline 
    CDSB-C (short)& \textbf{.148\textpm.004} &  \textbf{.124\textpm .002} & \textbf{.108\textpm .002} & \textbf{.099\textpm .001}\tabularnewline
    \hline 
    \hline 
    CSGM (long) & .204\textpm.005 &  .163\textpm .008 & .140\textpm .002 & .129\textpm .001\tabularnewline
    \hline 
    CDSB (long) & \textbf{.140\textpm.008} &  .129\textpm .003 & .123\textpm .003 & .120\textpm .002\tabularnewline
    \hline 
    CSGM-C (long) & .186\textpm.005 & .142\textpm.003 & .120\textpm.001 & .109\textpm.002\tabularnewline
    \hline 
    CDSB-C (long) & .176\textpm.006 &  \textbf{.120\textpm .002} & \textbf{.110\textpm .003} & \textbf{.106\textpm .002}\tabularnewline
    \hline 
    \end{tabular}

}
\par\end{centering}
\caption{\label{tab:filtering-supp}RMSEs over 10 runs between (a) each algorithm's filtering means
and the true states $x_{1:T}$ for $N=20$ (short) and $N=100$ (long); (b) each algorithm's filtering standard deviations and the ground truth filtering standard deviations. 
The lowest errors are highlighted in bold.}
\end{table}

\end{document}